%% file: main.tex
\documentclass{article}

\newif\ifpreprint
\preprinttrue

\usepackage{iclr2027_conference,times}

\usepackage{latexsym}
\usepackage[T1]{fontenc}
\usepackage[utf8]{inputenc}
\usepackage{microtype}
\usepackage{inconsolata}
\usepackage{graphicx}
\usepackage{amsmath}
\usepackage{amssymb}
\usepackage{amsthm}
\usepackage{algorithm}
\usepackage{algpseudocode}
\usepackage{booktabs}
\usepackage{enumitem}
\usepackage{float}
\usepackage[framemethod=TikZ]{mdframed}
\usepackage{xcolor}
\usepackage{colortbl}
\usepackage{hyperref}
\usepackage{url}

\newtheorem{theorem}{Theorem}
\newtheorem{corollary}{Corollary}
\newtheorem{proposition}{Proposition}

\mdfsetup{skipbelow=-4pt}
\mdfdefinestyle{theorembox}{%
  linewidth=1pt,
  linecolor=blue!50!black,
  backgroundcolor=blue!5,
  roundcorner=4pt,
  skipabove=3pt,
  innertopmargin=4pt,
  innerbottommargin=2pt,
  splittopskip=4pt,
  splitbottomskip=4pt,
  innerleftmargin=8pt,
  innerrightmargin=8pt,
  nobreak=false,
}
\mdfdefinestyle{promptbox}{%
  linewidth=0.8pt,
  linecolor=blue!45!black,
  backgroundcolor=blue!3,
  roundcorner=3pt,
  skipabove=6pt,
  innertopmargin=6pt,
  innerbottommargin=6pt,
  splittopskip=4pt,
  splitbottomskip=4pt,
  innerleftmargin=10pt,
  innerrightmargin=10pt,
  nobreak=false,
}
\surroundwithmdframed[style=theorembox]{theorem}
\surroundwithmdframed[style=theorembox]{corollary}
\surroundwithmdframed[style=theorembox]{proposition}

\newcommand{\policychange}[2]{\textcolor{#1}{\tiny\,(#2)}}
\newcommand{\policygain}[1]{\policychange{green!45!black}{#1}}
\newcommand{\policyloss}[1]{\policychange{red!70!black}{#1}}
\newcommand{\policyeven}[1]{\policychange{black!55}{#1}}

\newcommand{\initialpolicy}[1]{\textcolor{black!55}{#1}}

\newcommand{\papertitle}{Rubric Rewards from Item Response Theory}
\title{\papertitle}

\author{
  \setcounter{footnote}{2}%
  Milad Yazdani\textsuperscript{1,}\thanks{Work done during an internship at Microsoft.}
  \And
  \setcounter{footnote}{1}%
  Yaser Souri\textsuperscript{2,}\thanks{Corresponding author: \href{mailto:yasersouri@microsoft.com}{\texttt{yasersouri@microsoft.com}}.}
  \And
  Xiren Zhou\textsuperscript{2}
  \And
  \setcounter{footnote}{0}%
  Pranit Chawla\textsuperscript{2}
  \AND
  Dena Shahriari\textsuperscript{3}
  \And
  Subhojit Som\textsuperscript{2}
  \And
  Xia Song\textsuperscript{2}
  \AND
  \normalfont\textsuperscript{1}\,Department of Electrical and Computer Engineering, University of British Columbia \\
  \normalfont\textsuperscript{2}\, Microsoft,
  \normalfont\textsuperscript{3}\,School of Biomedical Engineering, University of British Columbia
}

\ifpreprint
  \iclrfinalcopy
  \newcommand{\codeurl}{https://github.com/milad1378yz/rrt}
\else
  \newcommand{\codeurl}{https://anonymous.4open.science/r/rrt-9DC4}
\fi
\begin{document}
\maketitle
\ifpreprint\lhead{\papertitle}\fi

\begin{abstract}
Many language tasks have no single answer that can be checked automatically. Rubrics provide criteria for judging responses to these tasks. For reinforcement learning, the resulting verdicts must be combined into a scalar reward. A common approach sums the points assigned to satisfied criteria. Distinct verdict patterns can thus receive the same reward, and the fixed points encode how much each criterion should count, not how strongly its verdict distinguishes the current rollouts. Beyond this aggregation problem, judging the full rubric needs more judge requests as the criterion count grows. To address these limitations, Rubric Response Theory (RRT) measures quality and selects criteria when rubric criteria are monotone indicators of a shared target. Rather than adding assigned points, RRT uses a two parameter item response model that treats the verdict pattern as evidence about scalar quality specific to the rubric. Under this model, its likelihood score maximizes the local signal-to-noise ratio for quality. Its Response Parameter Network (RPN) reads the prompt and criterion text to predict criterion difficulty and discrimination. As the policy distribution changes during training, RRT uses online expectation maximization to update the RPN from current rollout verdicts. With Qwen3.5-4B as the policy, RRT's macro criterion score across Medical, Science, Rubrics as Rewards Science, and RubricBench is 1.7 points above that of group relative policy optimization (GRPO). On hard and very hard criteria in Medical and Science, RRT gains 2.8 to 5.6 points over GRPO. At half the criterion budget, adaptive Fisher selection with a frozen RPN keeps the macro criterion score across four datasets within 0.1 points of GRPO with full judging. These results show RRT can reduce judge requests while remaining competitive with GRPO. \href{\codeurl}{Code}
\end{abstract}

\input{sections/introduction}

\input{sections/related_works}

\input{sections/method}

\input{sections/experiments}

\input{sections/discussion}

\bibliographystyle{iclr2027_conference}
\bibliography{refs}

\input{sections/appendix}

\end{document}

%% file: sections/introduction.tex
\section{Introduction}
\label{sec:introduction}
Reinforcement learning for language models has a well-defined reward when success is automatically verifiable, but many language tasks have no single correct answer. Such tasks use large language models as judges \citep{zheng2023judging} and rubrics specific to each prompt \citep{gunjal2026rubrics,viswanathan2025checklists}. A rubric decomposes an evaluation target into natural language criteria such as required content, reasoning steps, output constraints, and errors to avoid. For policy training, the resulting binary criterion verdicts are mapped to a scalar reward for each rollout. Beyond this aggregation problem, judging the full rubric requires more judge requests as criteria are added.

To produce this scalar reward, the baseline takes a weighted average of satisfied criteria using their assigned points \citep{gunjal2026rubrics,arora2025healthbench}. This additive form can assign the same total to distinct verdict patterns and therefore gives them the same rubric reward in group relative policy optimization (GRPO) \citep{guo2025deepseek}. The assigned points encode how much each criterion should count in the rubric, not how strongly its verdict distinguishes the current rollouts. The contribution of each verdict also does not depend on the rollout's other verdicts.

This work introduces Rubric Response Theory (RRT), which adapts item response theory (IRT) to on-policy reward inference and criterion selection for rubrics specific to each prompt. For rubrics whose criteria are monotone indicators of one shared target, RRT uses the verdicts to estimate rollout quality instead of adding assigned points. Like questions in a test, criteria can differ in difficulty and in how strongly they distinguish responses of different quality. RRT represents each rubric criterion with difficulty and discrimination parameters from IRT \citep{chen2025item}. Difficulty locates a criterion on the quality scale, while discrimination controls how sharply its pass probability changes with quality. The Response Parameter Network (RPN) predicts these parameters from the prompt and criterion text. Using these parameters, RRT finds the quality that best explains the full verdict pattern and uses it as the GRPO reward. Rubric points specify how much a criterion should count, while RRT measures how much its verdict reveals about shared quality. This distinction lets RRT distinguish rollouts that receive the same reward based on rubric points but whose verdict patterns provide different evidence about quality. Because the rollout distribution changes with the policy, RRT uses online hard expectation maximization (EM) to update the RPN from current rollout verdicts. It uses the same parameters to select criteria under a criterion budget.

\begin{figure}[!htbp]
\centering
\includegraphics[width=0.9\textwidth]{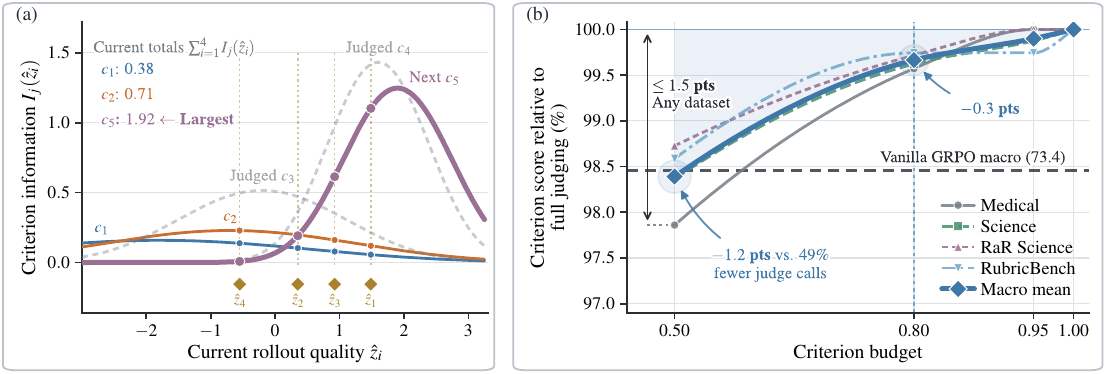}
\caption{Criterion selection in RRT using Fisher information. (a) Criterion information for five criteria and four rollouts after judging two criteria. Gold diamonds show inferred rollout qualities. Dots show information from unjudged criteria at these qualities, dashed gray curves show judged criteria, and $c_5$ is selected next. (b) Criterion scores after RRT with adaptive Fisher selection, normalized to full judging for each dataset. Shading shows the macro gap from full judging.}
\label{fig:opening}
\end{figure}

The contributions are the following.

\begin{itemize}[leftmargin=0.9em,nosep]
  \item RRT formulates rubric aggregation as Bayesian inference, using the posterior mode of quality as a scalar GRPO reward to distinguish rollouts with equal rubric point totals. It combines the likelihood of full verdict patterns with a quality prior. The RPN predicts criterion difficulty and discrimination from prompt and criterion text for unseen rubrics.
  \item RRT provides an online EM procedure to calibrate the RPN from current rollout verdicts as the policy changes. A theoretical analysis establishes that, in RRT's item response model, the rubric likelihood score maximizes the local signal-to-noise ratio (SNR) for small changes in rollout quality among all scalar functions of the verdicts, reaching the Fisher information bound.
  \item RRT extends adaptive Fisher selection to groups of rollouts, reducing judge requests under a criterion budget. It ranks unjudged criteria by total Fisher information at inferred rollout qualities and updates those qualities after each selected criterion is judged.
\end{itemize}

\textbf{Key Findings:} Across Medical, Science, Rubrics as Rewards (RaR) Science, and RubricBench, RRT preserves the gains in macro criterion score from Vanilla GRPO across three base policies. It exceeds Vanilla GRPO by 1.7 points on Qwen3.5-4B. On rollouts from trained policies, adaptive Fisher selection reaches 95.0\% mean Pearson correlation with GRPO advantages from full rubric judging while leaving 21.0\% of criteria unjudged. At half the criterion budget, RRT with adaptive Fisher selection keeps its macro criterion score across all four datasets within 0.1 points of Vanilla GRPO with full judging (Figure~\ref{fig:opening}).

%% file: sections/related_works.tex
\section{Related Work}
\label{sec:related_works}

Rubric and checklist supervision provides training signals specific to each prompt. \citet{gunjal2026rubrics} generate rubrics from reference answers and train with GRPO on a weighted sum of criterion verdicts or one rating from a judge that reads the full rubric. They call fixed weights brittle and leave learned weighting to future work. \citet{viswanathan2025checklists} generate checklists from failure modes in candidate responses and combine judge and program verifier scores with generated importance weights. Both fix criterion weights in advance or leave aggregation to the judge, and judge the full rubric for every response. RRT complements rubric generation. It learns from current rollouts how much each verdict reveals about quality, to compute the reward and to choose which criteria to judge.

Recent work adapts rubric aggregation to current rollouts. POW3R \citep{tyagi2026not} rescales human weights by contrast across rollouts, and DIVA \citep{cook2026check} weights soft criterion scores by their variance across responses. Both keep a weighted sum with one weight per criterion for all rollouts and judge every criterion. RRT replaces the weighted sum. It learns each criterion's difficulty and how well it separates good from poor responses, then finds the overall quality that best explains a rollout's full verdict pattern. Two rollouts with equal point totals can thus receive different rewards. This knowledge also lets RRT choose which criteria to judge and reduce judge requests.

IRT and learned rubric measurement serve calibration, assessment, data selection, and efficient evaluation. \citet{hashemi2024llm} train a network to calibrate a judge's rubric answers to human ratings. \citet{uto2021multidimensional} uses a Rasch model with item and rater effects to assess examinees. \citet{lalor2019learning} fit IRT to many neural models' responses and filter training data by difficulty. Adaptive testing selects questions from a calibrated pool to measure ability with fewer questions \citep{weiss1982improving}. \citet{truong2025reliable} predict question difficulty from text for adaptive testing of language models. Each measures a fixed examinee or model. RRT applies this measurement to policy training, where each rollout's measured quality becomes its reward. From the prompt and criterion text, RRT predicts each criterion's difficulty and how well it separates responses, so it works on unseen rubrics. Training verdicts then refine these estimates, so they stay current as the policy changes. The same estimates pick the most informative criteria for the current rollouts, so RRT judges fewer criteria per prompt.

%% file: sections/method.tex
\section{Method}
\label{sec:method}

\subsection{IRT for Rubric Criteria}
\label{sec:exam}

For a prompt $q$, let rollouts $i=1,\ldots,N$ be judged against rubric criteria $c_1,\ldots,c_K$. Each verdict is encoded as $G_{ij}\in\{0,1\}$, where $1$ denotes satisfaction of a positive criterion or avoidance of a pitfall (Appendix~\ref{sec:judge-prompts}). Let $w_j>0$ be the available points for criterion $j$. The reward based on rubric points is \(R_i^{\mathrm{base}}=\sum_j w_j\,G_{ij}/\sum_j w_j\in[0,1]\). RRT treats each rubric criterion as an item whose verdict provides evidence about scalar quality \(z_i\) for the rubric. The model assumes that every criterion's pass probability increases strictly with \(z_i\). For a fixed rubric and criterion parameters, the expected reward based on rubric points therefore increases strictly with quality (Theorem~\ref{thm:points-monotone} in Appendix~\ref{sec:points-progress-proofs}). The model also assumes local independence, so verdicts are conditionally independent given quality and criterion parameters \citep{chen2025item}.

Criterion \(j\) has difficulty \(b_j\) and discrimination \(a_j\). The RPN \(\psi\) predicts these parameters from the prompt \(q\) and criterion text \(c_j\), so \((a_j,b_j):=\psi(q,c_j)\) with \(a_j>0\). For rollout quality \(z_i\), define \(u_{ij}:=a_j(z_i-b_j)\) and \(P_{ij}:=F(u_{ij})\). The response function \(F:\mathbb{R}\to(0,1)\) is differentiable and strictly increasing, with \(F(0)=0.5\). Thus \(b_j\) is the quality where \(P_{ij}\) crosses \(0.5\) with slope \(a_jF'(0)\). This item response model has two parameters per criterion and lets discrimination vary \citep{birnbaum1968some}. With a logistic response function and \(a_j=1\) for every criterion, it reduces to the Rasch model with one parameter per criterion \citep{rasch1966item}. Appendix~\ref{sec:response-model-properties} states the assumptions on $F$.

\subsection{Bayesian Reward Inference and Online EM}
\label{sec:bayes}

RRT infers quality with the criterion parameters held fixed. Write $c=(c_1,\ldots,c_K)$ for the rubric and $G_i=(G_{i1},\ldots,G_{iK})$ for rollout $i$'s verdict vector. Since $G_{ij}\mid z_i,q,c_j,\psi\sim\mathrm{Bernoulli}(P_{ij})$, local independence gives
\begin{equation}
  p_\psi(G_i\mid z_i,q,c)=\prod\nolimits_{j=1}^{K}p_\psi(G_{ij}\mid z_i,q,c_j)=\prod\nolimits_{j=1}^{K}P_{ij}^{\,G_{ij}}(1-P_{ij})^{\,1-G_{ij}} .
  \label{eq:likelihood}
\end{equation}
RRT places a fixed Gaussian prior with mean zero and variance $\sigma_z^2$ on quality. If $\ell_{ij}(z)$ is the log of criterion $j$'s Bernoulli factor, Bayes' rule gives
\begin{equation}
  p_\psi(z_i\mid G_i,q,c)\overset{\mathrm{Bayes}}{\propto}p(z_i)\prod\nolimits_{j=1}^{K}P_{ij}^{\,G_{ij}}(1-P_{ij})^{\,1-G_{ij}},\quad \ell_i(z):=\sum\nolimits_{j=1}^{K}\ell_{ij}(z)-z^2/(2\sigma_z^2).
  \label{eq:estep}
\end{equation}
RRT uses the standard Gaussian CDF, $F=\Phi$, and the maximum a posteriori (MAP) reward \(R_i=\hat z_i=\arg\max_z\ell_i(z)\) \citep{mislevy1986bayes}. The Gaussian CDF lets criterion difficulty affect reward ordering, while the logistic CDF orders rollouts by $\sum_j a_jG_{ij}$ (Theorems~\ref{thm:logit} and~\ref{thm:gaussian} in Appendix~\ref{sec:response-model-properties}). Figure~\ref{fig:method-overview} shows how the full verdict pattern determines this reward.

\begin{figure}[!htbp]
\centering
\includegraphics[width=0.9\textwidth]{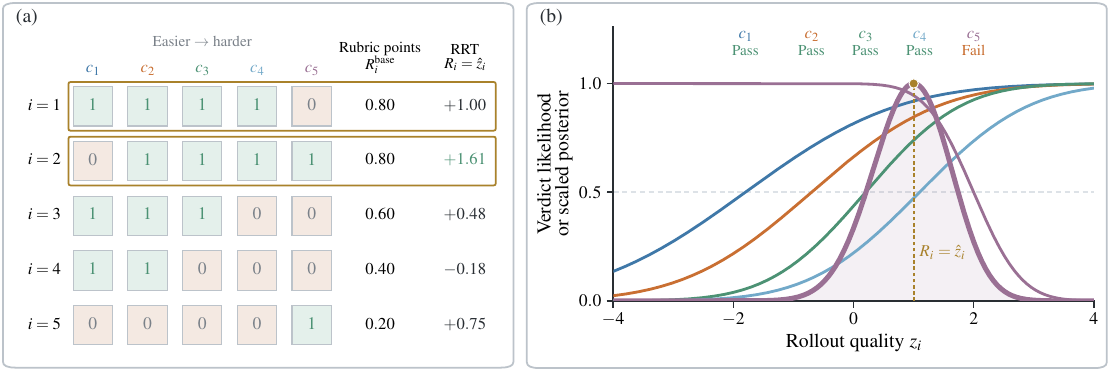}
\caption{RRT reward inference for an illustrative rubric. (a) Verdict vectors for equally weighted criteria ordered by difficulty, with both rewards shown. (b) E-step for rollout $i=1$. The thick curve combines the verdict likelihoods with the Gaussian prior, and its mode is the RRT reward.}
\label{fig:method-overview}
\end{figure}

For fixed criterion parameters, the log posterior is strictly concave and has a unique maximizer. Satisfying an additional criterion increases the MAP reward when the other verdicts and criterion parameters are fixed (Theorem~\ref{thm:concavity} and Corollary~\ref{cor:map-dominance} in Appendix~\ref{sec:map-reward-properties}). RRT finds this maximizer by bisection using the log posterior gradient from Theorem~\ref{thm:general-pass-probability} in Appendix~\ref{sec:response-model-properties}. RRT uses $R_i=\hat z_i$ as the GRPO reward \citep{guo2025deepseek}, while Vanilla GRPO uses $R_i^{\mathrm{base}}$, the normalized points score used in evaluation. Appendix~\ref{sec:grpo-objective} specifies the GRPO update.

As the policy changes during joint training, RRT calibrates $\psi$ from new rollout verdicts using one online EM sweep per policy step \citep{bach2015paired}. The E-step reuses Eq.~\ref{eq:estep} to compute a posterior mode for each rollout. With the posterior modes held fixed, the M-step loss for one prompt group is
\begin{equation}
  \mathcal L(\psi)
  =-\frac{1}{NK}\sum\nolimits_{i=1}^{N}\sum\nolimits_{j=1}^{K}
  \ell_{ij}(\hat z_i;\psi)+\frac{\lambda_a}{K}\sum\nolimits_{j=1}^{K}(\log a_j)^2.
  \label{eq:mstep}
\end{equation}
The regularizer pulls $a_j$ toward one when $\lambda_a>0$. Appendix~\ref{sec:hard-em} gives the EM objective, algorithms, and implementation details.

\subsection{Fisher Information for Aggregation and Selection}
\label{sec:criterion-information}

The item response model used for reward inference also quantifies each criterion's Fisher information about quality. Let $\phi$ be the standard Gaussian density and define the criterion likelihood score as $\mathcal S_{ij}(z_i):=\partial_{z_i}\ell_{ij}(z_i)$.

\begin{theorem}[Fisher information of a rubric criterion]
\label{thm:fisher-criterion}
Under the item response model with $F=\Phi$, the criterion likelihood score has conditional mean zero. Its Fisher information about rollout quality is
\begin{equation}
  I_j(z_i)
  =\mathbb E\!\left[\mathcal S_{ij}(z_i)^2\mid z_i\right]
  =a_j^2\frac{\phi(u_{ij})^2}{P_{ij}(1-P_{ij})}=-\mathbb E\!\left[\frac{\partial^2\ell_{ij}(z_i)}{\partial z_i^2}\,\middle|\,z_i\right].
  \label{eq:fisher-criterion}
\end{equation}
\end{theorem}
Appendix~\ref{sec:fisher-information} gives the proof.
Summing criterion information gives a local bound for rubric aggregation. Write $\mathcal S_i(z_i):=\sum_j\mathcal S_{ij}(z_i)$ for the rubric likelihood score and $I(z_i):=\sum_j I_j(z_i)$ for the rubric information.

\begin{theorem}[Local information bound for rubric aggregation]
\label{thm:local-efficiency}
Under $F=\Phi$, fix a quality level $z_i$ and the criterion parameters, and assume conditionally independent verdicts. For a statistic $T(G_i)$ with $0<\operatorname{Var}(T\mid z_i)<\infty$, define its local SNR for quality as
\begin{equation}
  \operatorname{SNR}_T(z_i):=\big(\partial_{z_i}\mathbb E[T\mid z_i]\big)^2\big/\operatorname{Var}(T\mid z_i).
  \label{eq:local-snr}
\end{equation}
Then $\operatorname{SNR}_T(z_i)\le I(z_i)$, with equality if and only if $T-\mathbb E[T\mid z_i]=\gamma\,\mathcal S_i(z_i)$ almost surely for some $\gamma\neq0$. Up to a nonzero affine map, the rubric likelihood score is therefore the unique scalar verdict signal that achieves this bound. For the reward based on rubric points, equality holds if and only if $w_j\propto a_j\phi(u_{ij})/[P_{ij}(1-P_{ij})]$ for every $j$. Once two criteria have different criterion parameters, no fixed vector of rubric points achieves the bound at every quality level.
\end{theorem}
Appendix~\ref{sec:points-progress-proofs} gives the proof.
The MAP reward combines the rubric likelihood score and prior through realized posterior curvature. A Fisher scoring surrogate uses the expected local precision at a common quality level (Theorem~\ref{thm:local-advantage} and Eq.~\ref{eq:local-advantage} in Appendix~\ref{sec:points-progress-proofs}). Appendix~\ref{sec:local-efficiency} derives the statistic with minimum variance subject to unit local response to quality and the expected local precision.

Fisher information also guides criterion selection under a budget. It measures expected information before judging, while the likelihood score measures evidence from an observed verdict (Figure~\ref{fig:information-profiles} in Appendix~\ref{sec:fisher-information}). Following classical adaptive testing \citep{weiss1982improving}, RRT uses the criterion information from Theorem~\ref{thm:fisher-criterion} to rank unjudged criteria by $\sum_i I_j(\hat z_i)$ at the qualities currently inferred for the rollout group. The qualities are updated after each selected criterion is judged.

%% file: sections/experiments.tex
\section{Experiments}
\label{sec:experiments}

The experiments are designed to address the three following research questions (RQs).

\begin{description}[leftmargin=2.4em,labelsep=0.5em,nosep]
  \item[RQ1:] How does RRT compare with the baselines in verdict prediction, reward variation, and ordering stability within each prompt?
  \item[RQ2:] How does RRT affect criterion score, normalized points score, and training behavior across datasets and policies?
  \item[RQ3:] What tradeoff does selection based on Fisher information provide between judge usage, reward fidelity, and evaluation scores?
\end{description}

\subsection{Experimental Setup}
\label{sec:setup}

The default base policy is Qwen3.5-4B \citep{qwen2026qwen35}, and the datasets are Medical, Science, Rubrics as Rewards (RaR) Science, and RubricBench. Transfer comparisons also use Qwen3.5-2B and Llama-3.1-8B-Instruct \citep{grattafiori2024llama}. Appendix~\ref{sec:data-sources} gives data sources, splits, and roles. Each condition has one training run, and each policy step uses 32 prompts, eight rollouts per prompt, and one proximal policy optimization (PPO) epoch. Each policy comparison keeps the data, base checkpoint, sampler, judge, and GRPO implementation fixed. GPT-5.5 \citep{openai2026gpt55} with reasoning disabled gives one binary verdict per rollout and criterion. Using these verdicts, Vanilla GRPO trains on the reward $R_i^{\mathrm{base}}$ based on rubric points. The main RRT condition applies one stochastic partial M-step per policy step.

Criterion score is the percentage of rubric criteria satisfied. Normalized points score is the weighted fraction of available points. Scores, ROC-AUC values, correlations, rates, and shares are reported as percentages, and differences between them are percentage points. Macro means are unweighted. Policy scores average three evaluation sampling seeds. Reported confidence intervals are 95\% bootstrap intervals over prompt groups, paired when two conditions are compared on the same groups. They quantify variation across prompt groups with the trained policies fixed. Appendices~\ref{sec:judge-prompts} and~\ref{sec:experimental-configuration} give the policy input and judge prompts, full configuration, and checkpoint selection.

\subsection{Verdict Prediction and Reward Stability}
\label{sec:loco}

Leave-one-criterion-out prediction infers rollout quality from the other $K-1$ verdicts. For their verdicts $G_{i,-j}$ and texts $c_{-j}$, the prediction is \(\widehat P_{ij}^{(-j)}=\int F\!\left(a_j(z-b_j)\right)p_\psi\!\left(z\mid G_{i,-j},q,c_{-j}\right)dz\). Under this protocol, four methods predict $G_{ij}$ from the same rollouts. RRT uses a frozen RPN. Its matched $a_j=1$ baseline adjusts $b_j$ to preserve marginal pass rate. The RPN baseline without rollout evidence uses $\mathbb{E}_{z\sim\mathcal N(0,\sigma_z^2)}[F(a_j(z-b_j))]$. The other baseline averages the remaining $K-1$ verdicts. The evaluation reports ROC-AUC within each criterion and pooled ROC-AUC.
Results are averaged over Qwen3.5-2B and Llama-3.1-8B-Instruct. The macro mean covers the four datasets.

\begin{table}[!htbp]
\centering
\caption{Leave-one-criterion-out verdict prediction by ROC-AUC. Results within each criterion use the RPN without rollout evidence as the reference, and pooled results use the mean of the other verdicts. Blue marks RRT, and green parentheses give gains from the reference.}
\label{tab:loco}
\small
\setlength{\tabcolsep}{1.2pt}
\renewcommand{\arraystretch}{0.9}
\begin{tabular}{@{}lcccc@{\hspace{4pt}\vrule width 0.8pt\hspace{4pt}}c@{}}
\toprule
Prediction method & Science & RaR Science & Medical & RubricBench & Macro mean \\
\midrule
\rowcolor{black!4}
\multicolumn{6}{c}{\footnotesize Rollout evidence: ROC-AUC within each criterion} \\
\midrule
\initialpolicy{RPN without rollout evidence} & \initialpolicy{50.0} & \initialpolicy{50.0} & \initialpolicy{50.0} & \initialpolicy{50.0} & \initialpolicy{50.0} \\
Mean of other verdicts & 69.6 & 67.2 & 63.8 & 62.5 & \textbf{65.8} \\
RRT, $a_j=1$ & 69.7 & 67.0 & 63.9 & \textbf{62.8} & \textbf{65.8} \\
\rowcolor{blue!12}
RRT
& \textbf{70.0}\policygain{20.0}
& \textbf{67.4}\policygain{17.4}
& \textbf{64.0}\policygain{14.0}
& 61.8\policygain{11.8}
& \textbf{65.8}\policygain{15.8} \\
\midrule
\rowcolor{black!4}
\multicolumn{6}{c}{\footnotesize Joint rollout and criterion ranking: pooled ROC-AUC} \\
\midrule
\initialpolicy{Mean of other verdicts} & \initialpolicy{69.4} & \initialpolicy{70.6} & \initialpolicy{67.0} & \initialpolicy{70.5} & \initialpolicy{69.3} \\
RPN without rollout evidence & 74.6 & 76.5 & 76.7 & 68.8 & 74.1 \\
RRT, $a_j=1$ & 79.2 & 81.2 & 80.1 & \textbf{77.2} & 79.4 \\
\rowcolor{blue!12}
RRT
& \textbf{79.7}\policygain{10.3}
& \textbf{81.4}\policygain{10.8}
& \textbf{80.3}\policygain{13.3}
& 76.6\policygain{6.1}
& \textbf{79.5}\policygain{10.1} \\
\bottomrule
\end{tabular}
\end{table}

RRT gains 10.1 points in pooled macro ROC-AUC over the mean of other verdicts and also exceeds the RPN without rollout evidence on every dataset (Table~\ref{tab:loco}). All three methods using rollout evidence reach 65.8\% macro ROC-AUC within each criterion.

The RPN configuration ablations infer quality from all verdicts and vary the response function, parameter count, text conditioning, embedder, and policy. The adopted model improves macro ROC-AUC by 0.7 points over the model with $a_j=1$ (Table~\ref{tab:rpn-prediction} in Appendix~\ref{sec:rpn-prediction}). The RPN representation analysis also compares predicted and empirical criterion difficulty, with Spearman correlations of 38.0\% to 47.9\% (Table~\ref{tab:rpn-geometry} in Appendix~\ref{sec:rpn-geometry-appendix}). To compare reward orderings, both response functions use $a_j=1$ and the same criterion difficulties. The Gaussian CDF separates 83.3\% of rollout pairs with equal pass counts on Medical and 39.3\% on Science, while the logistic CDF leaves all such pairs tied (Table~\ref{tab:link-separation} in Appendix~\ref{sec:link-separation}).

The calibration analysis compares online and frozen RPNs in predicting empirical criterion pass rates as the policy changes. Online updates increase macro Pearson correlation by 1.7 points on the next policy step, with gains on both datasets (Table~\ref{tab:em-diagnostics} in Appendix~\ref{sec:em-diagnostics-appendix}). Two criterion trajectories selected post hoc end with lower predicted difficulty and higher empirical pass rates (Figure~\ref{fig:em-case-studies}). The RubricBench RPN is fitted with hard (posterior mode) and soft (full posterior) E-steps at two discrimination regularizer weights. At both weights, the hard E-step improves ROC-AUC by 0.6 to 1.3 points and reduces criterion loss by 0.011 to 0.017 (Table~\ref{tab:marginal-em} in Appendix~\ref{sec:marginal-em}).

Reward variation within prompts is compared between RRT and rubric points across policies and rollout group sizes. Relative to rubric points, RRT has 1.2 to 2.2 times the variance share within prompts and 2\% to 58\% fewer tied pairs in all 12 dataset and policy cells. With groups drawn from 48 rollouts per prompt, the tied pair share falls by 2.2 to 2.3 points on Medical and 1.5 to 1.6 on Science across tested group sizes from 2 to 32, including the training size $n=8$ (Figure~\ref{fig:reward-signal} and Table~\ref{tab:group-size} in Appendix~\ref{sec:reward-signal}). In the companion analysis, 52.3\% to 80.5\% of criteria receive the same verdict across all rollouts in a group (Table~\ref{tab:unanimous-criteria} in Appendix~\ref{sec:unanimous-criteria}).

Variation in judge verdicts is measured with three independent verdicts for each prompt, response, and criterion. Mean disagreement with the consensus verdict is 1.27\% across all 12 dataset and policy cells (Tables~\ref{tab:judge-agreement} and~\ref{tab:consensus-calibration} in Appendix~\ref{sec:judge-agreement-tables}). Stable nonzero ordering share is the fraction of rollout pairs that remain separated and retain their order. On repeated verdicts, RRT with marginal calibration increases its macro mean by 5.4 points over rubric points. Under corruption of the consensus verdicts, mean gains over rubric points are 0.99 to 2.00 points with estimated $a_j$, exceeding gains with $a_j=1$ at every tested level in both corruption channels. At the strongest tested level, 0.10, its mean order flip rate is lower than rubric points whether corruption targets split verdicts or all verdicts (Tables~\ref{tab:judge-reward-robustness} and~\ref{tab:consensus-noise} in Appendix~\ref{sec:judge-reward-stability}).

The comparison of noise structures measures cosine similarity to advantages from clean criterion scores under controlled corruption. At level 0.20, RRT with marginal calibration gains 1.3 to 5.1 points in mean similarity over criterion score across four channels that corrupt criterion subsets. Criterion score leads by 0.8 to 1.2 points under lenient noise tied to response length and symmetric noise on all criteria (Table~\ref{tab:noise-coverage} in Appendix~\ref{sec:noise-robustness-tables}). The companion comparison varies corruption level on a fixed 30\% criterion subset and includes POW3R and DIVA. At each tested nonzero level, RRT with marginal calibration has the highest similarity on every dataset, with gains of 3.0 to 7.0 points over criterion score at 0.20 (Table~\ref{tab:noise-doseresponse} in Appendix~\ref{sec:noise-robustness-tables}).

Variation across rollout samples is measured with pass rates from disjoint blocks of eight rollouts. Reliability, measured by Pearson correlation between blocks, is 93.6\% to 96.0\% over all criteria and 70.9\% to 71.9\% over uncertain criteria, whose empirical pass rates over 48 rollouts lie between 10\% and 90\% (Table~\ref{tab:passrate-reliability} in Appendix~\ref{sec:passrate-noise}). For criterion sampling, orderings from random rubric halves are compared using Kendall's $\tau$, with RRT parameters from the online RPN. RRT raises rank correlation over rubric points by 15.1 points on RaR Science and 0.5 to 3.6 points on the other datasets (Table~\ref{tab:split-half} in Appendix~\ref{sec:split-half-table}). Across quintiles of rank correlation from rubric points, RRT's gain falls from 14 points in the lowest quintile to 1 point in the highest (Table~\ref{tab:split-half-dose} in the same appendix).

To check local independence, the analysis measures dependence between criterion verdicts after accounting for fitted quality. Fitted quality explains 72.8\% to 85.3\% of pairwise mutual information (Table~\ref{tab:local-independence-detail} in Appendix~\ref{sec:local-independence-details}). For pairs with the most similar criterion text, the share with residual redundancy exceeds the bootstrap null rate by 3.4 to 11.9 points. For the least similar text, it stays within 0.7 points of the null rate.

\subsection{Policy Training Results}
\label{sec:policy-results}

This experiment compares policy performance across rubric rewards and criterion parameter settings. Baseline rewards include POW3R \citep{tyagi2026not} and DIVA \citep{cook2026check} (Appendix~\ref{sec:baseline-rewards}). Empirical pass rate conditions use $a_j=1$ and $b_j=1-2\bar G_j$, using pass rates from batch or cached rollouts (Appendix~\ref{sec:passrate-noise}). RRT + frozen RPN and RRT + online RPN predict both $a_j$ and $b_j$ from text.

\begin{table}[!htbp]
\centering
\caption{Criterion and normalized points scores. Bold marks the highest score in each column. RRT variants use parameters from empirical pass rates or the RPN, with the same reward and E-step. RubricBench uses one point per criterion, so both scores are equal.}
\label{tab:policy-training-results}
\small
\setlength{\tabcolsep}{0.8pt}
\renewcommand{\arraystretch}{0.9}
\begin{tabular}{@{}lccccccc@{\hspace{7pt}\vrule width 0.5pt\hspace{7pt}}cc@{}}
\toprule
& \multicolumn{2}{c}{Medical}
& \multicolumn{2}{c}{Science}
& \multicolumn{2}{c}{RaR Science}
& RubricBench
& \multicolumn{2}{c}{Macro mean} \\
\cmidrule(lr){2-3}\cmidrule(lr){4-5}\cmidrule(lr){6-7}
\cmidrule(lr){8-8}\cmidrule(lr){9-10}
Condition & \shortstack{\tiny Criterion\\\tiny score} & \shortstack{\tiny Normalized\\\tiny points score}
& \shortstack{\tiny Criterion\\\tiny score} & \shortstack{\tiny Normalized\\\tiny points score}
& \shortstack{\tiny Criterion\\\tiny score} & \shortstack{\tiny Normalized\\\tiny points score}
& \shortstack{\tiny Criterion\\\tiny score}
& \shortstack{\tiny Criterion\\\tiny score} & \shortstack{\tiny Normalized\\\tiny points score} \\
\midrule
\initialpolicy{Base policy}
& \initialpolicy{64.1} & \initialpolicy{63.3}
& \initialpolicy{76.5} & \initialpolicy{70.9}
& \initialpolicy{68.7} & \initialpolicy{71.0}
& \initialpolicy{74.7}
& \initialpolicy{71.0} & \initialpolicy{70.0} \\

Vanilla GRPO
& 67.7\policygain{3.6} & 66.8\policygain{3.5}
& 78.3\policygain{1.8} & 77.8\policygain{6.9}
& 70.6\policygain{1.9} & 73.7\policygain{2.7}
& 77.0\policygain{2.3}
& 73.4\policygain{2.4} & 73.8\policygain{3.8} \\

POW3R
& 70.6\policygain{6.5} & 69.6\policygain{6.3}
& 79.0\policygain{2.5} & 78.7\policygain{7.8}
& 70.1\policygain{1.4} & 73.2\policygain{2.2}
& 77.4\policygain{2.7}
& 74.3\policygain{3.3} & 74.7\policygain{4.7} \\

DIVA
& 64.9\policygain{0.8} & 64.4\policygain{1.1}
& 78.8\policygain{2.3} & 78.3\policygain{7.4}
& 70.6\policygain{1.9} & 73.8\policygain{2.8}
& 75.8\policygain{1.1}
& 72.5\policygain{1.5} & 73.1\policygain{3.1} \\

\midrule
\multicolumn{10}{c}{\footnotesize Same RRT reward $R_i=\hat z_i$ and E-step across criterion parameter settings} \\
\midrule

\multicolumn{10}{l}{RRT +} \\

\hspace{1em}\(\hookrightarrow\) batch pass rate
& 68.0\policygain{3.9} & 67.0\policygain{3.7}
& 78.3\policygain{1.8} & 77.8\policygain{6.9}
& 70.8\policygain{2.1} & 73.5\policygain{2.5}
& 77.3\policygain{2.6}
& 73.6\policygain{2.6} & 73.9\policygain{3.9} \\

\hspace{1em}\(\hookrightarrow\) cached pass rate
& 67.8\policygain{3.7} & 66.8\policygain{3.5}
& 78.3\policygain{1.8} & 77.9\policygain{7.0}
& 70.0\policygain{1.3} & 73.0\policygain{2.0}
& 77.4\policygain{2.7}
& 73.4\policygain{2.4} & 73.8\policygain{3.8} \\

\hspace{1em}\(\hookrightarrow\) frozen RPN
& 70.1\policygain{6.0} & \textbf{69.8}\policygain{6.5}
& 79.5\policygain{3.0} & 79.2\policygain{8.3}
& 70.7\policygain{2.0} & 73.8\policygain{2.8}
& 77.9\policygain{3.2}
& 74.6\policygain{3.6} & 75.2\policygain{5.2} \\

\rowcolor{blue!12}
\hspace{1em}\(\hookrightarrow\) online RPN
& \textbf{70.9}\policygain{6.8} & \textbf{69.8}\policygain{6.5}
& \textbf{79.8}\policygain{3.3} & \textbf{79.4}\policygain{8.5}
& \textbf{70.9}\policygain{2.2} & \textbf{74.0}\policygain{3.0}
& \textbf{78.7}\policygain{4.0}
& \textbf{75.1}\policygain{4.1} & \textbf{75.5}\policygain{5.5} \\
\bottomrule
\end{tabular}
\end{table}

RRT + online RPN is highest or tied in every column of Table~\ref{tab:policy-training-results}, exceeding Vanilla GRPO by 1.7 and POW3R by 0.8 points on both macro metrics. The frozen RPN is 1.0 to 1.2 macro criterion points above empirical pass rate conditions, and the online RPN scores 0.5 points above the frozen RPN.

RRT's macro criterion score is 0.2 points below Vanilla GRPO on Qwen3.5-2B and 0.1 points above it on Llama-3.1-8B-Instruct. It scores higher on RubricBench with both policies, as on Qwen3.5-4B (Table~\ref{tab:policy-transfer-results} in Appendix~\ref{sec:policy-transfer-appendix}). Difficulty bands are defined by empirical criterion difficulty under the base policy. RRT + online RPN exceeds Vanilla GRPO by 2.8 to 5.6 points in seven of eight bands on Medical and Science, including every Medium, Hard, and Very hard band (Figure~\ref{fig:criterion-difficulty} in Appendix~\ref{sec:criterion-difficulty-appendix}). Policies trained on Medical and Science are also evaluated on HealthBench and ResearchQA, respectively. RRT gains 0.1 to 0.7 macro criterion points over Vanilla GRPO (Table~\ref{tab:cross-benchmark-results} in Appendix~\ref{sec:cross-benchmark-generalization}). The policy comparison also measures response length. RRT has lower median response length in all 12 dataset and policy combinations, with the mean of dataset medians 10.6\% to 47.9\% below Vanilla GRPO across policies (Figure~\ref{fig:response-length} in Appendix~\ref{sec:response-length-figure}).

\subsection{Criterion Selection Based on Fisher Information}
\label{sec:fisher-judging}

The comparison finds the smallest criterion budget reaching 95.0\% reward fidelity, the mean Pearson correlation between GRPO advantages from partial and full judging. All $K$ criteria give $\mathbf A^{(K)}$ and $m$ selected criteria give $\mathbf A^{(m)}$. Methods share a frozen RPN and verdict matrix per dataset. The selection methods are random, discrimination ($a_j^2$), static Fisher ($N I_j(0)$), and adaptive Fisher selection.

\begin{table}[!htbp]
\centering
\caption{Unjudged criteria at the smallest budget reaching 95.0\% correlation with GRPO advantages from the full rubric on rollouts from trained policies. Parentheses give gains over random selection.}
\label{tab:fisher-judging}
\small
\setlength{\tabcolsep}{2pt}
\renewcommand{\arraystretch}{0.9}
\begin{tabular}{@{}lcccc@{\hspace{5pt}\vrule width 0.5pt\hspace{5pt}}c@{}}
\toprule
Selection method & Medical & Science & RaR Science & RubricBench & Macro mean \\
\midrule
Random & 10.7\% & 15.2\% & 5.1\% & 8.9\% & 10.0\% \\
Discrimination ($a_j^2$) & 20.0\%\policygain{9.3} & 18.0\%\policygain{2.7} & 5.1\%\policyeven{0.0} & 8.9\%\policyeven{0.0} & 13.0\%\policygain{3.0} \\
Static Fisher & 20.0\%\policygain{9.3} & 20.5\%\policygain{5.3} & 18.3\%\policygain{13.2} & 18.6\%\policygain{9.6} & 19.3\%\policygain{9.3} \\
\rowcolor{blue!12}
Adaptive Fisher & 20.8\%\policygain{10.1} & 22.9\%\policygain{7.6} & 18.3\%\policygain{13.2} & 21.9\%\policygain{13.0} & 21.0\%\policygain{11.0} \\
\bottomrule
\end{tabular}
\end{table}

On rollouts from trained policies, adaptive Fisher leaves 11.0 points more criteria unjudged than random selection and 1.7 more than static Fisher in the macro mean (Table~\ref{tab:fisher-judging}). Discrimination alone matches random selection on RaR Science and RubricBench, where static Fisher leaves 13.2 and 9.6 points more criteria unjudged, respectively. On rollouts from base policies, macro gains over random selection are 10.6 points for static Fisher and 9.5 for adaptive Fisher (Table~\ref{tab:fisher-judging-base} in Appendix~\ref{sec:fisher-judging-extra}).

To measure how these budgets affect policy training, this experiment compares criterion score and normalized points scores under full and partial judging. All RRT conditions share the frozen RPN. Random selection uses a criterion budget of 0.50 for Vanilla GRPO and RRT. Adaptive Fisher budgets are 0.50, 0.80, and 0.95, with full judging for reference.

\begin{table}[!htbp]
\centering
\caption{Scores under full and partial judging. Budget is the fraction of criteria judged. Parentheses give differences from Vanilla GRPO with full judging. Only criterion score is shown for RubricBench.}
\label{tab:fisher-budget-training}
\small
\setlength{\tabcolsep}{0.8pt}
\renewcommand{\arraystretch}{0.9}
\resizebox{\textwidth}{!}{%
\begin{tabular}{@{}lccccccc@{\hspace{7pt}\vrule width 0.5pt\hspace{7pt}}cc@{}}
\toprule
& \multicolumn{2}{c}{Medical} & \multicolumn{2}{c}{Science} & \multicolumn{2}{c}{RaR Science} & RubricBench & \multicolumn{2}{c}{Macro mean} \\
\cmidrule(lr){2-3}\cmidrule(lr){4-5}\cmidrule(lr){6-7}\cmidrule(lr){8-8}\cmidrule(lr){9-10}
Condition & \shortstack{\tiny Criterion\\\tiny score} & \shortstack{\tiny Normalized\\\tiny points score}
& \shortstack{\tiny Criterion\\\tiny score} & \shortstack{\tiny Normalized\\\tiny points score}
& \shortstack{\tiny Criterion\\\tiny score} & \shortstack{\tiny Normalized\\\tiny points score}
& \shortstack{\tiny Criterion\\\tiny score}
& \shortstack{\tiny Criterion\\\tiny score} & \shortstack{\tiny Normalized\\\tiny points score} \\
\midrule
Vanilla GRPO & 67.7 & 66.8 & 78.3 & 77.8 & 70.6 & 73.7 & 77.0 & 73.4 & 73.8 \\
\midrule
RRT, full judging (1.00)
& 70.1\policygain{2.4} & 69.8\policygain{3.0}
& 79.5\policygain{1.2} & 79.2\policygain{1.4}
& 70.7\policygain{0.1} & 73.8\policygain{0.1}
& 77.9\policygain{0.9}
& 74.6\policygain{1.2} & 75.2\policygain{1.4} \\
\midrule
\multicolumn{10}{l}{Random selection (0.50)} \\
\hspace{1em}\(\hookrightarrow\) Vanilla GRPO
& 65.9\policyloss{1.8} & 64.0\policyloss{2.8}
& 74.7\policyloss{3.6} & 74.8\policyloss{3.0}
& 67.5\policyloss{3.1} & 70.3\policyloss{3.4}
& 76.3\policyloss{0.7}
& 71.1\policyloss{2.3} & 71.4\policyloss{2.5} \\
\hspace{1em}\(\hookrightarrow\) RRT
& 67.8\policygain{0.1} & 67.0\policygain{0.2}
& 77.2\policyloss{1.1} & 77.0\policyloss{0.8}
& 70.3\policyloss{0.3} & 72.2\policyloss{1.5}
& 76.5\policyloss{0.5}
& 73.0\policyloss{0.5} & 73.2\policyloss{0.7} \\
\midrule
\multicolumn{10}{l}{RRT, adaptive Fisher} \\

\hspace{1em}\(\hookrightarrow\) 0.95
& 70.1\policygain{2.4} & 69.8\policygain{3.0}
& 79.4\policygain{1.1} & 79.1\policygain{1.3}
& 70.7\policygain{0.1} & 73.8\policygain{0.1}
& 77.7\policygain{0.7}
& 74.5\policygain{1.1} & 75.1\policygain{1.2} \\
\hspace{1em}\(\hookrightarrow\) 0.80
& 69.8\policygain{2.1} & 69.4\policygain{2.6}
& 79.2\policygain{0.9} & 78.8\policygain{1.0}
& 70.5\policyloss{0.1} & 73.5\policyloss{0.2}
& 77.7\policygain{0.7}
& 74.3\policygain{0.9} & 74.9\policygain{1.0} \\
\hspace{1em}\(\hookrightarrow\) 0.50
& 68.6\policygain{0.9} & 68.2\policygain{1.4}
& 78.2\policyloss{0.1} & 77.8\policyeven{0.0}
& 69.8\policyloss{0.8} & 72.8\policyloss{0.9}
& 76.8\policyloss{0.2}
& 73.4\policyloss{0.1} & 73.9\policygain{0.1} \\
\bottomrule
\end{tabular}%
}
\end{table}

At criterion budget 0.50, adaptive Fisher selection keeps both RRT macro scores within 0.1 points of Vanilla GRPO with full judging (Table~\ref{tab:fisher-budget-training}). Relative to each method's full judging score, random selection at this budget lowers macro criterion scores by 2.3 points for Vanilla GRPO and 1.6 for RRT. Adaptive Fisher selection at this budget reduces RRT's judge requests by 49.0\% on Medical and Science relative to full judging. Separate API measurements yield parseable verdicts for all 11,712 completed requests (Tables~\ref{tab:judge-cost-overhead} and~\ref{tab:judge-api-usage} in Appendix~\ref{sec:judge-cost}).

The cost analysis measures generation, judging, and total time per policy step. Judging is the longest measured stage in every tested condition. Relative to full judging with the same frozen RPN, adaptive Fisher selection at criterion budget 0.50 reduces median judging time by 49.1\% to 49.6\% and total step time by 23.5\% to 28.7\% on Medical and Science. Generation time stays within 1.0\% of full judging (Table~\ref{tab:step-time-cost} in Appendix~\ref{sec:step-time-cost}). Added computation for RRT + online RPN is 0.123\% of a Vanilla GRPO Medical step (Table~\ref{tab:rrt-overhead}). One RPN warm start per dataset serves all RPN variants. Embedding and fitting from cached verdicts cost 0.89 to 7.64 GPU hours across datasets and embedder sizes (Table~\ref{tab:warm-start-cost} in Appendix~\ref{sec:warm-start-cost}). The trained policy adds zero parameters or inference components at deployment (Table~\ref{tab:inference-cost} in Appendix~\ref{sec:inference-cost}).

%% file: sections/discussion.tex
\section{Discussion and Conclusion}
\label{sec:discussion}

RRT gives GRPO access to differences between verdict patterns that point totals discard (RQ1, Proposition~\ref{prop:points-ties} in Appendix~\ref{sec:points-progress-proofs}). Fewer ties preserve these distinctions across policies and rollout group sizes (Appendix~\ref{sec:reward-signal}). The rubric likelihood score maximizes local SNR for quality under the item response model (Theorem~\ref{thm:local-efficiency}). This supports likelihood aggregation because differences in criterion parameters make the optimal relative weights vary with quality. MAP rewards also depend on posterior curvature from realized verdicts (Theorem~\ref{thm:local-advantage} in Appendix~\ref{sec:points-progress-proofs}). Unanimous criteria, common in rollout groups (Appendix~\ref{sec:unanimous-criteria}), can change MAP reward gaps through shared likelihood factors. RRT can thus use the full rubric, even unanimous criteria, to shape GRPO advantages.

Policy score gains with an RPN support estimating criterion parameters from text rather than empirical pass rates alone (Table~\ref{tab:policy-training-results}). Criterion parameters complement rollout evidence in pooled verdict prediction (Table~\ref{tab:loco}). Less reliable empirical pass rates on uncertain criteria further support combining text and rollout evidence (Appendix~\ref{sec:passrate-noise}). Text is informative about criterion difficulty, as predicted and empirical difficulty correlate positively (Table~\ref{tab:rpn-geometry} in Appendix~\ref{sec:rpn-geometry-appendix}). Estimating discrimination improves mean reward stability over fixed discrimination under corruption with marginal calibration (Appendix~\ref{sec:judge-reward-stability}). Theorems~\ref{thm:logit} and~\ref{thm:gaussian} in Appendix~\ref{sec:response-model-properties} motivate the Gaussian CDF: difficulty can change reward ordering, while logistic ordering depends only on totals weighted by discrimination. At $a_j=1$ and matched difficulties, the Gaussian CDF separates observed rollout pairs with equal pass counts (Table~\ref{tab:link-separation} in Appendix~\ref{sec:link-separation}). As the policy changes, online EM reuses GRPO verdicts to update the RPN. These updates improve correlation between predicted pass probabilities and empirical criterion pass rates at the next policy step (Table~\ref{tab:em-diagnostics} in Appendix~\ref{sec:em-diagnostics-appendix}). The hard E-step has lower criterion loss than the tested soft variants (Table~\ref{tab:marginal-em} in Appendix~\ref{sec:marginal-em}). These results support an RPN warm start followed by online calibration from current rollout verdicts.

With criterion parameters fitted per rubric by marginal calibration, RRT preserves reward distinctions under judge variation. Its mean stable nonzero ordering share exceeds that of rubric points under repeated judging and at every tested corruption level. It combines fewer ties with better order preservation under repeated judging and has a lower mean order flip rate at the strongest tested corruption level (Tables~\ref{tab:judge-reward-robustness} and~\ref{tab:consensus-noise} in Appendix~\ref{sec:judge-reward-stability}). RRT with marginal calibration better preserves advantages from clean criterion scores when errors concentrate on particular criteria, while symmetric errors on all criteria and lenient errors tied to response length favor criterion score (Table~\ref{tab:noise-coverage} in Appendix~\ref{sec:noise-robustness-tables}). With criterion parameters from the online RPN, RRT is less sensitive to criterion sampling than rubric points on all tested datasets. The largest gain occurs in the lowest quintile of rank correlation under rubric points (Tables~\ref{tab:split-half} and~\ref{tab:split-half-dose} in Appendix~\ref{sec:split-half-table}). Dependence diagnostics support the item response model by showing that fitted quality explains most pairwise dependence, with residual redundancy concentrated among criteria with similar text (Table~\ref{tab:local-independence-detail} in Appendix~\ref{sec:local-independence-details}). Modeling this dependence could address redundancy and known bias in unidimensional IRT models \citep{yen1984effects}. Multiple quality targets could also accommodate rubrics with explicit tradeoffs.

RRT + online RPN improves criterion satisfaction over Vanilla GRPO, POW3R, and DIVA in the primary Qwen3.5-4B comparison (RQ2, Table~\ref{tab:policy-training-results}). It exceeds Vanilla GRPO on criteria the base policy often misses (Figure~\ref{fig:criterion-difficulty} in Appendix~\ref{sec:criterion-difficulty-appendix}). Gains on other benchmarks support generalization beyond the training rubrics (Table~\ref{tab:cross-benchmark-results} in Appendix~\ref{sec:cross-benchmark-generalization}). Across the tested policies, RRT produces shorter responses. On Qwen3.5-2B and Llama-3.1-8B-Instruct, its macro criterion scores remain close to Vanilla GRPO's (Appendices~\ref{sec:policy-transfer-appendix} and~\ref{sec:response-length-figure}).

Fisher information gives a common basis for reward inference and criterion selection (RQ3, Theorem~\ref{thm:fisher-criterion}). RRT connects the information principle used in adaptive testing and model evaluation \citep{weiss1982improving,truong2025reliable} to criterion selection during policy training with rubrics. Difficulty locates each criterion's informative quality range, and discrimination controls its information peak (Theorem~\ref{thm:information-frontier} in Appendix~\ref{sec:fisher-information}). Fisher selection's savings over random selection at matched reward fidelity show that these parameters can reduce judging while approximating GRPO advantages (Table~\ref{tab:fisher-judging}, Appendix~\ref{sec:fisher-judging-extra}). Discrimination alone matches random selection on RaR Science and RubricBench, while static Fisher selection reduces judging on both (Table~\ref{tab:fisher-judging}). This supports selection using difficulty as well as discrimination. Mean savings are greatest with static Fisher selection on rollouts from base policies and adaptive Fisher selection on those from trained policies (Tables~\ref{tab:fisher-judging} and~\ref{tab:fisher-judging-base}). This shift supports adapting judge allocation as the policy changes.

At half the criterion budget, RRT's macro criterion score falls less than Vanilla GRPO's under random selection. Adaptive Fisher selection keeps RRT's macro scores close to Vanilla GRPO with full judging at about half the judge requests (Table~\ref{tab:fisher-budget-training}, Appendix~\ref{sec:judge-cost}). Larger criterion budgets recover more of RRT's gains from full judging (Table~\ref{tab:fisher-budget-training}). Since judging is the longest measured stage of a policy step, selecting fewer criteria shortens steps with nearly unchanged generation time (Table~\ref{tab:step-time-cost} in Appendix~\ref{sec:step-time-cost}). Online RPN updates add little computation (Table~\ref{tab:rrt-overhead}). One RPN warm start per dataset serves all RPN variants (Appendix~\ref{sec:warm-start-cost}). These costs remain in training: RRT adds no policy parameters or inference components at deployment (Appendix~\ref{sec:inference-cost}).

Judging cost and time are a major challenge in reinforcement learning with rubrics. For rubrics whose criteria are monotone indicators of a shared target, RRT addresses this challenge with a common model in which difficulty and discrimination determine the evidence in each verdict and the expected information of each unjudged criterion. RRT thus distinguishes verdict patterns with the same point total and selects informative criteria. This reduces judge usage and shortens policy steps while retaining the gains from Vanilla GRPO. Beyond these savings, RRT improves criterion satisfaction, including on difficult criteria, and stabilizes reward orderings under repeated judging with marginal calibration. The deployed policy has no added parameters or inference components.

%% file: sections/appendix.tex
\appendix

\section{Theoretical Results and Proofs}
\label{sec:theorems}

\subsection{Response Function Assumptions and Properties}
\label{sec:response-model-properties}

The response function $F:\mathbb R\to(0,1)$ of Section~\ref{sec:exam} maps the signed margin $a_j(z_i-b_j)$ to a probability. It must satisfy two conditions.
\begin{itemize}
  \item $F(t)\to0$ as $t\to-\infty$, $F(t)\to1$ as $t\to+\infty$, and $F$ is strictly increasing. Increasing a rollout's quality for the rubric never lowers its chance of satisfying a criterion.
  \item $F(0)=0.5$, so $z_i=b_j\iff P_{ij}=0.5$. Thus $b_j$ is the criterion difficulty.
\end{itemize}
The logistic CDF $\sigma$ and the Gaussian CDF $\Phi$ both satisfy them. RRT adopts $\Phi$.

Fix a rollout $i$ and use the signed margin $u_{ij}$, with $a_j>0$ and prior variance $\sigma_z^2>0$, and assume that $F$ is differentiable. Define the response weight function
\begin{equation}
  s_F(u):=\frac{F'(u)}{F(u)(1-F(u))}
  =\frac{d}{du}\operatorname{logit}F(u).
  \label{eq:general-response-weight}
\end{equation}

\begin{theorem}[Log posterior gradient for general $F$]
\label{thm:general-pass-probability}
The log posterior objective $\ell_i$ for rollout $i$ in Eq.~\ref{eq:estep} has gradient
\begin{equation}
  \ell_i'(z_i)=\sum_j a_j\,s_F(u_{ij})\big(G_{ij}-F(u_{ij})\big)-\frac{z_i}{\sigma_z^2}.
  \label{eq:estep-grad}
\end{equation}
\end{theorem}

\begin{proof}
Differentiating the log likelihood $\ell_{ij}$ of criterion $j$ with respect to its margin $u_{ij}$ gives
\[
  \frac{\partial\ell_{ij}}{\partial u_{ij}}
  =F'(u_{ij})\left[\frac{G_{ij}}{F(u_{ij})}-\frac{1-G_{ij}}{1-F(u_{ij})}\right]
  =s_F(u_{ij})\big(G_{ij}-F(u_{ij})\big).
\]
Multiplying by $\partial u_{ij}/\partial z_i=a_j$, summing over $j$, and adding the prior derivative $-z_i/\sigma_z^2$ gives Eq.~\ref{eq:estep-grad}.
\end{proof}

\begin{theorem}[Difficulty dependence of the verdict coefficient]
\label{thm:logit}
Let $F$ be twice continuously differentiable. The coefficient of $G_{ij}$ in Eq.~\ref{eq:estep-grad} is $a_j s_F(u_{ij})$, and
\begin{equation}
  \frac{\partial}{\partial b_j}\big[a_j s_F(u_{ij})\big]
  =-a_j^2\,s_F'(u_{ij}).
  \label{eq:difficulty-weight-derivative}
\end{equation}
This coefficient is independent of $b_j$ for all margins if and only if $s_F$ is constant. For a centered $F$, this holds if and only if $F(u)=\sigma(\gamma_Fu)$ for some $\gamma_F>0$. In particular, for $F=\sigma$, the reward order depends only on $\sum_j a_jG_{ij}$ and not on $b_j$.
\end{theorem}

\begin{proof}
Theorem~\ref{thm:general-pass-probability} gives the coefficient $a_j s_F(u_{ij})$, and $\partial u_{ij}/\partial b_j=-a_j$ yields Eq.~\ref{eq:difficulty-weight-derivative}, which vanishes for every margin if and only if $s_F'\equiv0$. Since $s_F=\frac{d}{du}\operatorname{logit}F$, a constant $s_F\equiv\gamma_F$ integrates to $\operatorname{logit}F(u)=\gamma_Fu+C$, and $F(0)=1/2$ forces $C=0$, so $F=\sigma(\gamma_Fu)$. Strict increase requires $\gamma_F>0$. The converse follows by differentiation.

For $F=\sigma$, $\sigma'=\sigma(1-\sigma)$ gives $s_\sigma\equiv1$, so $\ell_i'(z_i)=\sum_j a_j G_{ij}-g(z_i)$ with
\[
  g(z_i):=\sum_j a_j\sigma(u_{ij})+\frac{z_i}{\sigma_z^2}.
\]
The map $g$ does not involve the verdicts. Each $a_j\sigma(u_{ij})$ is nondecreasing in $z_i$ and the term $z_i/\sigma_z^2$ is strictly increasing with range $\mathbb R$, so $g$ is a strictly increasing bijection. The stationary condition $g(\hat z_i)=\sum_j a_j G_{ij}$ then has a unique solution that increases with its right side, the same map for all rollouts. Hence the reward order is the order of $\sum_j a_j G_{ij}$, which contains no $b_j$.
\end{proof}

For a scaled $F(u)=\sigma(\gamma_Fu)$ the coefficient scale is $s_F\equiv\gamma_F$, absorbed into $a_j$. The analysis uses $\gamma_F=1$. Theorem~\ref{thm:logit} concerns the order of the MAP rewards. Under every $F$, including $\sigma$, difficulties can change reward gaps and normalized GRPO advantages when a group has at least three distinct weighted verdict totals.

\begin{theorem}[Effect of $F=\Phi$]
\label{thm:gaussian}
Suppose $F(u)=\Phi(u)$, with standard Gaussian density $\phi$. Then
\begin{equation}
  s_\Phi(u)=\frac{\phi(u)}{\Phi(u)(1-\Phi(u))}=\lambda(u)+\lambda(-u),
  \qquad
  \lambda(u):=\frac{\phi(u)}{\Phi(u)}.
  \label{eq:gaussian-cdf-response-weight}
\end{equation}
The response weight is even and nonconstant, satisfies $s_\Phi(0)=4/\sqrt{2\pi}$, and obeys $s_\Phi(u)\sim|u|$ as $|u|\to\infty$. Thus $b_j$ changes the verdict coefficient through $u_{ij}$. Moreover, a change in one difficulty can reverse the RRT reward order of two fixed verdict vectors.
\end{theorem}

\begin{proof}
Substituting $F=\Phi$, so $F'=\phi$, into Eq.~\ref{eq:general-response-weight}, and using $1-\Phi(u)=\Phi(-u)$, gives Eq.~\ref{eq:gaussian-cdf-response-weight}. The corresponding signed criterion likelihood score in Eq.~\ref{eq:estep-grad} is
\[
  a_j
  \begin{cases}
    \lambda(u_{ij}), & G_{ij}=1,\\
    -\lambda(-u_{ij}), & G_{ij}=0.
  \end{cases}
\]
Because $\lambda$ is decreasing, raising $b_j$ decreases $u_{ij}$ and increases the positive contribution of a satisfied criterion while decreasing the magnitude of the negative contribution of an unsatisfied criterion. Passing a harder criterion therefore gives a larger positive likelihood score. Failing a harder criterion gives a negative likelihood score with smaller magnitude.

At $u=0$, $s_\Phi(0)=\phi(0)/\Phi(0)^2=4/\sqrt{2\pi}$. Mills' ratio gives $1-\Phi(u)\sim\phi(u)/u$ as $u\to+\infty$, so
\[
  s_\Phi(u)
  =\frac{\phi(u)}{\Phi(u)(1-\Phi(u))}
  \sim u.
\]
Evenness gives the corresponding $|u|$ asymptotic in the left tail.

For the rank claim, take three criteria with $\sigma_z^2=1$, $a_j=1$, $b_2=b_3=0$, and let $b_1=d$. Compare the fixed verdict vectors
\[
  G^A=(0,1,1),
  \qquad
  G^B=(1,0,0).
\]
At $d=0$, the log posterior derivatives are $\ell_A'(0)=\lambda(0)>0$ and $\ell_B'(0)=-\lambda(0)<0$. The derivatives are strictly decreasing, so $\hat z_A>0>\hat z_B$. As $d\to\infty$, the derivative for $A$ is
\[
  -\lambda(d-z)+2\lambda(z)-z
  \le 2\lambda(z)-z,
\]
so its root stays bounded above by the finite root of $2\lambda(z)-z=0$. The derivative for $B$ is
\[
  \lambda(z-d)-2\lambda(-z)-z.
\]
If its root stayed bounded, then $\lambda(z-d)\sim d-z$ would diverge while the other terms stayed bounded. Also, $\ell_B'(0)=\lambda(-d)-2\lambda(0)>0$ for sufficiently large $d$, so the root is positive. Hence $\hat z_B\to\infty$. For sufficiently large $d$, $\hat z_B>\hat z_A$, so changing only $b_1$ reverses the order.
\end{proof}

The Gaussian CDF tail behavior makes the magnitude of an unexpected criterion likelihood score unbounded in the theoretical model. For an unexpected hard pass, $a_j\lambda(u_{ij})\sim a_j|u_{ij}|$ as $u_{ij}\to-\infty$. An unexpected easy failure has the same asymptotic magnitude in the opposite direction. By contrast, under $F=\sigma$ each criterion likelihood score has magnitude at most $a_j$. For a mastered criterion with $u_{ij}\gg0$, the expected pass likelihood score tends to zero while the magnitude of an unexpected failure score grows with $u_{ij}$.

\subsection{MAP Reward Derivation and Properties}
\label{sec:map-reward-properties}

The fixed quality prior is
\begin{equation}
  p(z_i)=\frac{1}{\sqrt{2\pi}\,\sigma_z}\,\exp\!\Big(-\frac{z_i^2}{2\sigma_z^2}\Big).
  \label{eq:prior}
\end{equation}
Bayes' rule turns the likelihood of Eq.~\ref{eq:likelihood} and the prior of Eq.~\ref{eq:prior} into the posterior over quality,
\[
  p_\psi(z_i\mid G_i,q,c)=\frac{p_\psi(G_i\mid z_i,q,c)\,p(z_i)}{p_\psi(G_i\mid q,c)} .
\]
The denominator $p_\psi(G_i\mid q,c)$ does not depend on $z_i$, so the most probable quality maximizes the numerator,
\begin{equation}
  \hat z_i=\arg\max_{z_i}\ p_\psi(z_i\mid G_i,q,c)=\arg\max_{z_i}\ p_\psi(G_i\mid z_i,q,c)\,p(z_i).
  \label{eq:argmax}
\end{equation}
Taking a logarithm turns the product into a sum without moving the maximizer. The log likelihood of rollout $i$'s verdict on criterion $j$ at quality $z$, the logarithm of that criterion's Bernoulli factor in Eq.~\ref{eq:likelihood}, is
\[
  \ell_{ij}(z)=G_{ij}\log F\big(a_j(z-b_j)\big){}+(1-G_{ij})\log\big(1-F(a_j(z-b_j))\big),
\]
so the likelihood contributes $\sum_j\ell_{ij}(z_i)$ and the prior contributes $\log p(z_i)=-z_i^2/(2\sigma_z^2)-\tfrac12\log(2\pi\sigma_z^2)$. Dropping the constant leaves the log posterior objective of Eq.~\ref{eq:estep}.

\begin{theorem}[Strict concavity of the log posterior objective]
\label{thm:concavity}
Fix criteria with $a_j>0$, difficulties $b_j\in\mathbb R$, and prior variance $\sigma_z^2>0$. Suppose the adopted response function is $F=\Phi$. Then the log posterior objective $\ell_i$ of Eq.~\ref{eq:estep} is smooth, strictly concave, and tends to $-\infty$ as $|z_i|\to\infty$. Its derivative is strictly decreasing and satisfies
\[
  \ell_i'(z_i)\to+\infty\text{ as }z_i\to-\infty,
  \qquad
  \ell_i'(z_i)\to-\infty\text{ as }z_i\to+\infty.
\]
Thus the reward $\hat z_i=\arg\max_z\ell_i(z)$ exists, is the unique root of $\ell_i'$, and is differentiable in the criterion parameters $(a_j,b_j)$.
\end{theorem}

\begin{proof}
Differentiating the log posterior objective $\ell_i$ of Eq.~\ref{eq:estep} twice gives
\[
  \ell_i''(z_i)=\sum_j a_j^2\big[G_{ij}\,\lambda'(u_{ij})+(1-G_{ij})\,\lambda'(-u_{ij})\big]-\frac{1}{\sigma_z^2}\le -\frac{1}{\sigma_z^2}<0,
\]
because $\lambda=(\log\Phi)'$ and $\log\Phi$ is concave \citep{bagnoli2005log}, so $\lambda'(x)\le0$ and every bracketed term is nonpositive. The Gaussian prior makes the displayed inequality strict. Hence $\ell_i$ is smooth and its derivative is strictly decreasing. Each log likelihood term is nonpositive, so $\ell_i(z_i)\le-z_i^2/(2\sigma_z^2)$ and the log posterior objective tends to $-\infty$ in both tails. Eq.~\ref{eq:estep-grad} also gives the stated derivative limits. The derivative is continuous and strictly decreasing, so it has exactly one root. This root is the unique maximizer. Its derivative with respect to the criterion parameters exists by implicit differentiation because $\ell_i''(\hat z_i)<0$.
\end{proof}

\begin{corollary}[Dominance of the MAP reward]
\label{cor:map-dominance}
Fix one prompt group and its criterion parameters. If two verdict vectors satisfy $G^A_j\ge G^B_j$ for every criterion, then $\hat z_A\ge\hat z_B$. The inequality is strict if at least one verdict differs. Thus a difficulty change can reverse only incomparable verdict vectors.
\end{corollary}

\begin{proof}
At every $z$, define $u_j=a_j(z-b_j)$. Theorem~\ref{thm:general-pass-probability} gives
\[
  \ell_A'(z)-\ell_B'(z)
  =
  \sum_j a_j s_\Phi(u_j)(G^A_j-G^B_j).
\]
Every summand is nonnegative, and one is positive when a verdict differs. At the root $\hat z_B$, this gives $\ell_A'(\hat z_B)\ge0$. Since $\ell_A'$ is strictly decreasing, its root lies weakly to the right, and strictly to the right when a verdict differs.
\end{proof}

\subsection{Criterion Information and Realized Evidence}
\label{sec:fisher-information}

\begin{proof}[Proof of Theorem~\ref{thm:fisher-criterion}]
Theorem~\ref{thm:general-pass-probability} gives
\[
  \mathcal S_{ij}(z_i)
  =
  a_j s_\Phi(u_{ij})(G_{ij}-P_{ij}).
\]
Since $\mathbb E[G_{ij}\mid z_i]=P_{ij}$, its conditional mean is zero. The conditional variance of a Bernoulli verdict gives
\[
  \mathbb E[(G_{ij}-P_{ij})^2\mid z_i]
  =
  P_{ij}(1-P_{ij}).
\]
Substituting the criterion likelihood score into the definition of $I_j$ gives
\[
  I_j(z_i)=a_j^2 s_\Phi(u_{ij})^2 P_{ij}(1-P_{ij})=a_j^2\frac{\phi(u_{ij})^2}{P_{ij}(1-P_{ij})}.
\]
The last equality in Eq.~\ref{eq:fisher-criterion} is the information identity for this smooth Bernoulli likelihood. It also follows by directly differentiating the criterion likelihood score and taking its conditional expectation.
\end{proof}

\begin{theorem}[Information frontier of a criterion under $F=\Phi$]
\label{thm:information-frontier}
For a fixed criterion $j$, the information
\[
  I_j(z_i)=a_j^2 f(u_{ij}),
  \qquad
  f(u)=\frac{\phi(u)^2}{\Phi(u)\Phi(-u)}.
\]
is symmetric around $z_i=b_j$. It strictly increases for $z_i<b_j$ and strictly decreases for $z_i>b_j$. Its unique maximum is
\[
  I_j(b_j)=\frac{2}{\pi}a_j^2.
\]
It also satisfies
\[
  I_j(z_i)\longrightarrow0
  \qquad
  \text{as }
  |z_i-b_j|\longrightarrow\infty.
\]
For fixed $t$,
\[
  \frac{I_j(b_j+t/a_j)}{a_j^2}
  =
  \frac{\phi(t)^2}{\Phi(t)(1-\Phi(t))}.
\]
\end{theorem}

\begin{proof}
The function $f$ is even because $\phi$ is even and $\Phi(-u)=1-\Phi(u)$. Let $X$ be a standard Gaussian random variable. Using $\lambda(u)=\phi(u)/\Phi(u)$ gives
\[
  \frac{d^2}{du^2}\log f(u)=-2-\lambda'(u)-\lambda'(-u)=-\operatorname{Var}(X\mid X\le u)-\operatorname{Var}(X\mid X>u)<0.
\]
The second equality uses $1+\lambda'(u)=\operatorname{Var}(X\mid X\le u)$ and Gaussian symmetry. Thus $f$ is strictly log-concave. An even, strictly log-concave function has its unique maximum at zero and is strictly monotone on either side. At zero,
\[
  f(0)
  =
  \frac{\phi(0)^2}{(1/2)(1/2)}
  =
  \frac{2}{\pi}.
\]
Mills' ratio gives $f(u)\sim |u|\phi(u)$ as $|u|\to\infty$, which tends to zero. The final display follows by substituting $z_i=b_j+t/a_j$.
\end{proof}

Theorem~\ref{thm:information-frontier} separates the criterion parameters. The difficulty $b_j$ places the information peak on the quality scale. The discrimination $a_j$ raises the peak quadratically and makes its width on the $z_i$ scale proportional to $1/a_j$.

For a realized verdict, the magnitude of one criterion likelihood score is
\[
  \left|
    \frac{\partial\ell_{ij}}{\partial z_i}
  \right|
  =
  \begin{cases}
    \displaystyle
    a_j\frac{\phi(u_{ij})}{P_{ij}},
    &G_{ij}=1,\\[8pt]
    \displaystyle
    a_j\frac{\phi(u_{ij})}{1-P_{ij}},
    &G_{ij}=0.
  \end{cases}
\]
Consequently,
\[
  \frac{
    \left|\partial\ell_{ij}/\partial z_i\right|_{G_{ij}=1}
  }{
    \left|\partial\ell_{ij}/\partial z_i\right|_{G_{ij}=0}
  }
  =
  \frac{1-P_{ij}}{P_{ij}}.
\]
These criterion likelihood scores measure evidence from the realized verdict, whereas Fisher information measures expected usefulness before the verdict is observed. A pass has nine times the likelihood score magnitude of a failure when $P_{ij}=0.1$. A failure has nine times the magnitude of a pass when $P_{ij}=0.9$. Expected information and realized evidence can therefore rank a criterion differently.

Figure~\ref{fig:information-profiles} uses pass probability as its horizontal coordinate. Let \(p=P_{ij}\) and \(u=\Phi^{-1}(p)\). The three plotted curves are
\[
  \frac{I_j(z_i)}{a_j^2}=\frac{\phi(u)^2}{p(1-p)},
  \qquad
  \left.\frac{1}{a_j}\frac{\partial\ell_{ij}}{\partial z_i}\right|_{G_{ij}=1}=\frac{\phi(u)}{p},
  \qquad
  \left.\frac{1}{a_j}\frac{\partial\ell_{ij}}{\partial z_i}\right|_{G_{ij}=0}=-\frac{\phi(u)}{1-p}.
\]
The information curve reaches \(2/\pi\) at \(p=1/2\). The pass and failure curves grow in magnitude when the observed verdict has low fitted probability.

\begin{figure}[!htbp]
  \centering
  \includegraphics[width=\textwidth]{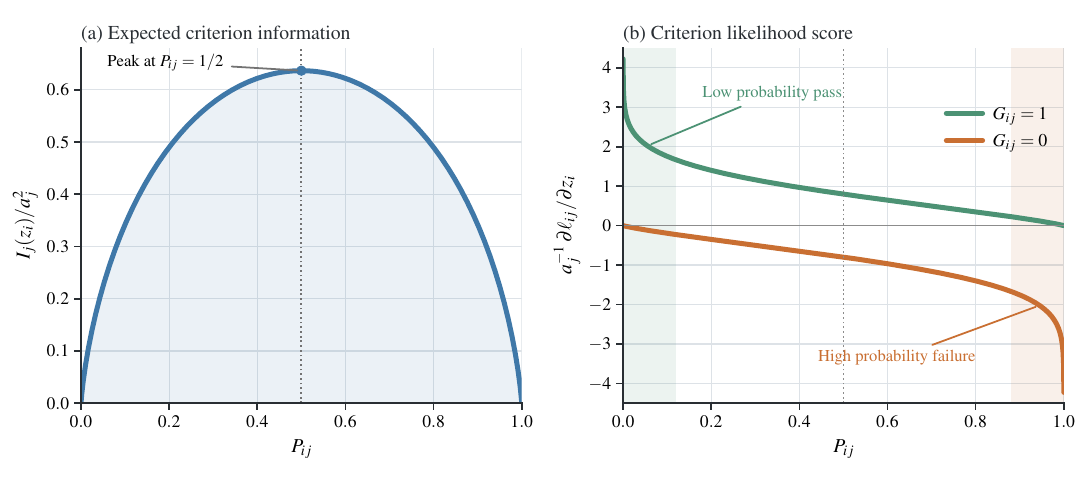}
  \caption{Criterion information and likelihood scores under the Gaussian CDF $F=\Phi$ as functions of pass probability $P_{ij}$. The left panel shows normalized information $I_j(z_i)/a_j^2$. The right shows $a_j^{-1}\partial\ell_{ij}/\partial z_i$ for a pass and failure. The dotted line marks $P_{ij}=1/2$, and shading marks verdicts with low probability.}
  \label{fig:information-profiles}
\end{figure}

\subsection{Local Efficiency of Reward Aggregation}
\label{sec:local-efficiency}

Fix a differentiable $F$, a quality level $z_i$, fixed criterion parameters $(a_j,b_j)$, and conditionally independent verdicts under the item response model of Section~\ref{sec:exam}. Define
\[
  I_{F,j}(z_i):=a_j^2\frac{F'(u_{ij})^2}{P_{ij}(1-P_{ij})}
  =P_{ij}(1-P_{ij})[a_js_F(u_{ij})]^2,
\]
and suppose $\sum_j I_{F,j}(z_i)>0$. Consider the linear statistic
\[
  T_v=\sum_j v_j(G_{ij}-P_{ij}),
\]
where $P_{ij}$ is evaluated at $z_i$. Suppose $T_v$ has unit local response to a quality change,
\begin{equation}
  \left.
  \frac{d}{dh}
  \mathbb E_{z_i+h}[T_v]
  \right|_{h=0}
  =
  \sum_j v_j a_jF'(u_{ij})
  =1.
  \label{eq:local-unbiasedness}
\end{equation}
The proof of Theorem~\ref{thm:local-efficiency} extends its bound to general $F$. Applying this bound to $T_v$ gives the statistic with minimum variance
\begin{equation}
  T^*
  =
  \frac{
    \sum_j a_j s_F(u_{ij})(G_{ij}-P_{ij})
  }{
    \sum_j I_{F,j}(z_i)
  }.
  \label{eq:efficient-local-estimator}
\end{equation}
Its conditional variance is
\[
  \operatorname{Var}(T^*\mid z_i)
  =
  \left[\sum_j I_{F,j}(z_i)\right]^{-1}.
\]
Differentiability gives $\mathbb E_{z_i+h}[T^*]=h+o(h)$. This unbiasedness is local.

For $F=\Phi$, $I_{F,j}=I_j$ and the numerator of Eq.~\ref{eq:efficient-local-estimator} is the rubric likelihood score $\mathcal S_i(z_i)$, the likelihood contribution in Eq.~\ref{eq:estep-grad}. The efficient coefficient is $a_j s_\Phi(u_{ij})$, not $a_j$ alone. The latter coefficient belongs to $F=\sigma$ because $s_\sigma\equiv1$.

Under conditional independence, the expected negative curvature of the log posterior is
\begin{equation}
  -\mathbb E[\ell_i''(z_i)\mid z_i]
  =
  \frac{1}{\sigma_z^2}+\sum_j I_j(z_i).
  \label{eq:expected-local-precision}
\end{equation}
The Gaussian prior contributes $1/\sigma_z^2$. Theorem~\ref{thm:fisher-criterion} gives the expected negative curvature $I_j(z_i)$ of each criterion, and the terms add under conditional independence. For duplicate or dependent criteria, information is not additive. After substituting $\hat z_i$, the inverse square root of the expected precision and the observed curvature $[-\ell_i''(\hat z_i)]^{-1/2}$ give local uncertainty approximations.

\subsection{Alignment with the Reward Based on Rubric Points}
\label{sec:points-progress-proofs}

Under the item response model, the expected reward based on rubric points is strictly increasing in scalar quality and does not decrease under first-order stochastic dominance. This section proves that result and compares this reward with the rubric likelihood score and MAP reward as local training signals. Fix a prompt $q$, its rubric $c$, and the criterion parameters $(a_j,b_j)=\psi(q,c_j)$. For the verdict vector $G_i$, the rubric likelihood score expands to
\[
  \mathcal S_i(z_i)=\sum_j\mathcal S_{ij}(z_i)
  =\sum_j a_j s_F(u_{ij})\big(G_{ij}-P_{ij}\big)
\]
for a general $F$. Define the rubric information for general $F$ as
\[
  I_F(z_i):=\sum_j I_{F,j}(z_i),
\]
which equals the rubric information $I(z_i)$ under the adopted $F=\Phi$. The results below assume the model of Eq.~\ref{eq:likelihood}.

\begin{theorem}[Monotone expected reward based on rubric points]
\label{thm:points-monotone}
Let $F$ be differentiable and strictly increasing with $F'>0$. Then
\[
  m(z_i):=\mathbb E\big[R_i^{\mathrm{base}}\,\big|\,z_i\big]
  =\frac{\sum_j w_jP_{ij}}{\sum_\ell w_\ell},
  \qquad
  m'(z_i)=\frac{\sum_j w_ja_jF'(u_{ij})}{\sum_\ell w_\ell}>0 .
\]
Let $Z_\theta$ be the quality of a rollout drawn from $\pi_\theta(\cdot\mid q)$, so that $\mathbb E[R_i^{\mathrm{base}}\mid q]=\mathbb E[m(Z_\theta)]$. If the quality after a policy update first-order stochastically dominates the quality before it, the expected reward based on rubric points does not decrease.
\end{theorem}

\begin{proof}
Linearity of expectation and $\mathbb E[G_{ij}\mid z_i]=P_{ij}$ give $m$. Differentiating $P_{ij}=F(u_{ij})$ gives $\partial P_{ij}/\partial z_i=a_jF'(u_{ij})$, and every term is positive because $w_j>0$, $a_j>0$, and $F'>0$. Since $m$ is increasing, $\mathbb E[m(Z_{\text{after}})]\ge\mathbb E[m(Z_{\text{before}})]$ is the defining property of first-order stochastic dominance.
\end{proof}

Theorem~\ref{thm:points-monotone} establishes alignment under first-order stochastic dominance. Outside this condition, the two training signals can disagree. A location shift of the quality distribution is one sufficient special case.

\begin{proof}[Proof of Theorem~\ref{thm:local-efficiency}]
The argument uses only differentiability of $F$ and conditional independence. It therefore gives the bound with $I_F(z_i)$ in place of $I(z_i)$ for a general response function, whenever $0<P_{ij}<1$ for every $j$ and $I_F(z_i)>0$. The verdict vector takes finitely many values, so the sum defining $\mathbb E[T\mid z_i]$ may be differentiated term by term. Writing $p(G_i\mid z_i)$ for the likelihood of Eq.~\ref{eq:likelihood},
\[
  \frac{d}{dz_i}\mathbb E[T\mid z_i]=\sum_{G_i}T(G_i)\,p(G_i\mid z_i)\,\mathcal S_i(z_i)=\mathbb E[T\,\mathcal S_i(z_i)\mid z_i]=\operatorname{Cov}\big(T,\mathcal S_i(z_i)\mid z_i\big),
\]
where the last step uses $\mathbb E[G_{ij}\mid z_i]=P_{ij}$, which makes each criterion likelihood score have conditional mean zero. Conditional independence gives
\[
  \operatorname{Var}\big(\mathcal S_i(z_i)\mid z_i\big)
  =\sum_j\mathbb E\big[\mathcal S_{ij}(z_i)^2\mid z_i\big]
  =I_F(z_i).
\]
Cauchy-Schwarz then gives $\operatorname{Cov}(T,\mathcal S_i)^2\le\operatorname{Var}(T)\,I_F(z_i)$, which is the stated bound, with equality if and only if the two centered variables are almost surely proportional under $p_\psi(G_i\mid z_i,q,c)$. The constant cannot be zero, since that would make $T$ almost surely constant and contradict $\operatorname{Var}(T\mid z_i)>0$. Replacing $T$ by $\alpha+\beta T$ with $\beta\neq0$ multiplies numerator and denominator of Eq.~\ref{eq:local-snr} by $\beta^2$, so the ratio is unchanged by any nonzero affine map of $T$.
\end{proof}

\begin{proof}[Proof of the clause for rubric points in Theorem~\ref{thm:local-efficiency}]
Under the same conditions,
\[
  \operatorname{SNR}_{\mathrm{base}}(z_i)
  =\frac{\big(\sum_j w_ja_jF'(u_{ij})\big)^2}
        {\sum_j w_j^2P_{ij}(1-P_{ij})}
  \le I_F(z_i),
\]
with equality if and only if $w_j=\gamma_w\,a_js_F(u_{ij})$ for every $j$ and some $\gamma_w>0$. For $F=\sigma$ this condition reads $w_j\propto a_j$. For $F=\Phi$ it reads $w_j\propto a_js_\Phi(u_{ij})$, whose right side depends on $b_j$ and on $z_i$. The displayed ratio follows from $\partial P_{ij}/\partial z_i=a_jF'(u_{ij})$, from conditional independence, and from the Bernoulli variance $P_{ij}(1-P_{ij})$. The factor $\sum_\ell w_\ell$ cancels. Equality in the information bound requires
\[
  \sum_j\left(\frac{w_j}{\sum_\ell w_\ell}-\gamma_w'a_js_F(u_{ij})\right)(G_{ij}-P_{ij})=0
\]
almost surely. Taking the conditional variance of the left side gives a sum of nonnegative terms with positive factors $P_{ij}(1-P_{ij})$, so every coefficient vanishes. Since $w_j>0$, equality requires $s_F(u_{ij})>0$ for every $j$ and $\gamma_w'>0$. For $F=\sigma$, $s_\sigma\equiv1$ by Theorem~\ref{thm:logit}. For $F=\Phi$, take two criteria and suppose equality holds at every $z_i$, so that
\[
  \chi(z_i)
  :=\frac{a_1s_\Phi\big(a_1(z_i-b_1)\big)}
         {a_2s_\Phi\big(a_2(z_i-b_2)\big)}
  =\frac{w_1}{w_2}
\]
is a constant $\chi_0$. Mills' ratio sharpens the asymptotic of Theorem~\ref{thm:gaussian} to $s_\Phi(u)=u+O(u^{-1})$ as $u\to\infty$, so $a_js_\Phi(a_j(z_i-b_j))=a_j^2(z_i-b_j)+O(z_i^{-1})$ and
\[
  \big(a_1^2-\chi_0a_2^2\big)z_i-\big(a_1^2b_1-\chi_0a_2^2b_2\big)\longrightarrow0 .
\]
An affine function with this limit vanishes identically, so $a_1^2=\chi_0a_2^2$ and then $b_1=b_2$. Evaluating $\chi$ at $z_i=b_1$ gives $\chi=a_1/a_2$ through $s_\Phi(0)=4/\sqrt{2\pi}$, so $a_1/a_2=\chi_0=a_1^2/a_2^2$ and $a_1=a_2$. Therefore, once two criteria have different criterion parameters, no fixed vector of rubric points achieves the bound at every quality level.
\end{proof}

The points $w_j$ specify how much each criterion should count. The coefficients of the verdicts in Eq.~\ref{eq:efficient-local-estimator} are proportional to $a_js_F(u_{ij})$. Under $F=\Phi$, they vary with quality while the points $w_j$ stay fixed. Unless the coefficients are proportional, the points total has lower local SNR than the likelihood score. The reward based on rubric points also cannot distinguish verdict vectors with the same points total.

\begin{proposition}[Points ties and reward separation]
\label{prop:points-ties}
Fix a prompt group and its criterion parameters. If two rollouts satisfy $\sum_j w_jG_{1j}=\sum_j w_jG_{2j}$, then $R_1^{\mathrm{base}}=R_2^{\mathrm{base}}$.
The RRT reward $\hat z_i$ can still separate two such rollouts. Under $F=\sigma$ it separates them if and only if $\sum_j a_jG_{1j}\neq\sum_j a_jG_{2j}$. Under $F=\Phi$ it can separate them even when those totals weighted by discrimination agree.
\end{proposition}

\begin{proof}
Equality of the rewards based on rubric points follows from $R_i^{\mathrm{base}}=\sum_j w_jG_{ij}/\sum_\ell w_\ell$.
The claim for $F=\sigma$ is Theorem~\ref{thm:logit}, which makes $\hat z_i$ a strictly increasing function of $\sum_j a_jG_{ij}$ shared by the group. For $F=\Phi$ take $K=2$, $w_1=w_2=1$, $a_1=a_2=1$, $\sigma_z^2=1$, $b_2=0$, $b_1=b>0$, and the verdict vectors $G_1=(1,0)$ and $G_2=(0,1)$, which agree on both totals. Eq.~\ref{eq:estep-grad} gives
\[
  \ell_1'(z)=\lambda(z-b)-\lambda(-z)-z,
  \qquad
  \ell_2'(z)=\lambda(z)-\lambda(b-z)-z .
\]
Since $\lambda>0$, it follows that $\ell_2'(z)\le\lambda(z)-z$, whose unique root $z^\star$ is finite, so $\hat z_2\le z^\star$. Mills' ratio gives $\lambda(z^\star-b)\to\infty$ as $b\to\infty$ while $\lambda(-z^\star)+z^\star$ is fixed, so $\ell_1'(z^\star)>0$ for large $b$. Because $\ell_1'$ is strictly decreasing by Theorem~\ref{thm:concavity}, $\hat z_1>z^\star\ge\hat z_2$.
\end{proof}

\begin{theorem}[Exact MAP expansion and group invariance]
\label{thm:local-advantage}
Suppose $F=\Phi$. For a group statistic vector $\mathbf T=(T_1,\ldots,T_N)$, write $\bar T=N^{-1}\sum_kT_k$ and define
\[
  \mathcal A_i[\mathbf T]:=\frac{T_i-\bar T}{\operatorname{std}_kT_k+\varepsilon}
\]
for the group advantage operator of Eq.~\ref{eq:reward}. Let $\hat{\mathbf z}=(\hat z_1,\ldots,\hat z_N)$.
If $\varepsilon=0$ and $\operatorname{std}_kT_k>0$, then $\mathcal A_i[\alpha\mathbf 1+\beta\mathbf T]=\mathcal A_i[\mathbf T]$ for every $\alpha$ and every $\beta>0$. Fix a quality level $z_0$. For every rollout there is a point $\xi_i$ between $z_0$ and $\hat z_i$ with
\[
  \hat z_i-z_0=\frac{\mathcal S_i(z_0)-z_0/\sigma_z^2}{-\ell_i''(\xi_i)} .
\]
\end{theorem}

\begin{proof}
For $\beta>0$ the group mean of $\alpha\mathbf 1+\beta\mathbf T$ is $\alpha+\beta\bar T$ and its standard deviation is $\beta\operatorname{std}_kT_k$, so the two factors $\beta$ cancel at $\varepsilon=0$ when the denominator is positive. Theorem~\ref{thm:concavity} makes $\ell_i$ smooth with $\ell_i'(\hat z_i)=0$, so the mean value theorem gives $0=\ell_i'(z_0)+\ell_i''(\xi_i)(\hat z_i-z_0)$. Eq.~\ref{eq:estep-grad} gives $\ell_i'(z_0)=\mathcal S_i(z_0)-z_0/\sigma_z^2$.
\end{proof}

A Fisher scoring surrogate replaces the realized curvature in Theorem~\ref{thm:local-advantage} by the expected local precision in Eq.~\ref{eq:expected-local-precision}. Define
\begin{equation}
  \widetilde z_i(z_0)
  :=z_0+
  \frac{\mathcal S_i(z_0)-z_0/\sigma_z^2}{\sigma_z^{-2}+I(z_0)}.
  \label{eq:local-advantage}
\end{equation}
Let $\widetilde{\mathbf z}(z_0)=(\widetilde z_1(z_0),\ldots,\widetilde z_N(z_0))$ and $\boldsymbol{\mathcal S}(z_0)=(\mathcal S_1(z_0),\ldots,\mathcal S_N(z_0))$. The denominator in Eq.~\ref{eq:local-advantage} is positive and shared across the group. Therefore, if $\varepsilon=0$ and the score vector has positive standard deviation, affine invariance gives the exact identity
\[
  \mathcal A_i[\widetilde{\mathbf z}(z_0)]
  =\mathcal A_i[\boldsymbol{\mathcal S}(z_0)].
\]
Replacing each realized curvature by the shared expected precision can change the exact MAP order.

The MAP reward can rank incomparable verdict vectors differently from the reward based on rubric points.

\section{Method and Optimization Details}
\label{sec:method-derivations}

\subsection{Online EM Objective and Algorithms}
\label{sec:hard-em}

Algorithm~\ref{alg:rrt} gives the full policy training and calibration loop.

\begin{algorithm}[!htbp]
\caption{RRT policy training with online GRPO and EM calibration of the RPN.}
\label{alg:rrt}
\small
\begin{algorithmic}[1]
\State \textbf{Initialize} policy $\pi_\theta$ from the checkpoint of the base policy and RPN $\psi$ neutrally or with a warm start
\For{each policy step}
  \State sample a batch of prompts and draw $N$ rollouts per prompt from $\pi_\theta$
  \State judge each rollout against its rubric $\to$ verdicts $G$
  \State \textbf{E-step:} in evaluation mode, infer $\hat z_i^{\mathrm{reward}}\gets\arg\max_{z}\ell_i(z)$ by bisection using Eq.~\ref{eq:estep-grad} (Algorithm~\ref{alg:estep}), $R_i\gets\hat z_i^{\mathrm{reward}}$
  \State compute $A_i$ within each prompt group (Eq.~\ref{eq:reward})
  \State \textbf{GRPO:} update $\pi_\theta$ on $J(\theta)$ (Eq.~\ref{eq:grpo})
  \State \textbf{Stochastic partial M-step:} one accumulated gradient step on $\psi$ with detached mode targets $\hat z_i^{\mathrm{M}}$ using Eq.~\ref{eq:mstep} (Algorithm~\ref{alg:mstep})
\EndFor
\end{algorithmic}
\end{algorithm}

Let $\mathbf z=(z_1,\ldots,z_N)$ collect the latent qualities of the rollouts. The ideal EM objective is the regularized log posterior for the complete data
\begin{equation}
  \mathcal C_{\mathrm{EM}}(\psi,\mathbf z)
  =\sum_{i=1}^{N}\Big[\sum_{j=1}^{K}\ell_{ij}(z_i;\psi)-\frac{z_i^2}{2\sigma_z^2}\Big]
  -N\lambda_a\sum_{j=1}^{K}(\log a_j)^2.
  \label{eq:hard-em-objective}
\end{equation}
With $\psi^{(s)}$ fixed, maximizing it over each $z_i$ gives the E-step of Eq.~\ref{eq:estep}. With $\mathbf z$ fixed, exact minimization of Eq.~\ref{eq:mstep} gives the other coordinate update up to normalization because the Gaussian quality prior has no parameter in $\psi$. Algorithm~\ref{alg:mstep} instead takes one stochastic optimizer step. Because $z_i$ is scalar and its log posterior is strictly concave by Theorem~\ref{thm:concavity}, each E-step has one solution and bisection computes it deterministically for fixed criterion parameters.

The likelihood alone has the IRT scale tradeoff
\[
  a_j\mapsto a_j/r,\qquad z_i\mapsto rz_i,\qquad b_j\mapsto rb_j,\qquad r>0,
\]
which leaves the response function argument $a_j(z_i-b_j)$ unchanged. The Gaussian prior and discrimination regularizer fix this scale. The prior mean anchors the location, and the constraint $a_j>0$ fixes the direction.

The E-step returns the MAP quality $\hat z_i$ of Eq.~\ref{eq:argmax}, as specified by Algorithm~\ref{alg:estep}. The log posterior objective $\ell_i$ is strictly concave by Theorem~\ref{thm:concavity}, so its derivative has one root. The algorithm starts from $[-B_z,B_z]$ and doubles any endpoint whose derivative has the wrong sign until this root is bracketed. The derivative limits guarantee that the expansion terminates. It then halves the bracket $T$ times using the derivative sign at the midpoint, with $\lambda(x)=\phi(x)/\Phi(x)$ as in Eq.~\ref{eq:gaussian-cdf-response-weight}. If the expanded bracket has width $W$, the returned midpoint has error at most $W/2^{T+1}$.

\begin{algorithm}[!htbp]
\caption{\textsc{EstimateZ}: MAP quality estimation by bisection.}
\label{alg:estep}
\small
\begin{algorithmic}[1]
\Require verdicts $G_{ij}$, criterion parameters $a_j,b_j$ from the RPN $\psi$, prior variance $\sigma_z^2$, initial bracket half-width $B_z>0$, iteration count $T$
\Ensure MAP quality $\hat z_i$ from Eq.~\ref{eq:argmax} for every rollout $i$
\State define $m_i(x)\gets\ell_i'(x)$ by Eq.~\ref{eq:estep-grad}
\State $\mathrm{lo}_i\gets -B_z,\quad \mathrm{hi}_i\gets +B_z$
\While{some $m_i(\mathrm{lo}_i)<0$}
  \State $\mathrm{lo}_i\gets2\mathrm{lo}_i$ for those rollouts
\EndWhile
\While{some $m_i(\mathrm{hi}_i)>0$}
  \State $\mathrm{hi}_i\gets2\mathrm{hi}_i$ for those rollouts
\EndWhile
\For{$t=1$ to $T$}
  \State $z_i\gets\tfrac12(\mathrm{lo}_i+\mathrm{hi}_i)$
  \State $\mathrm{hi}_i\gets z_i$ where $m_i(z_i)\le0$
  \State $\mathrm{lo}_i\gets z_i$ where $m_i(z_i)>0$
\EndFor
\State \Return $\hat z_i\gets\tfrac12(\mathrm{lo}_i+\mathrm{hi}_i)$ for every rollout $i$
\end{algorithmic}
\end{algorithm}

Algorithm~\ref{alg:mstep} sweeps the policy step's rollouts once per epoch in mini-batches of $B$. For each mini-batch, one differentiable forward pass through the RPN $\psi$ predicts $(a_{ij},b_{ij})$ for every pair $(q_i,c_{ij})$. It reruns the E-step on detached parameter values to obtain $\hat z_i^{\mathrm{M}}$, then adds its share of the gradient of Eq.~\ref{eq:mstep} to an accumulator. One AdamW update is applied at the end of the epoch, so $\psi$ does not change between mini-batches.

\begin{algorithm}[!htbp]
\caption{\textsc{UpdatePsi}: stochastic partial M-step for the RPN $\psi$.}
\label{alg:mstep}
\small
\begin{algorithmic}[1]
\Require the policy step's rollouts (rollout $i$ has prompt $q_i$, criteria $c_{ij}$, verdicts $G_{ij}$, and criterion count $K_i$), the current RPN $\psi$ with its AdamW state, and the hyperparameters $E$ (epochs), $B$ (mini-batch size), $\lambda_a$ (discrimination regularizer weight), $\tau_g$ (gradient clip norm)
\Ensure updated RPN $\psi$
\For{$e=1$ to $E$}
  \State shuffle the rollouts and split them into $M$ mini-batches of $B$
  \State $g_\psi\gets0$ \Comment{gradient accumulator}
  \For{each mini-batch $\mathcal{B}$ of $B$ rollouts}
    \State $(a_{ij},b_{ij})\gets\psi(q_i,c_{ij})$ for every criterion of every rollout in $\mathcal{B}$ \Comment{one differentiable forward}
    \State $\hat z_i^{\mathrm{M}}\gets\textsc{EstimateZ}(\mathcal{B})$ with $(a_{ij},b_{ij})$ detached \Comment{Alg.~\ref{alg:estep}, targets held fixed}
    \State $\displaystyle\mathcal{L}\gets\frac{1}{B}\sum_{i\in\mathcal B}\left[-\frac{1}{K_i}\sum_{j=1}^{K_i}\ell_{ij}(\hat z_i^{\mathrm{M}})+\frac{\lambda_a}{K_i}\sum_{j=1}^{K_i}(\log a_{ij})^2\right]$ \Comment{Eq.~\ref{eq:mstep}}
    \State $g_\psi\gets g_\psi+\nabla_\psi\mathcal{L}/M$ \Comment{accumulate, do not step}
  \EndFor
  \State clip $\lVert g_\psi\rVert$ to $\tau_g$, then take one AdamW step on $\psi$
\EndFor
\State \Return $\psi$
\end{algorithmic}
\end{algorithm}

\subsection{GRPO Objective}
\label{sec:grpo-objective}

Let $\widetilde R_i$ denote the scalar reward passed to GRPO. It is centered within each prompt's group of $N$ rollouts to form the advantage
\begin{equation}
  A_i=\frac{\widetilde R_i-\overline{\widetilde R}}
  {\operatorname{std}_k \widetilde R_k+\varepsilon},\quad
  \overline{\widetilde R}=\tfrac1N\sum_k \widetilde R_k.
  \label{eq:reward}
\end{equation}
Here $\varepsilon>0$ stabilizes groups when reward variance is near zero.

Write rollout $o_i=(o_{i1},\ldots,o_{iL_i})$, where $L_i$ is its generated token count. The standard GRPO update uses the clipped surrogate objective of PPO \citep{schulman2017proximal} and a Kullback-Leibler (KL) penalty toward the fixed reference policy $\pi_{\mathrm{ref}}$. Let $\pi_\theta$ be the current policy and $\pi_{\theta_{\mathrm{old}}}$ the old policy that generated the rollouts. Define the token likelihood ratio
\[
  \varrho_{it}=
  \frac{\pi_\theta(o_{it}\mid q,o_{i,<t})}
       {\pi_{\theta_{\mathrm{old}}}(o_{it}\mid q,o_{i,<t})}
\]
and $r^{\mathrm{ref}}_{it}=\pi_{\mathrm{ref}}(o_{it}\mid q,o_{i,<t})/\pi_\theta(o_{it}\mid q,o_{i,<t})$. Its KL estimator at the token level is $\widehat D_{\mathrm{KL},it}=r^{\mathrm{ref}}_{it}-\log r^{\mathrm{ref}}_{it}-1$. The policy maximizes
\begin{equation}
  J(\theta)=\mathbb E_i\Bigg[\frac{1}{L_i}\sum_{t=1}^{L_i}\Big(
  \min\!\big(\varrho_{it}A_i,
  \operatorname{clip}(\varrho_{it},1-\epsilon_c,1+\epsilon_c)A_i\big)-\beta_{\mathrm{KL}}\widehat D_{\mathrm{KL},it}\Big)\Bigg],
  \label{eq:grpo}
\end{equation}
where $\epsilon_c$ is the likelihood ratio clip radius and $\beta_{\mathrm{KL}}$ weights the KL penalty.

\section{Additional Criterion and Reward Experiments}
\label{sec:additional-experiments}

\subsection{Text Prediction of Criterion Parameters}
\label{sec:rpn-prediction}

This experiment tests whether criterion parameters predicted by the RPN $\psi$ from text recover empirical criterion difficulty and predict observed verdicts. The embedder ablation compares Qwen3 Embedding models \citep{zhang2025qwen3} and Llama-Embed-Nemotron-8B \citep{babakhin2025llama}. The criterion parameters predicted by the RPN, together with inferred rollout quality $\hat z_i$, give the fitted pass probability \(\widehat P_{ij}^{\mathrm{MAP}}=F\!\left(a_j(\hat z_i-b_j)\right)\).

The response function $F$ varies across the Gaussian CDF $\Phi$, the logistic CDF $\sigma$, and the complementary log-log function $\mathrm{cll}(t)=1-\exp(-e^t)$. The last choice relaxes $F(0)=0.5$, so $b_j$ is a location parameter rather than the 50\% pass threshold in that configuration. A second ablation varies how many criterion parameters $\psi$ predicts. The model with three parameters adds a lower asymptote $\gamma_j$, which the pass probability approaches as $z_i\to-\infty$ \citep{birnbaum1968some}. The model with four parameters adds an upper asymptote $\xi_j$, which the pass probability approaches as $z_i\to+\infty$ \citep{magis2013note}. The resulting fitted probability adds \(\gamma_j\) to \((\xi_j-\gamma_j)F(a_j(\hat z_i-b_j))\), with $\xi_j=1$ in the model with three parameters and $\gamma_j=0$, $\xi_j=1$ in the models with one and two parameters. The model with one parameter also fixes $a_j=1$. The RPN reads $(q,c_j)$ and predicts every free parameter, and the asymptotes are constrained to $0<\gamma_j<\xi_j<1$.

For a set $\mathcal I$ of rollouts, define the empirical criterion pass rate as $\bar G_{j,\mathcal I}=|\mathcal I|^{-1}\sum_{i\in\mathcal I}G_{ij}$, and write $\bar G_j$ when the set is clear. Two complementary evaluation metrics are used. First, Spearman correlation between the criterion difficulty parameter $b_j$ and empirical criterion difficulty $1-\bar G_j$ measures whether the RPN predicts which criteria are hard from text (Table~\ref{tab:rpn-geometry}). Second, ROC-AUC between $\widehat P_{ij}^{\mathrm{MAP}}$ and $G_{ij}$ measures verdict ranking. Quality $\hat z_i$ is inferred from the full rollout verdict vector, including the evaluated verdict $G_{ij}$. Section~\ref{sec:loco} describes its leave-one-criterion-out counterpart, while Appendix~\ref{sec:em-diagnostics-appendix} defines a separate metric for the next policy step.

\begin{table}[!htbp]
\centering
\caption{Verdict ROC-AUC of RPN configurations across datasets when the evaluated verdict is included. Each block varies one component: the response function, number of criterion parameters, text conditioning, text embedder, or rollout policy. The remaining components use the standard configuration. Bold marks the highest macro mean in each block.}
\label{tab:rpn-prediction}
\small
\setlength{\tabcolsep}{3pt}
\renewcommand{\arraystretch}{1.0}
\begin{tabular}{@{}lcccc@{\hspace{7pt}\vrule width 0.5pt\hspace{7pt}}c@{}}
\toprule
& \multicolumn{5}{c}{Verdict included in fit (ROC-AUC)} \\
\cmidrule(lr){2-6}
RPN configuration & Medical & Science & RaR Science & RubricBench & Macro mean \\
\midrule
\multicolumn{6}{c}{$F$} \\
\midrule
$\Phi$ (Gaussian CDF) & 80.6 & 84.2 & 91.7 & 91.9 & 87.1 \\
$\sigma$ (logistic CDF) & 80.5 & 84.1 & 90.1 & 90.6 & 86.3 \\
$\mathrm{cll}$ (complementary log-log) & 80.5 & 84.3 & 92.2 & 91.8 & \textbf{87.2} \\
\midrule
\multicolumn{6}{c}{Number of criterion parameters} \\
\midrule
One parameter, $b_j$ & 80.4 & 83.8 & 90.8 & 90.7 & 86.4 \\
Two parameters, $(a_j,b_j)$ & 80.6 & 84.2 & 91.7 & 91.9 & \textbf{87.1} \\
Three parameters, $(a_j,b_j,\gamma_j)$ & 80.6 & 84.2 & 91.7 & 91.7 & 87.0 \\
Four parameters, $(a_j,b_j,\gamma_j,\xi_j)$ & 80.5 & 84.2 & 91.5 & 90.4 & 86.6 \\
\midrule
\multicolumn{6}{c}{Text conditioning} \\
\midrule
$\psi(c_j)$ & 80.6 & 84.1 & 91.8 & 91.7 & \textbf{87.1} \\
$\psi(q,c_j)$ & 80.6 & 84.2 & 91.7 & 91.9 & \textbf{87.1} \\
\midrule
\multicolumn{6}{c}{Text embedder} \\
\midrule
Qwen3-Embedding-0.6B & 79.2 & 82.9 & 91.3 & 91.5 & 86.2 \\
Qwen3-Embedding-4B & 80.6 & 84.2 & 91.7 & 91.9 & 87.1 \\
Qwen3-Embedding-8B & 81.4 & 84.9 & 91.8 & 91.6 & 87.4 \\
Llama-Embed-Nemotron-8B & 83.1 & 86.3 & 92.5 & 92.8 & \textbf{88.7} \\
\midrule
\multicolumn{6}{c}{Rollout policy} \\
\midrule
Qwen3.5-4B & 80.6 & 84.2 & 91.7 & 91.9 & \textbf{87.1} \\
Qwen3.5-2B & 82.1 & 84.1 & 90.2 & 86.9 & 85.8 \\
Llama-3.1-8B-Instruct & 84.5 & 83.0 & 88.7 & 86.4 & 85.6 \\
\bottomrule
\end{tabular}
\end{table}

Across all configurations in Table~\ref{tab:rpn-prediction}, macro ROC-AUC ranges from 85.6\% to 88.7\%. The adopted item response model with two parameters and a Gaussian CDF reaches 87.1\%, compared with 86.4\% for the model with one parameter, 87.0\% for the model with three parameters, and 86.6\% for the model with four parameters.

\subsection{RPN Representation Geometry}
\label{sec:rpn-geometry-appendix}

This analysis tests whether nearby RPN representations have similar predicted criterion difficulties and verdict patterns. It projects the last hidden representations used to predict $b_j$ with t-distributed stochastic neighbor embedding (t-SNE) \citep{van2008visualizing}.

\begin{figure}[!htbp]
\centering
\includegraphics[width=\textwidth]{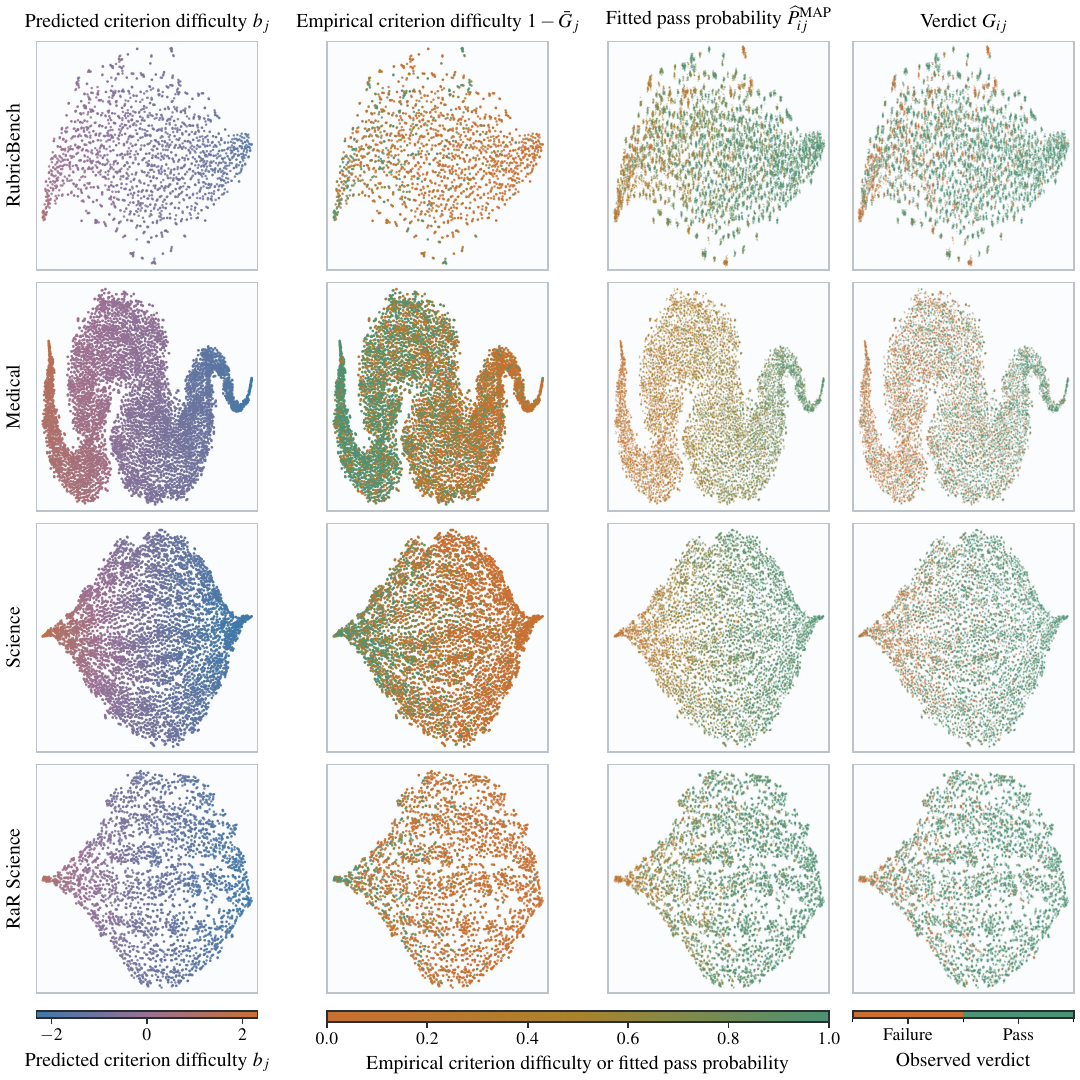}
\caption{t-SNE projections of RPN representations. Rows show the four datasets. The first two columns color pairs by predicted criterion difficulty $b_j$ and empirical criterion difficulty $1-\bar G_j$. The last two color sampled verdicts at jittered pair coordinates by the pass probability $\widehat P_{ij}^{\mathrm{MAP}}$ fitted using the RPN and inferred rollout quality, and by verdict $G_{ij}$. Empirical criterion difficulty and fitted pass probability share a color scale.}
\label{fig:rpn-representation}
\end{figure}

Figure~\ref{fig:rpn-representation} shows stronger local structure for predicted criterion difficulty and fitted pass probability than for empirical criterion difficulty and observed verdicts.

Table~\ref{tab:rpn-geometry} reports the Spearman correlation between $b_j$ and $1-\bar G_j$ as difficulty $\rho_{\mathrm S}$. Trustworthiness at 10 neighbors (T@10) measures how well the projection preserves local neighborhoods. The analysis also computes Moran's $I$ after normalizing each row of the graph that connects the 10 nearest neighbors of each displayed variable \citep{moran1950notes}. Larger values mean that nearby points have more similar displayed values.

\begin{table}[!htbp]
\centering
\caption{Criterion difficulty prediction and local t-SNE geometry. Columns report difficulty Spearman correlation, trustworthiness at 10 neighbors, and Moran's $I$ for predicted criterion difficulty, empirical criterion difficulty, pass probability fitted using the RPN and inferred rollout quality, and verdict. All values are percentages.}
\label{tab:rpn-geometry}
\small
\setlength{\tabcolsep}{7pt}
\renewcommand{\arraystretch}{1.0}
\begin{tabular}{@{}lcccccc@{}}
\toprule
& & & \multicolumn{4}{c}{Moran's $I$} \\
\cmidrule(lr){4-7}
Dataset & Difficulty $\rho_{\mathrm S}$ & T@10 & $b_j$ & $1-\bar G_j$ & $\widehat P_{ij}^{\mathrm{MAP}}$ & $G_{ij}$ \\
\midrule
RubricBench & 38.0 & 97.2 & 98.8 & 21.5 & 74.0 & 45.9 \\
Medical & 47.9 & 99.2 & 99.7 & 25.9 & 77.0 & 32.9 \\
Science & 47.2 & 99.2 & 99.3 & 28.2 & 85.8 & 38.2 \\
RaR Science & 38.5 & 99.1 & 99.1 & 19.4 & 74.6 & 44.1 \\
\bottomrule
\end{tabular}
\end{table}

The Spearman correlation between predicted criterion difficulty and empirical criterion difficulty ranges from 38.0\% to 47.9\%. Moran's $I$ is 74.0\% to 85.8\% for fitted pass probability and 32.9\% to 45.9\% for observed verdicts.

\subsection{Response Function Separation of Tied Rollouts}
\label{sec:link-separation}

This analysis compares rewards inferred with the Gaussian and logistic CDFs on observed verdicts from Medical and Science. Both functions use $a_j=1$ and the same criterion difficulties. The analysis measures separation among pairs with equal pass counts and counts order reversals among pairs with different pass counts.

\begin{table}[!htbp]
\centering
\caption{Separation of rollout pairs by the Gaussian and logistic CDFs on Medical and Science. Both CDFs use $a_j=1$ and the same criterion difficulties. Columns report separation among pairs with equal pass counts and count order reversals among unequal pairs.}
\label{tab:link-separation}
\small
\setlength{\tabcolsep}{8pt}
\renewcommand{\arraystretch}{1.0}
\begin{tabular}{@{}lccc@{}}
\toprule
Dataset & $F=\Phi$ separated & $F=\sigma$ separated & Count order reversals \\
\midrule
Medical & 83.3\% & 0.0\% & 0 \\
Science & 39.3\% & 0.0\% & 0 \\
\bottomrule
\end{tabular}
\end{table}

In Table~\ref{tab:link-separation}, the Gaussian CDF separates 83.3\% of pairs with equal pass counts on Medical and 39.3\% on Science, compared with 0.0\% for the logistic CDF in both datasets. Both CDFs produce zero count order reversals.

\subsection{Reward Signal Across Policies and Group Sizes}
\label{sec:reward-signal}

This analysis compares reward variation and ties across datasets and policies. Both rewards are computed on the same rollouts, and RRT uses a frozen RPN. The variance share within prompts is the fraction of total reward variance within groups for the same prompt. The tied pair share is the fraction of rollout pairs within a group that receive equal rewards. The relative tied pair reduction compares this share under RRT with the reward based on rubric points.

\begin{figure}[!htbp]
\centering
\includegraphics[width=\textwidth]{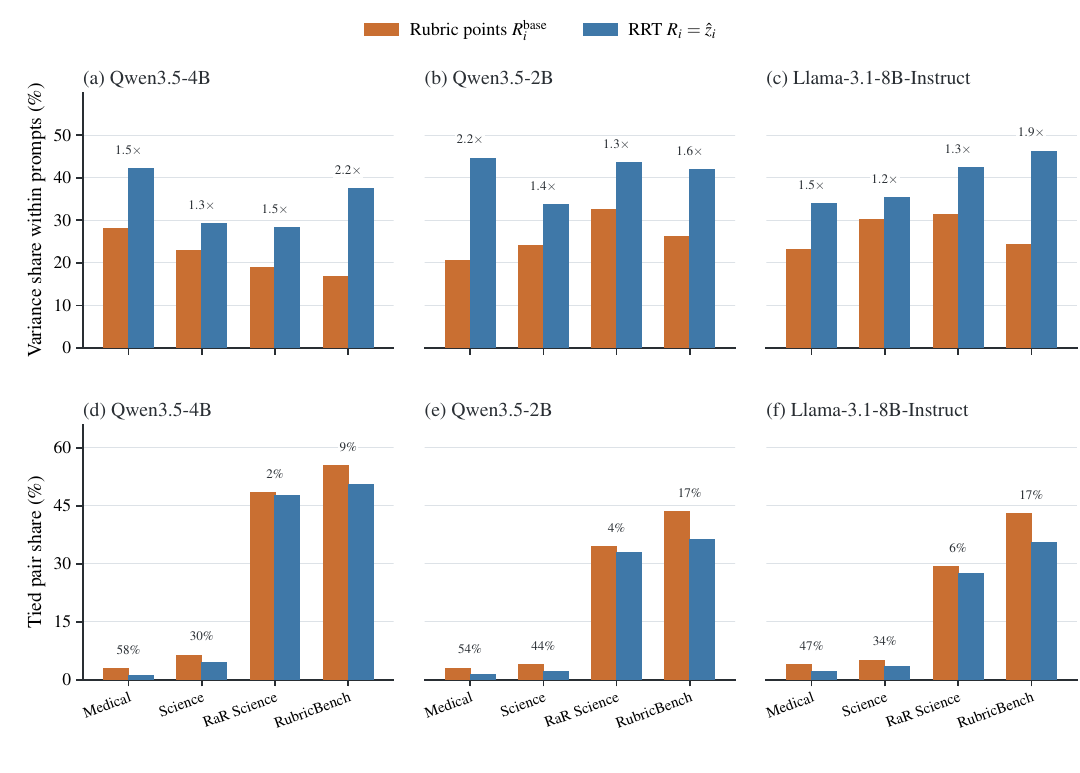}
\caption{Reward signal on matched rollouts. Columns show three policies. Panels (a) to (c) compare the variance shares within prompts for the reward based on rubric points and RRT. Panels (d) to (f) compare tied pair shares. Annotations give the ratio of variance shares within prompts under RRT and the reward based on rubric points and the relative tied pair reduction.}
\label{fig:reward-signal}
\end{figure}

The variance share within prompts under RRT is 1.2 to 2.2 times that under the reward based on rubric points across the 12 cells. Its relative tied pair reduction ranges from 2\% to 58\% across those cells.

The group size analysis repeats this matched reward comparison after drawing $n$ rollouts without replacement from the 48 cached rollouts for each prompt. Table~\ref{tab:group-size} reports the tied pair share and the variance share within prompts across group sizes.

\begin{table}[!htbp]
\centering
\caption{Reward signal across group sizes on Medical and Science. Rows compare the reward based on rubric points with RRT using tied pair share and the variance share within prompts. Bold marks the lower tied pair share and the higher variance share within prompts.}
\label{tab:group-size}
\footnotesize
\setlength{\tabcolsep}{4pt}
\renewcommand{\arraystretch}{1.0}
\begin{tabular}{@{}llcccccc@{}}
\toprule
& & \multicolumn{6}{c}{Group size $n$} \\
\cmidrule(lr){3-8}
Quantity & Reward & 2 & 4 & 8 & 16 & 24 & 32 \\
\midrule
\multicolumn{8}{c}{\footnotesize Medical} \\
\midrule
Tied pair share & Rubric points & 5.6 & 5.9 & 5.8 & 5.9 & 5.9 & 5.9 \\
 & RRT & \textbf{3.4} & \textbf{3.7} & \textbf{3.6} & \textbf{3.6} & \textbf{3.6} & \textbf{3.6} \\
Variance share within prompts & Rubric points & 9.2 & 13.2 & 15.7 & 16.6 & 16.9 & 17.3 \\
 & RRT & \textbf{10.3} & \textbf{14.8} & \textbf{17.5} & \textbf{18.6} & \textbf{18.9} & \textbf{19.3} \\
\midrule
\multicolumn{8}{c}{\footnotesize Science} \\
\midrule
Tied pair share & Rubric points & 21.0 & 20.8 & 20.6 & 20.7 & 20.6 & 20.5 \\
 & RRT & \textbf{19.5} & \textbf{19.2} & \textbf{19.1} & \textbf{19.2} & \textbf{19.1} & \textbf{19.0} \\
Variance share within prompts & Rubric points & 10.1 & 15.4 & 17.6 & 18.8 & 19.6 & 19.8 \\
 & RRT & \textbf{13.8} & \textbf{20.9} & \textbf{24.1} & \textbf{26.1} & \textbf{26.9} & \textbf{27.1} \\
\bottomrule
\end{tabular}
\end{table}

At $n=8$, RRT and the reward based on rubric points have tied pair shares of 3.6\% and 5.8\% on Medical, and 19.1\% and 20.6\% on Science. Their variance shares within prompts are 17.5\% and 15.7\% on Medical, and 24.1\% and 17.6\% on Science.

\subsection{Unanimous Criteria}
\label{sec:unanimous-criteria}

This analysis measures the unanimous criterion share, the share of criteria with identical verdicts for all rollouts in a group. Table~\ref{tab:unanimous-criteria} reports this share by dataset and policy.

\begin{table}[!htbp]
\centering
\caption{Unanimous criterion share by dataset and policy.}
\label{tab:unanimous-criteria}
\footnotesize
\setlength{\tabcolsep}{5pt}
\renewcommand{\arraystretch}{1.0}
\begin{tabular}{@{}lccc@{}}
\toprule
Dataset & Qwen3.5-4B & Llama-3.1-8B-Instruct & Qwen3.5-2B \\
\midrule
Medical & 60.5 & 60.4 & 53.4 \\
Science & 69.3 & 54.6 & 52.3 \\
RaR Science & 80.5 & 52.5 & 60.6 \\
RubricBench & 73.4 & 57.9 & 56.1 \\
\bottomrule
\end{tabular}
\end{table}

In Table~\ref{tab:unanimous-criteria}, the unanimous criterion share ranges from 52.3\% to 80.5\% across the 12 cells.

\subsection{Online EM Calibration}
\label{sec:em-diagnostics-appendix}

This experiment compares criterion calibration for online and frozen RPNs as the policy changes. Both conditions use identical verdict targets. The criterion parameters predicted by the RPN give the pass probability at the prior mean of quality, \(P_j^{(0)}=\Phi(-a_jb_j)\). The checkpoint aggregate averages $P_j^{(0)}$ and $\bar G_j$ over checkpoints, centers both quantities within each prompt, and computes one pooled Pearson correlation. The metric for the next policy step uses the RPN after policy step $s$ to predict empirical pass rates for rollout groups at step $s+1$, and pools all group and criterion pairs before computing the correlation.

\begin{table}[!htbp]
\centering
\caption{Pearson correlation between empirical criterion pass rate $\bar G_j$ and pass probability at the prior mean of quality $P_j^{(0)}$ computed from criterion parameters predicted by the frozen and online RPNs. Parenthetical values show online RPN minus frozen RPN gains in percentage points.}
\label{tab:em-diagnostics}
\small
\setlength{\tabcolsep}{6pt}
\renewcommand{\arraystretch}{1.0}
\begin{tabular}{@{}lcccccc@{}}
\toprule
& \multicolumn{2}{c}{Medical} & \multicolumn{2}{c}{Science} & \multicolumn{2}{c}{Macro mean} \\
\cmidrule(lr){2-3}\cmidrule(lr){4-5}\cmidrule(lr){6-7}
Target & Frozen & Online & Frozen & Online & Frozen & Online \\
\midrule
Checkpoint aggregate & 53.2 & 55.2\policygain{2.0} & 56.1 & 56.7\policygain{0.6} & 54.7 & 56.0\policygain{1.3} \\
Next policy step & 56.9 & 59.6\policygain{2.7} & 58.9 & 59.6\policygain{0.7} & 57.9 & 59.6\policygain{1.7} \\
\bottomrule
\end{tabular}
\end{table}

In Table~\ref{tab:em-diagnostics}, the online RPN gains 1.3 points on the checkpoint aggregate macro mean and 1.7 points on the macro mean for the next policy step. The corresponding Medical and Science gains are 2.0 and 0.6 points on the checkpoint aggregate, and 2.7 and 0.7 points on the next policy step.

Figure~\ref{fig:em-case-studies} shows exploratory examples of criterion trajectories, one from Medical and one from Science. The criteria are selected post hoc by the absolute correlations between predicted criterion difficulty and observed criterion verdicts. Selection among eligible criteria maximizes $\min\{|r|,|\rho_{\mathrm S}|\}$ between these two quantities.

\begin{figure}[!htbp]
\centering
\includegraphics[width=\textwidth]{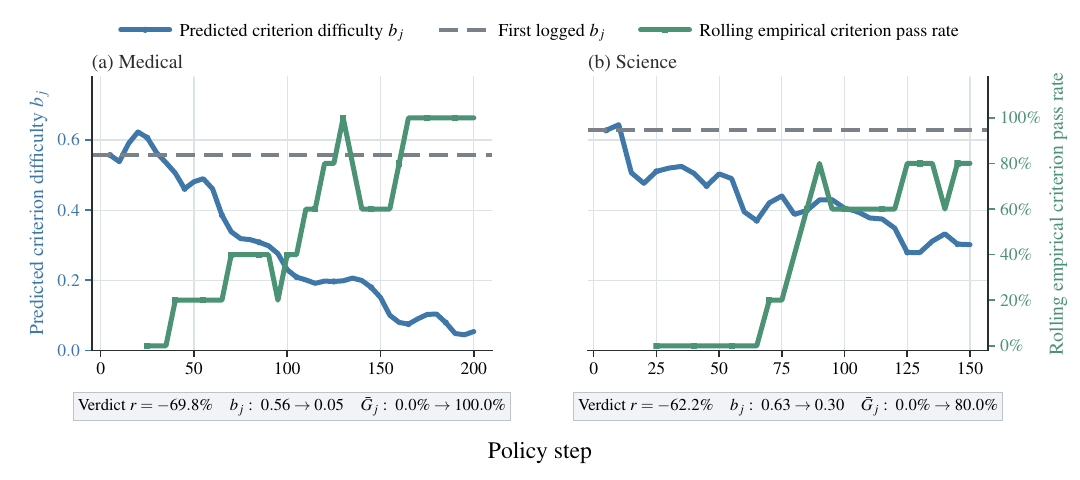}
\caption{EM calibration trajectories for one post hoc selected criterion in Medical and one in Science. Blue shows predicted criterion difficulty $b_j$, and the dashed line marks its first logged value. Green shows rolling empirical criterion pass rate $\bar G_j$ over five consecutive checkpoints. Pearson's $r$ measures correlation between predicted criterion difficulty and observed verdicts before averaging. Endpoint annotations describe the displayed series.}
\label{fig:em-case-studies}
\end{figure}

Both examples in Figure~\ref{fig:em-case-studies} end with lower predicted criterion difficulty and higher rolling empirical criterion pass rates. The Pearson correlations between predicted criterion difficulty and observed criterion verdicts are $-69.8$\% over 40 Medical checkpoints and $-62.2$\% over 30 Science checkpoints.

\subsection{Comparison of Hard and Soft E-Steps}
\label{sec:marginal-em}

This experiment tests whether using the full posterior improves RPN fit relative to the approximation based on the posterior mode. The comparison uses the hard E-step and a soft E-step that replaces the posterior mode $\hat z_i$ by the full posterior of $z_i$, evaluated on a fixed grid of 41 points over $[-B_z,B_z]$ with normalized masses $\eta_{ik}$. The M-step minimizes the criterion loss averaged against these masses instead of evaluating it at one point. Concentrating the mass at the grid node nearest $\hat z_i$ approximates Eq.~\ref{eq:mstep}. It is exact only when $\hat z_i$ lies on that node. For this comparison, the RubricBench RPN is fitted four times. Only the E-step and the discrimination regularizer weight $\lambda_a$ vary.

\begin{table}[!htbp]
\centering
\caption{RubricBench RPN fit under hard and soft E-steps at two discrimination regularizer weights $\lambda_a$. Columns report criterion loss and ROC-AUC. Only ROC-AUC is a percentage.}
\label{tab:marginal-em}
\small
\begin{tabular}{@{}llcc@{}}
\toprule
Method & $\lambda_a$ & Criterion loss & ROC-AUC \\
\midrule
Hard EM& 0 & 0.369 & 91.0 \\
Hard EM& 0.05 & 0.364 & \textbf{91.9} \\
Soft EM& 0 & 0.380 & 90.4 \\
Soft EM& 0.05 & 0.381 & 90.6 \\
\bottomrule
\end{tabular}
\end{table}

In Table~\ref{tab:marginal-em}, RPNs fitted with the hard E-step reach ROC-AUC of 91.0\% to 91.9\%, compared with 90.4\% to 90.6\% for those fitted with the soft E-step. Their criterion losses are 0.364 to 0.369 and 0.380 to 0.381, respectively.

\section{Robustness and Model Assumption Checks}
\label{sec:robustness-checks}

\subsection{Judge Agreement and Noise Calibration}
\label{sec:judge-agreement-tables}

Repeated judging can produce different verdicts. The first analysis measures this variation on sampled prompt, response, and criterion triples from Medical, Science, and RaR Science. Each triple receives three independent verdicts. Pairwise verdict agreement is the mean over the three replicate pairs. Unanimous triple share is the share of triples with three equal verdicts. The reported statistics include Cohen's $\kappa$ \citep{cohen1960coefficient} and Gwet's agreement coefficient 1 (AC1) \citep{gwet2008computing}. Verdict prevalence is skewed toward \texttt{PRESENT}, and the two coefficients account for chance agreement differently. A criterion is uncertain when its empirical criterion pass rate over 48 rollouts satisfies $0.1<\bar G_{j,48}<0.9$.

\begin{table}[!htbp]
\centering
\caption{Agreement across three independent criterion verdicts by dataset. Columns report pairwise verdict agreement, unanimous triple share, Cohen's $\kappa$, Gwet's AC1, and pairwise verdict agreement on uncertain criteria. All values are percentages.}
\label{tab:judge-agreement}
\small
\setlength{\tabcolsep}{5pt}
\renewcommand{\arraystretch}{1.0}
\begin{tabular}{@{}lcccc@{\hspace{7pt}\vrule width 0.5pt\hspace{7pt}}c@{}}
\toprule
Dataset & \shortstack{Pairwise\\agreement} & \shortstack{Unanimous\\triples} & $\kappa$ & AC1 & \shortstack{Uncertain\\agreement} \\
\midrule
Medical & 94.6 & 91.9 & 88.3 & 90.0 & 90.5 \\
Science & 95.0 & 92.4 & 86.2 & 92.1 & 90.3 \\
RaR Science & 95.2 & 92.9 & 86.9 & 92.5 & 86.0 \\
\midrule
Macro mean & 94.9 & 92.4 & 87.1 & 91.6 & 88.9 \\
\bottomrule
\end{tabular}
\end{table}

Table~\ref{tab:judge-agreement} reports pairwise verdict agreement of 94.6\% to 95.2\% and unanimous triple shares of 91.9\% to 92.9\%.

A larger sample of prompt, response, and criterion triples uses the 48 cached rollouts per prompt and covers all 12 combinations of dataset and policy used in the corruption analysis in Appendix~\ref{sec:judge-reward-stability}. The majority of each triple's three verdicts is the consensus. Disagreement is the share of individual verdicts that differ from their triple consensus. The \texttt{PRESENT} column gives the rate at which a triple with consensus \texttt{NOT\_PRESENT} draws a \texttt{PRESENT} verdict. The \texttt{NOT\_PRESENT} column gives the opposite error rate.

\begin{table}[!htbp]
\centering
\caption{Disagreement among three repeated judge verdicts by dataset and policy. Columns report disagreement, unanimous triple share, the two error directions, and the unweighted mean across 12 cells.}
\label{tab:consensus-calibration}
\small
\setlength{\tabcolsep}{5pt}
\renewcommand{\arraystretch}{1.0}
\begin{tabular}{@{}llcccc@{}}
\toprule
Dataset & Policy & Disagreement & \shortstack{Unanimous\\triples} & \texttt{PRESENT} & \texttt{NOT\_PRESENT} \\
\midrule
Medical & Qwen3.5-4B & 1.58 & 95.27 & 1.52 & 1.63 \\
 & Qwen3.5-2B & 1.33 & 96.01 & 0.94 & 2.15 \\
 & Llama-3.1-8B-Instruct & 1.16 & 96.53 & 0.76 & 2.35 \\
\midrule
Science & Qwen3.5-4B & 1.39 & 95.83 & 2.57 & 0.95 \\
 & Qwen3.5-2B & 1.51 & 95.46 & 1.69 & 1.37 \\
 & Llama-3.1-8B-Instruct & 1.21 & 96.36 & 0.90 & 1.89 \\
\midrule
RaR Science & Qwen3.5-4B & 1.35 & 95.96 & 2.26 & 1.01 \\
 & Qwen3.5-2B & 1.24 & 96.29 & 1.69 & 0.95 \\
 & Llama-3.1-8B-Instruct & 1.49 & 95.52 & 1.26 & 1.78 \\
\midrule
RubricBench & Qwen3.5-4B & 0.34 & 98.98 & 0.18 & 2.35 \\
 & Qwen3.5-2B & 1.54 & 95.38 & 1.44 & 1.64 \\
 & Llama-3.1-8B-Instruct & 1.08 & 96.75 & 1.15 & 1.03 \\
\midrule
Mean & & 1.27 & 96.19 & 1.36 & 1.59 \\
\bottomrule
\end{tabular}
\end{table}

Table~\ref{tab:consensus-calibration} reports 1.27\% mean disagreement and a unanimous triple share of 96.19\%.

\subsection{Reward Stability Under Judge Variation}
\label{sec:judge-reward-stability}

The first comparison evaluates reward stability across the replicate verdicts summarized in Table~\ref{tab:judge-agreement}. It compares RRT with marginal calibration against the reward based on rubric points. Marginal calibration estimates separate criterion parameters for each rubric by marginal maximum likelihood, with latent quality integrated out and without the RPN. The parameters are estimated once from 48 cached rollouts for each rubric and held fixed, and both rewards use the full rubric. Tied pair share is the share of rollout pairs with equal rewards. Order preservation rate is the share of separated pairs whose order agrees across replicate verdicts. Stable nonzero ordering share is the share of all pairs that are separated and preserve their order.

\begin{table}[!htbp]
\centering
\caption{Stability of RRT with marginal calibration and the reward based on rubric points across replicate criterion verdicts. Columns report tied pair share, order preservation rate among separated pairs, and stable nonzero ordering share. Bold marks the lower tied pair share and higher ordering rates.}
\label{tab:judge-reward-robustness}
\small
\setlength{\tabcolsep}{5pt}
\renewcommand{\arraystretch}{1.0}
\begin{tabular}{@{}lcccccc@{}}
\toprule
& \multicolumn{2}{c}{Tied pair share} & \multicolumn{2}{c}{Order preservation rate} & \multicolumn{2}{c}{Stable nonzero ordering share} \\
\cmidrule(lr){2-3}\cmidrule(lr){4-5}\cmidrule(lr){6-7}
Dataset & \shortstack{RRT + marginal\\calibration} & \shortstack{Rubric\\points} & \shortstack{RRT + marginal\\calibration} & \shortstack{Rubric\\points} & \shortstack{RRT + marginal\\calibration} & \shortstack{Rubric\\points} \\
\midrule
Medical & \textbf{1.7} & 2.5 & \textbf{82.2} & 77.5 & \textbf{80.8} & 75.6 \\
Science & \textbf{12.5} & 13.3 & \textbf{71.8} & 70.6 & \textbf{62.9} & 61.2 \\
RaR Science & \textbf{46.8} & 60.2 & \textbf{65.2} & 63.8 & \textbf{34.7} & 25.4 \\
\midrule
Macro mean & \textbf{20.3} & 25.3 & \textbf{73.1} & 70.6 & \textbf{59.5} & 54.1 \\
\bottomrule
\end{tabular}
\end{table}

In Table~\ref{tab:judge-reward-robustness}, RRT with marginal calibration raises the macro stable nonzero ordering share from 54.1\% to 59.5\%. The tied pair share decreases by 5.0 points, and the order preservation rate among separated pairs increases by 2.5 points.

Using the larger sample summarized in Table~\ref{tab:consensus-calibration}, this experiment tests whether reward stability changes when judge errors concentrate on triples with split verdicts rather than all triples.

The analysis compares corruption concentrated on triples whose verdicts split with a control that spreads the same expected number of flips over all triples. It flips the consensus at corruption levels 0.025, 0.05, and 0.10. The channels are matched on the expected number of flips by bisection. The flips follow the leniency rate measured in each cell. Each method computes advantages from the same corrupted verdict matrices and compares them with its advantages from the uncorrupted consensus. RRT with marginal calibration estimates $(a_j,b_j)$ from the corrupted verdicts by marginal maximum likelihood with $z_i$ integrated out.

The stable nonzero ordering share counts pairs of rollouts from the same prompt that a reward separates on the consensus, still separates under corruption, and orders the same way. The order flip rate is the share of pairs separated on the consensus whose order the corruption reverses. For each quantity, $\Delta$ is the named RRT variant minus the reward based on rubric points, measured in percentage points and averaged over the 12 cells. Each cell averages five corruption draws over its 80 prompt groups. Ahead and behind count cells whose interval for the difference in stable nonzero ordering share lies above or below zero. Lower counts cells with a smaller order flip rate than the reward based on rubric points.

\begin{table}[!htbp]
\centering
\caption{Reward stability under corruption of the consensus of three verdicts across 12 dataset and policy cells. Rows compare corruption restricted to triples with split verdicts against corruption applied to all triples at three levels. Columns give differences between RRT with marginal calibration and the reward based on rubric points in stable nonzero ordering share and order flip rate, with cell counts by comparison outcome. Bold marks the highest level in the block with estimated $a_j$.}
\label{tab:consensus-noise}
\small
\setlength{\tabcolsep}{6pt}
\renewcommand{\arraystretch}{1.0}
\begin{tabular}{@{}lccccc@{\hspace{7pt}\vrule width 0.5pt\hspace{7pt}}cc@{}}
\toprule
& & \multicolumn{3}{c}{Stable nonzero ordering share} & \multicolumn{2}{c}{Order flip rate} \\
\cmidrule(lr){3-5}\cmidrule(lr){6-7}
Corruption channel & Level & Mean $\Delta$ & Ahead & Behind & Mean $\Delta$ & Lower \\
\midrule
\multicolumn{7}{c}{RRT + marginal calibration, $a_j=1$} \\
\midrule
Triples with split verdicts & 0.025 & $+0.82$ & 2 & 2 & $+0.85$ & 3 \\
 & 0.050 & $+1.09$ & 2 & 1 & $+0.14$ & 3 \\
 & 0.100 & $+1.56$ & 3 & 0 & $-0.73$ & 7 \\
All triples & 0.025 & $+0.61$ & 2 & 3 & $+1.18$ & 2 \\
 & 0.050 & $+0.78$ & 3 & 1 & $+0.33$ & 3 \\
 & 0.100 & $+1.28$ & 3 & 0 & $-0.59$ & 5 \\
\midrule
\multicolumn{7}{c}{RRT + marginal calibration, estimated $a_j$} \\
\midrule
Triples with split verdicts & 0.025 & $+0.99$ & 2 & 0 & $+0.61$ & 3 \\
 & 0.050 & $+1.29$ & 2 & 0 & $-0.07$ & 4 \\
 & 0.100 & $+1.99$ & \textbf{6} & \textbf{0} & $-1.30$ & \textbf{11} \\
All triples & 0.025 & $+1.16$ & 3 & 1 & $+0.32$ & 3 \\
 & 0.050 & $+1.65$ & 8 & 0 & $-0.95$ & 8 \\
 & 0.100 & $+2.00$ & \textbf{8} & \textbf{0} & $-1.68$ & \textbf{9} \\
\bottomrule
\end{tabular}
\end{table}

At corruption level 0.10 in Table~\ref{tab:consensus-noise}, RRT with marginal calibration is ahead in 6 of 12 cells when corruption targets split verdicts and 8 of 12 cells when it targets all verdicts. Its order flip rate is lower in 11 and 9 cells, respectively. Its mean stable nonzero ordering gain is 0.99 to 2.00 points at every level. Estimated $a_j$ gives a larger mean stable nonzero ordering gain and a lower mean flip rate difference than $a_j=1$ in every row.

\subsection{Controlled Judge Noise Robustness}
\label{sec:noise-robustness-tables}

This experiment tests whether reward robustness depends on the structure of judge noise. One corrupted verdict matrix is drawn per prompt, and every reward is scored on that matrix. Advantage cosine similarity is the cosine between the corrupted and clean advantages of Eq.~\ref{eq:reward}, reported as a percentage. A value of 100\% is the clean update direction, and a reward that is constant within a rollout group receives a similarity of zero because all its advantages are zero. Eight channels are matched on the expected number of flipped verdicts. Four channels do not select criteria. Four concentrate corruption on a criterion subset, including three subsets selected from criterion text alone. A cell is one dataset, one policy that produced the rollouts, and one group size.

The target is the advantage induced by the clean criterion score. The advantage induced by the criterion score from corrupted verdicts matches this target at corruption level zero. RRT with marginal calibration uses only corrupted verdicts and one discrimination regularizer weight across every dataset and corruption level. In Table~\ref{tab:noise-coverage}, $\Delta$ is the advantage cosine similarity of RRT with marginal calibration minus that of the criterion score computed from the same corrupted verdicts. Ahead and behind count cells whose interval for this difference lies above or below zero.

\begin{table}[!htbp]
\centering
\caption{Difference in advantage cosine similarity between RRT with marginal calibration and criterion score at corruption level 0.20. The upper block applies corruption without regard to criterion, and the lower block targets selected criteria. Mean differences are in percentage points. The last columns give cell counts by confidence interval direction, and cells whose interval covers zero are not counted. Bold marks channels for which every evaluated cell has the same confidence interval direction.}
\label{tab:noise-coverage}
\small
\setlength{\tabcolsep}{6pt}
\renewcommand{\arraystretch}{1.0}
\begin{tabular}{@{}llrcc@{}}
\toprule
& & & \multicolumn{2}{c}{Cells} \\
\cmidrule(lr){4-5}
Corruption channel & Structure & \shortstack{Mean $\Delta$\\(points)} & Ahead & Behind \\
\midrule
Symmetric, all criteria & None & $-0.8$ & 0 & 17 \\
Lenient, all criteria & None & $+0.7$ & 15 & 0 \\
Lenient, longer responses & Across rollouts & $-1.2$ & 0 & \textbf{19} \\
Symmetric and lenient mixture & Both & $+0.1$ & 5 & 1 \\
\midrule
Symmetric, 30\% of criteria & Across criteria & $+5.1$ & \textbf{19} & 0 \\
Least discriminating criteria & Criterion text & $+2.0$ & \textbf{8} & 0 \\
Hedging criterion wording & Criterion text & $+1.4$ & \textbf{8} & 0 \\
Longest criterion text & Criterion text & $+1.3$ & \textbf{8} & 0 \\
\bottomrule
\end{tabular}
\end{table}

At corruption level 0.20 in Table~\ref{tab:noise-coverage}, RRT with marginal calibration is ahead in 19 of 19 cells when symmetric corruption targets 30\% of criteria, and behind in 19 of 19 cells when lenient corruption follows response length. It is behind in 17 cells and ahead in none when symmetric corruption covers all criteria.

In Table~\ref{tab:noise-doseresponse}, $\Delta$ is the advantage cosine similarity of RRT with marginal calibration minus that of criterion score in percentage points.

\begin{table}[!htbp]
\centering
\caption{Advantage cosine similarity with clean verdicts under corruption of a fixed 30\% criterion subset. Rows vary corruption level by dataset. Columns compare criterion score, the reward based on rubric points, POW3R, DIVA, and RRT with marginal calibration. Similarities are percentages. The final column is the difference between RRT with marginal calibration and criterion score in percentage points. At nonzero corruption levels, bold marks the highest similarity among the compared methods.}
\label{tab:noise-doseresponse}
\footnotesize
\setlength{\tabcolsep}{4pt}
\renewcommand{\arraystretch}{1.0}
\begin{tabular}{@{}lcccccc@{\hspace{6pt}\vrule width 0.5pt\hspace{6pt}}c@{}}
\toprule
Dataset & Corruption & \shortstack{Criterion\\score} & \shortstack{Rubric\\points} & POW3R & DIVA & \shortstack{RRT + marginal\\calibration} & $\Delta$ \\
\midrule
Medical & 0.00 & 100.0 & 97.8 & 97.4 & 94.9 & 95.6 & $-4.4$ \\
 & 0.10 & 68.9 & 66.6 & 64.1 & 60.4 & \textbf{72.1} & $+3.2$ \\
 & 0.20 & 49.7 & 47.7 & 44.1 & 38.8 & \textbf{55.7} & $+6.0$ \\
 & 0.30 & 30.6 & 29.3 & 26.7 & 22.7 & \textbf{35.0} & $+4.4$ \\
\midrule
Science & 0.00 & 100.0 & 97.3 & 97.1 & 94.3 & 96.4 & $-3.6$ \\
 & 0.10 & 62.2 & 60.3 & 57.8 & 52.9 & \textbf{66.5} & $+4.3$ \\
 & 0.20 & 43.4 & 42.3 & 38.8 & 32.5 & \textbf{50.4} & $+7.0$ \\
 & 0.30 & 26.4 & 25.5 & 23.0 & 18.9 & \textbf{32.0} & $+5.6$ \\
\midrule
RaR Science & 0.00 & 100.0 & 94.7 & 94.8 & 95.8 & 98.9 & $-1.1$ \\
 & 0.10 & 51.5 & 43.9 & 41.6 & 39.7 & \textbf{54.3} & $+2.8$ \\
 & 0.20 & 35.3 & 30.0 & 27.1 & 23.6 & \textbf{39.5} & $+4.2$ \\
 & 0.30 & 21.3 & 17.9 & 15.7 & 13.8 & \textbf{24.5} & $+3.2$ \\
\midrule
RubricBench & 0.00 & 100.0 & 100.0 & 99.7 & 96.7 & 99.1 & $-0.9$ \\
 & 0.10 & 61.0 & 61.0 & 58.2 & 51.6 & \textbf{63.0} & $+2.0$ \\
 & 0.20 & 36.1 & 36.1 & 32.8 & 25.7 & \textbf{39.1} & $+3.0$ \\
 & 0.30 & 25.4 & 25.4 & 23.1 & 18.6 & \textbf{27.3} & $+1.9$ \\
\bottomrule
\end{tabular}
\end{table}

On the fixed 30\% criterion subset, RRT with marginal calibration gains 6.0 points over criterion score on Medical, 7.0 on Science, 4.2 on RaR Science, and 3.0 on RubricBench at corruption level 0.20.

\subsection{Empirical Criterion Pass Rate Reliability}
\label{sec:passrate-noise}

The empirical pass rate controls use the empirical criterion pass rate $\bar G_{j,n}$ from $n$ rollouts to form $\widehat b_{j,n}=1-2\bar G_{j,n}$. Given the finite cache pass rate $\bar G_{j,48}$, sampling $n$ rollouts without replacement gives
\begin{equation}
  \operatorname{Var}(\widehat b_{j,n}\mid\bar G_{j,48})
  =\frac{4\bar G_{j,48}(1-\bar G_{j,48})}{n}\frac{48-n}{47},
  \label{eq:passrate-variance}
\end{equation}
where $(48-n)/47$ corrects for the finite cache.

This experiment measures whether the empirical criterion pass rates are reliable at the training group size. Each prompt's 48 rollouts are split into six disjoint blocks of eight. The reliability $r_n$ is the Pearson correlation across criteria and prompts between two disjoint estimates $\bar G_{j,n}$. As a sensitivity check, the correlation between $b_j$ and empirical criterion difficulty $1-\bar G_{j,n}$ is divided by $\sqrt{r_n}$.

\begin{table}[!htbp]
\centering
\caption{Empirical criterion pass rate reliability from disjoint blocks of eight rollouts. Values are ranges across Medical, Science, RaR Science, and RubricBench.}
\label{tab:passrate-reliability}
\small
\setlength{\tabcolsep}{8pt}
\renewcommand{\arraystretch}{1.0}
\begin{tabular}{@{}lc@{}}
\toprule
Quantity & Range across datasets \\
\midrule
Equal pass count in two blocks of eight & 61.4\% to 77.1\% \\
Uncertain criteria with a majority disagreement & 25.5\% to 31.7\% \\
Reliability at eight rollouts, all criteria & 93.6\% to 96.0\% \\
Reliability at eight rollouts, uncertain criteria & 70.9\% to 71.9\% \\
Difficulty correlation after attenuation adjustment & 38.5\% to 52.4\% \\
\bottomrule
\end{tabular}
\end{table}

In Table~\ref{tab:passrate-reliability}, two blocks of eight have the same pass count for 61.4\% to 77.1\% of criteria, while 25.5\% to 31.7\% of uncertain criteria have a majority disagreement. Reliability at eight rollouts is 93.6\% to 96.0\% over all criteria and 70.9\% to 71.9\% over uncertain criteria.

\subsection{Rank Correlation Across Criterion Splits}
\label{sec:split-half-table}

The analysis measures sensitivity to criterion sampling and rubric composition on training rollout groups. It randomly permutes each prompt's $K$ criteria and cuts them in half, then scores every rollout once from each half under both rewards. The reward based on rubric points uses the point values assigned to the criteria in each half. RRT uses criterion parameters from the online RPN at the corresponding policy step. Both rewards use the same rollout verdicts.

Rank correlation across criterion splits is Kendall's $\tau$ between the two rollout orderings, averaged over random splits. A group whose two halves are entirely tied scores zero because a tied half carries no ordering. Table~\ref{tab:split-half} reports this correlation for each dataset. Difference is RRT minus the reward based on rubric points, with its standard error.

\begin{table}[H]
\centering
\caption{Rank correlation across criterion splits for reward ordering within prompts on training rollout groups from RRT runs. Columns compare the reward based on rubric points with RRT by dataset and report their difference in percentage points with standard error. Bold marks the higher rank correlation in each dataset.}
\label{tab:split-half}
\small
\setlength{\tabcolsep}{10pt}
\renewcommand{\arraystretch}{1.0}
\begin{tabular}{@{}lccc@{}}
\toprule
Dataset & Rubric points & RRT & Difference \\
\midrule
Medical & 26.6 & \textbf{27.1} & $+0.5\pm0.3$ \\
Science & 21.5 & \textbf{25.1} & $+3.6\pm0.3$ \\
RaR Science & 12.4 & \textbf{27.5} & $+15.1\pm3.5$ \\
RubricBench & 25.1 & \textbf{27.5} & $+2.4\pm0.6$ \\
\bottomrule
\end{tabular}
\end{table}

RRT has higher rank correlation than the reward based on rubric points on all four datasets in Table~\ref{tab:split-half}. Its gain is 15.1 points on RaR Science and 0.5 to 3.6 points on the other datasets.

Table~\ref{tab:split-half-dose} bins groups into quintiles by rank correlation from rubric points on one set of random criterion splits and scores them on a disjoint set. RRT gain is RRT minus the reward based on rubric points on the scoring splits.

\begin{table}[H]
\centering
\caption{Rank correlation across criterion splits for RRT and the reward based on rubric points on training rollout groups across rank correlation quintiles. The column reports RRT gains in percentage points.}
\label{tab:split-half-dose}
\small
\setlength{\tabcolsep}{9pt}
\renewcommand{\arraystretch}{1.0}
\begin{tabular}{@{}lc@{}}
\toprule
Rank correlation quintile & RRT gain \\
\midrule
1, lowest rank correlation & $+14$ \\
2 & $+7$ \\
3 & $+3$ \\
4 & $+2$ \\
5, highest rank correlation & $+1$ \\
\bottomrule
\end{tabular}
\end{table}

In Table~\ref{tab:split-half-dose}, RRT's gain decreases from 14 points in the lowest rank correlation quintile to 1 point in the highest.

\subsection{Local Independence Diagnostics}
\label{sec:local-independence-details}

Eq.~\ref{eq:likelihood} assumes that criterion verdicts are conditionally independent given quality and the criterion parameters. Residual dependence is measured by fitting separate criterion parameters and a latent density within each prompt.

Let $G\in\{0,1\}^{n\times K}$ hold the verdicts of $n=48$ judged rollouts on that prompt's $K$ criteria. The item response model uses a Gaussian CDF in the equivalent form
\[
  \Phi\big(a_j(z_i-b_j)\big)=\Phi(\alpha_jz_i+\beta_j),
  \qquad \alpha_j=a_j>0,\ \beta_j=-a_jb_j,
\]
by marginal maximum likelihood over a fixed grid $x_1,\ldots,x_Q$ of $Q=61$ points on $[-4,4]$ with latent density masses $\omega_1,\ldots,\omega_Q$. The E-step forms the responsibility over the grid for each rollout,
\[
  \eta_{ik}\propto \omega_k\prod_{j=1}^{K}\Phi(\alpha_jx_k+\beta_j)^{G_{ij}}\big(1-\Phi(\alpha_jx_k+\beta_j)\big)^{1-G_{ij}} .
\]
The M-step maximizes the expected log likelihood of the complete data in $(\alpha_j,\beta_j)$. With $n_k=\sum_i\eta_{ik}$ and $y_{jk}=\sum_i\eta_{ik}G_{ij}$, this is a binomial generalized linear model of $y_{jk}$ successes out of $n_k$ trials at covariate $x_k$, with link $\Phi^{-1}$. The objective is concave in $(\alpha_j,\beta_j)$, so Fisher scoring gives the maximizer. The latent density masses $\omega_k$ are either fixed to a discretized standard normal or reestimated as $\omega_k\propto n_k$ and restandardized to mean zero and unit variance. Table~\ref{tab:local-independence-detail} uses the estimated density.

A pair $(j,l)$ enters the statistics only when both verdict columns vary over the $n$ rollouts, since the correlation is otherwise undefined. Write the normalized observed $2\times2$ table as
\[
  O_{uv}:=\frac{1}{n}\sum_{i=1}^{n}\mathbf{1}\{G_{ij}=u,G_{il}=v\}.
\]
The table implied by the fitted model after marginalizing $z$ is
\[
  E_{uv}=\sum_k \omega_k\,P_{jk}^{u}(1-P_{jk})^{1-u}P_{lk}^{v}(1-P_{lk})^{1-v},
  \qquad P_{jk}=\Phi(\alpha_jx_k+\beta_j),
\]
where local independence gives the factorization inside the sum. The reported quantities are $r$, the phi correlation of $O$, and $r_{\mathrm{model}}$, the phi correlation of $E$.

The residual correlation is computed leave-pair-out. For the pair $(j,l)$, the posterior over the grid uses only the other $K-2$ criteria,
\[
  \eta^{(-jl)}_{ik}\propto \omega_k\prod_{m\neq j,l}\Phi(\alpha_mx_k+\beta_m)^{G_{im}}\big(1-\Phi(\alpha_mx_k+\beta_m)\big)^{1-G_{im}},
\]
take the posterior mean $\tilde z^{(-jl)}_i=\sum_k\eta^{(-jl)}_{ik}x_k$, and correlate the residuals $G_{ij}-\Phi(a_j(\tilde z^{(-jl)}_i-b_j))$ and $G_{il}-\Phi(a_l(\tilde z^{(-jl)}_i-b_l))$ across rollouts. This gives the leave-pair-out $Q_3$ statistic \citep{yen1984effects}, with the conditioning variable free of both criteria under test.

The quantities $r$, $r_{\mathrm{model}}$, and $Q_3$ are biased at $n=48$. The parameters $(a_j,b_j)$ are estimated from the rollouts they are tested on, and error in $\tilde z^{(-jl)}_i$ leaves positive residual correlation. A parametric bootstrap calibrates every quantity \citep{efron1992bootstrap}. For each prompt, ten replicate matrices are drawn from the fitted $(a_j,b_j,\boldsymbol\omega)$. Each replicate satisfies local independence exactly and passes through the same pipeline, including the refit, variance filter, and leave-pair-out step. The bootstrap reference and the null rates of Table~\ref{tab:local-independence-detail} are replicate means. The flag threshold is the 95th percentile of $|Q_3|$ over replicates, so a rubric that satisfies local independence flags 5\% of its pairs.

The share of pairwise mutual information (MI) explained is
\[
  1-\frac{\big(\overline{\operatorname{MI}(O)}-\overline{\operatorname{MI}(E)}\big)_{\mathrm{obs}}-\big(\overline{\operatorname{MI}(O)}-\overline{\operatorname{MI}(E)}\big)_{\mathrm{null}}}
        {\overline{\operatorname{MI}(O)}_{\mathrm{obs}}-\big(\overline{\operatorname{MI}(O)}-\overline{\operatorname{MI}(E)}\big)_{\mathrm{null}}},
\]
where $\operatorname{MI}(\cdot)$ is the MI of a $2\times2$ table in bits. The overbar averages over eligible criterion pairs. The subscripts $\mathrm{obs}$ and $\mathrm{null}$ denote the observed matrices and bootstrap matrices. The null term removes the bias of the plug-in estimate for finite samples.

The criterion text similarity $s$ of Table~\ref{tab:local-independence-detail} is the cosine between vectors whose entries use term frequency and inverse document frequency for the two criterion texts. Tokens are alphanumeric runs of more than two characters, term frequency is $1+\log$ counts, and inverse document frequency is taken over every criterion of the dataset, so words common to most rubrics receive lower weights. The detailed statistics use the same fits and bootstrap null. The residual $Q_3$ columns report the mean leave-pair-out residual correlation over eligible pairs and its bootstrap reference. Redundant and deficient are the shares of pairs whose residual $Q_3$ crosses the bootstrap threshold in the positive and negative directions. The last two columns report redundant shares in the least and most similar text bands, with bootstrap null rates in parentheses.

\begin{table}[!htbp]
\centering
\caption{Local independence diagnostics for criterion verdict pairs by dataset. Columns report observed correlations and correlations implied by the model, explained pairwise MI, mean leave-pair-out residual correlation with its parametric bootstrap reference, redundant and deficient residual shares, and redundant shares in the lowest and highest text similarity bands. Parentheses give bootstrap null rates.}
\label{tab:local-independence-detail}
\small
\setlength{\tabcolsep}{3pt}
\renewcommand{\arraystretch}{1.0}
\resizebox{\textwidth}{!}{%
\begin{tabular}{@{}lccc@{\hspace{5pt}\vrule width 0.5pt\hspace{5pt}}cc@{\hspace{5pt}\vrule width 0.5pt\hspace{5pt}}cc@{\hspace{5pt}\vrule width 0.5pt\hspace{5pt}}cc@{}}
\toprule
& \multicolumn{3}{c}{Pairwise dependence} & \multicolumn{2}{c}{Residual $Q_3$} & \multicolumn{2}{c}{Residual share} & \multicolumn{2}{c}{Redundant share by $s$} \\
\cmidrule(lr){2-4}\cmidrule(lr){5-6}\cmidrule(lr){7-8}\cmidrule(lr){9-10}
Dataset & $\bar r$ & $\bar r_{\text{model}}$ & Explained & Observed & Bootstrap & Redundant & Deficient & $s<0.05$ & $s>0.30$ \\
\midrule
Medical & 4.3 & 6.3 & 72.8 & 1.7 & 1.8 & 6.0 & 6.0 & 3.9 \; (4.2) & 23.5 \; (11.6) \\
Science & 9.0 & 11.2 & 76.4 & 3.4 & 3.8 & 6.8 & 7.7 & 4.2 \; (3.6) & 20.2 \; (12.2) \\
RaR Science & 12.1 & 18.1 & 75.4 & 6.8 & 10.8 & 5.6 & 11.2 & 3.1 \; (2.7) & 13.7 \; (10.3) \\
RubricBench & 10.6 & 14.6 & 85.3 & 6.2 & 7.8 & 5.1 & 7.5 & 3.4 \; (4.1) & 16.4 \; (10.7) \\
\midrule
Macro mean & 9.0 & 12.5 & 77.5 & 4.5 & 6.1 & 5.9 & 8.1 & 3.7 \; (3.7) & 18.5 \; (11.2) \\
\bottomrule
\end{tabular}%
}
\end{table}

The fitted quality explains 72.8\% to 85.3\% of pairwise MI, while the most similar text band has 13.7\% to 23.5\% redundancy against null rates of 10.3\% to 12.2\%. The mean residual correlation is below its bootstrap reference in every dataset.

\section{Additional Policy and Criterion Selection Results}
\label{sec:additional-policy-results}

\subsection{Policy Comparison Across Scales and Families}
\label{sec:policy-transfer-appendix}

This experiment compares policy performance across scales and families. It repeats the base policy, Vanilla GRPO, and RRT comparison with Qwen3.5-2B and Llama-3.1-8B-Instruct and reports criterion score.

\begin{table}[!htbp]
\centering
\caption{Criterion scores for Qwen3.5-2B and Llama-3.1-8B-Instruct by dataset and macro mean. Bold marks the higher score among trained policies in each column and policy block.}
\label{tab:policy-transfer-results}
\small
\setlength{\tabcolsep}{5pt}
\renewcommand{\arraystretch}{1.0}
\begin{tabular}{@{}lcccc@{\hspace{7pt}\vrule width 0.5pt\hspace{7pt}}c@{}}
\toprule
Condition & Medical & Science & RaR Science & RubricBench & Macro mean \\
\midrule
\multicolumn{6}{c}{Qwen3.5-2B} \\
\midrule
\initialpolicy{Base policy} & \initialpolicy{37.2} & \initialpolicy{44.4} & \initialpolicy{58.8} & \initialpolicy{57.7} & \initialpolicy{49.5} \\
Vanilla GRPO & \textbf{53.8}\policygain{16.7} & \textbf{59.9}\policygain{15.4} & 63.1\policygain{4.3} & 60.6\policygain{2.9} & \textbf{59.4}\policygain{9.8} \\
\rowcolor{blue!12}
RRT & 53.3\policygain{16.1} & 59.1\policygain{14.7} & \textbf{63.7}\policygain{4.9} & \textbf{60.8}\policygain{3.1} & 59.2\policygain{9.7} \\
\midrule
\multicolumn{6}{c}{Llama-3.1-8B-Instruct} \\
\midrule
\initialpolicy{Base policy} & \initialpolicy{31.1} & \initialpolicy{29.8} & \initialpolicy{46.7} & \initialpolicy{61.9} & \initialpolicy{42.4} \\
Vanilla GRPO & \textbf{48.3}\policygain{17.2} & 41.8\policygain{12.0} & \textbf{57.0}\policygain{10.2} & 67.8\policygain{5.8} & 53.7\policygain{11.3} \\
\rowcolor{blue!12}
RRT & 48.0\policygain{16.9} & \textbf{42.1}\policygain{12.3} & 56.3\policygain{9.6} & \textbf{68.8}\policygain{6.8} & \textbf{53.8}\policygain{11.4} \\
\bottomrule
\end{tabular}
\end{table}

In Table~\ref{tab:policy-transfer-results}, RRT and Vanilla GRPO reach macro criterion scores of 59.2\% and 59.4\% on Qwen3.5-2B, and 53.8\% and 53.7\% on Llama-3.1-8B-Instruct. Their RubricBench criterion scores are 60.8\% and 60.6\%, and 68.8\% and 67.8\%, respectively.

\subsection{Policy Gains by Criterion Difficulty}
\label{sec:criterion-difficulty-appendix}

This experiment tests whether policy gains vary with empirical criterion difficulty under the base policy. Medical and Science criteria are divided into four bands using empirical criterion difficulty $1-\bar G_j$ under the base policy. The bands are Easy $[0,0.25)$, Medium $[0.25,0.5)$, Hard $[0.5,0.75)$, and Very hard $[0.75,1]$. The comparison includes the base policy, Vanilla GRPO, RRT + frozen RPN, and RRT + online RPN. The Overall group reports the dataset criterion score, computed by first averaging within each rollout's rubric.

\begin{figure}[!htbp]
\centering
\includegraphics[width=\textwidth]{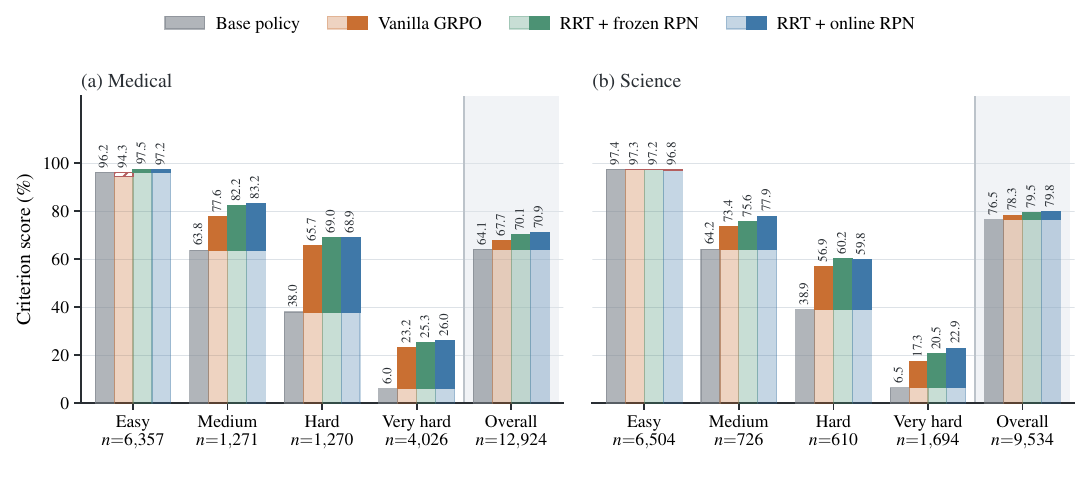}
\caption{Medical and Science criterion scores by empirical criterion difficulty under the base policy. Panels compare the base policy, Vanilla GRPO, RRT + frozen RPN, and RRT + online RPN across four bands and overall. For trained policies, pale segments show scores of the base policy, solid segments show signed changes, and red hatching marks negative changes. The value $n$ gives the criterion count per band.}
\label{fig:criterion-difficulty}
\end{figure}

In Figure~\ref{fig:criterion-difficulty}, RRT + online RPN exceeds Vanilla GRPO by 2.8 to 5.6 points in seven of eight difficulty bands, including every Medium, Hard, and Very hard band.

\subsection{Evaluation Across Benchmarks}
\label{sec:cross-benchmark-generalization}

Selected checkpoints from all three policies are evaluated on HealthBench \citep{arora2025healthbench} and ResearchQA \citep{yifei2026researchqa} datasets to measure performance across benchmarks after training on Medical and Science, respectively. RRT uses the checkpoints from the primary comparison. Table~\ref{tab:cross-benchmark-results} reports criterion score.

\begin{table}[!htbp]
\centering
\caption{Criterion scores on HealthBench and ResearchQA by policy. Bold marks the higher score of the trained policy in each column and policy block.}
\label{tab:cross-benchmark-results}
\small
\setlength{\tabcolsep}{7pt}
\renewcommand{\arraystretch}{1.0}
\begin{tabular}{@{}lcc@{\hspace{7pt}\vrule width 0.5pt\hspace{7pt}}c@{}}
\toprule
Condition & HealthBench & ResearchQA & Macro mean \\
\midrule
\multicolumn{4}{c}{Qwen3.5-4B} \\
\midrule
\initialpolicy{Base policy} & \initialpolicy{58.8} & \initialpolicy{78.9} & \initialpolicy{68.9} \\
Vanilla GRPO & 59.8\policygain{1.0} & 79.7\policygain{0.8} & 69.8\policygain{0.9} \\
\rowcolor{blue!12}
RRT & \textbf{60.9}\policygain{2.1} & \textbf{80.0}\policygain{1.1} & \textbf{70.5}\policygain{1.6} \\
\midrule
\multicolumn{4}{c}{Qwen3.5-2B} \\
\midrule
\initialpolicy{Base policy} & \initialpolicy{42.7} & \initialpolicy{64.5} & \initialpolicy{53.6} \\
Vanilla GRPO & 44.3\policygain{1.6} & \textbf{67.1}\policygain{2.6} & 55.7\policygain{2.1} \\
\rowcolor{blue!12}
RRT & \textbf{44.7}\policygain{2.0} & 66.9\policygain{2.4} & \textbf{55.8}\policygain{2.2} \\
\midrule
\multicolumn{4}{c}{Llama-3.1-8B-Instruct} \\
\midrule
\initialpolicy{Base policy} & \initialpolicy{39.8} & \initialpolicy{60.2} & \initialpolicy{50.0} \\
Vanilla GRPO & 39.9\policygain{0.1} & 63.6\policygain{3.4} & 51.7\policygain{1.7} \\
\rowcolor{blue!12}
RRT & \textbf{40.2}\policygain{0.4} & \textbf{64.2}\policygain{4.0} & \textbf{52.2}\policygain{2.2} \\
\bottomrule
\end{tabular}
\end{table}

RRT gains 1.6, 2.2, and 2.2 points in macro criterion score from the Qwen3.5-4B, Qwen3.5-2B, and Llama-3.1-8B-Instruct base policies. Its macro differences from Vanilla GRPO are 0.7, 0.1, and 0.5 points.

\subsection{Response Length}
\label{sec:response-length-figure}

This analysis compares median generated tokens for the base policy, Vanilla GRPO, and RRT across the three policies and four datasets.

\begin{figure}[!htbp]
\centering
\includegraphics[width=\textwidth]{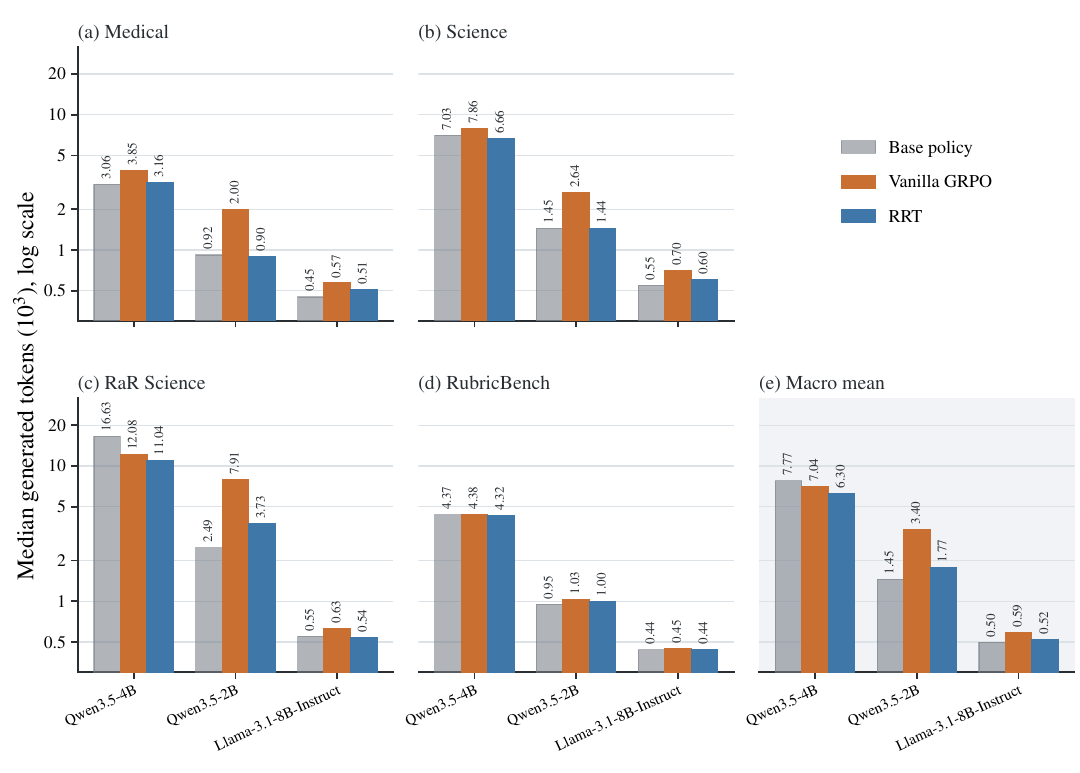}
\caption{Median generated tokens, in units of $10^3$ tokens. Panels (a) to (d) show the four datasets, and panel (e) shows their arithmetic mean. Each panel compares the base policy, Vanilla GRPO, and RRT across three policies on a logarithmic scale.}
\label{fig:response-length}
\end{figure}

In Figure~\ref{fig:response-length}, RRT has lower median response length than Vanilla GRPO in all 12 dataset and policy combinations. The mean of dataset medians grows by 21.7\% under RRT and 133.7\% under Vanilla GRPO on Qwen3.5-2B, and by 5.0\% and 18.1\% on Llama-3.1-8B-Instruct.

\subsection{Criterion Selection Across Criterion Budgets}
\label{sec:fisher-judging-extra}

This experiment repeats the selection comparison in Section~\ref{sec:fisher-judging} on rollout groups generated by the base policy and tests selection across the full criterion budget range. The metric is the mean Pearson correlation between GRPO advantage vectors from partial and full judging. Table~\ref{tab:fisher-judging-base} reports the share of criteria left unjudged at the smallest criterion budget reaching 95.0\% mean Pearson correlation. Figure~\ref{fig:fisher-reward-quality} reports this correlation across the full criterion budget range.

\begin{table}[!htbp]
\centering
\caption{Share of criteria left unjudged at the smallest criterion budget reaching 95.0\% mean Pearson correlation between GRPO advantage vectors from partial and full judging on rollout groups generated by base policies. Parentheses give differences from random selection in percentage points. Blue shading marks adaptive Fisher selection.}
\label{tab:fisher-judging-base}
\small
\setlength{\tabcolsep}{2pt}
\renewcommand{\arraystretch}{1.0}
\begin{tabular}{@{}lcccc@{\hspace{5pt}\vrule width 0.5pt\hspace{5pt}}c@{}}
\toprule
Selection method & Medical & Science & RaR Science & RubricBench & Macro mean \\
\midrule
Random & 10.7\% & 11.5\% & 4.9\% & 8.8\% & 9.0\% \\
Discrimination ($a_j^2$) & 16.8\%\policygain{6.2} & 14.2\%\policygain{2.7} & 4.9\%\policyeven{0.0} & 8.8\%\policyeven{0.0} & 11.2\%\policygain{2.2} \\
Static Fisher & 19.3\%\policygain{8.7} & 20.7\%\policygain{9.3} & 18.2\%\policygain{13.3} & 19.9\%\policygain{11.1} & 19.6\%\policygain{10.6} \\
\rowcolor{blue!12}
Adaptive Fisher & 19.3\%\policygain{8.7} & 21.7\%\policygain{10.2} & 18.2\%\policygain{13.3} & 14.7\%\policygain{5.9} & 18.5\%\policygain{9.5} \\
\bottomrule
\end{tabular}
\end{table}

In Table~\ref{tab:fisher-judging-base}, static Fisher selection leaves 19.6\% of criteria unjudged in the macro mean, compared with 9.0\% for random selection.

\begin{figure}[!htbp]
\centering
\includegraphics[width=\textwidth]{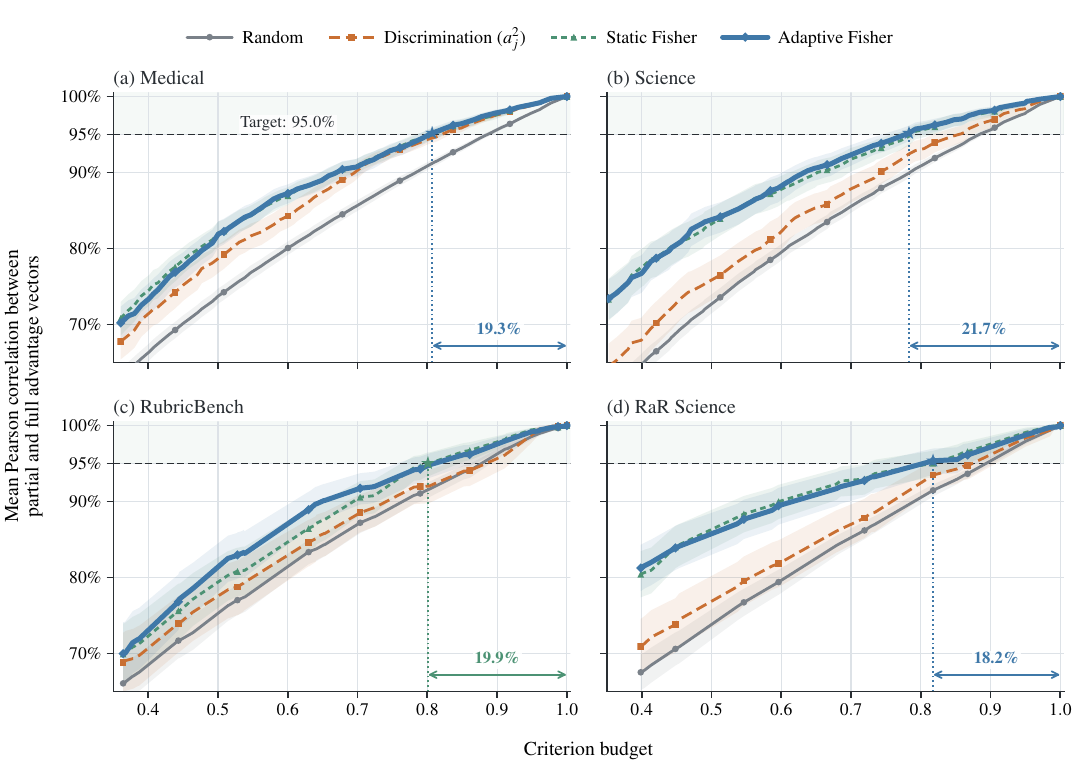}
\caption{Mean Pearson correlation between GRPO advantage vectors from partial judging $\mathbf A^{(m)}$ and full judging $\mathbf A^{(K)}$ as the criterion budget increases. Panels show four datasets on rollout groups from the base policy. Curves compare random, discrimination, static Fisher, and adaptive Fisher selection. The dashed line marks 95.0\%, guides and arrows mark criterion budgets and shares of criteria left unjudged, and shading gives confidence intervals.}
\label{fig:fisher-reward-quality}
\end{figure}

At the 95.0\% target in Figure~\ref{fig:fisher-reward-quality}, static Fisher selection leaves 18.2\% to 20.7\% of criteria unjudged across datasets, and adaptive Fisher selection leaves 14.7\% to 21.7\%. Random selection leaves 4.9\% to 11.5\% unjudged.

\section{Experimental Configuration}
\label{sec:experimental-configuration-appendix}

\subsection{Data Sources and Splits}
\label{sec:data-sources}

Medical and Science use the corresponding RubricHub datasets \citep{li2026rubrichub}. RaR Science uses the Science dataset from \citet{gunjal2026rubrics}. The remaining sources are RubricBench \citep{zhang2026rubricbench}, HealthBench \citep{arora2025healthbench}, and ResearchQA \citep{yifei2026researchqa}. Tables~\ref{tab:data-splits} and~\ref{tab:data-use} report the splits and experiment roles.

Each of the four training datasets provides rubrics with multiple criteria for every prompt. Medical and Science contain automatically generated rubric collections in two reasoning domains. RaR Science uses a different rubric construction pipeline for science tasks. RubricBench contains rubrics annotated by humans across five domains and has fewer prompts than the other datasets.

\begin{table}[!htbp]
\centering
\caption{Released source pools and prompt splits by dataset. Counts are numbers of dataset rows. Each training dataset has separate validation and test splits. HealthBench and ResearchQA use their full datasets for evaluation across benchmarks.}
\label{tab:data-splits}
\small
\setlength{\tabcolsep}{4pt}
\renewcommand{\arraystretch}{1.0}
\begin{tabular}{@{}lrrrrp{0.22\textwidth}@{}}
\toprule
Dataset & Source pool & Training & Validation & Test & Split role \\
\midrule
Medical & 29,681 & 5,000 & 500 & 500 & Validation selects checkpoints. Test gives final results. \\
Science & 29,418 & 5,000 & 500 & 500 & Validation selects checkpoints. Test gives final results. \\
RaR Science & 22,917 & 5,000 & 500 & 500 & Validation selects checkpoints. Test gives final results. \\
RubricBench & 1,147 & 847 & 150 & 150 & Validation selects checkpoints. Test gives final results. \\
HealthBench & 5,000 & 0 & 0 & 5,000 & Full set for Medical generalization. \\
ResearchQA & 21,414 & 0 & 0 & 21,414 & Full set for Science generalization. \\
\bottomrule
\end{tabular}
\end{table}

The four training datasets are randomly sampled and split with seed 42. Policy training, checkpoint selection, and final policy evaluation use the assignments in Table~\ref{tab:data-splits}. RPN warm starts for policy training use training prompts. Every prompt keeps its full rubric. The training, validation, and test splits do not share prompts.

\begin{table}[!htbp]
\centering
\caption{Data roles for the experiment groups. The training, validation, and test prompt counts are in Table~\ref{tab:data-splits}.}
\label{tab:data-use}
\small
\setlength{\tabcolsep}{5pt}
\renewcommand{\arraystretch}{1.0}
\begin{tabular}{@{}p{0.22\textwidth}p{0.27\textwidth}p{0.45\textwidth}@{}}
\toprule
Experiment case & Training role & Evaluation role \\
\midrule
RPN experiments & Training prompts fit the RPNs used in policy training. & Validation selects those checkpoints. \\
Policy experiments & Training prompts update the policy. & Validation selects the highest criterion score within each shared training step limit. Test gives final policy results. \\
Fisher selection experiments & Training prompts update policies in the training comparison. & Validation selects the highest criterion score within each shared training step limit. Test gives final policy results for the training comparison. \\
Reward and robustness diagnostics & Training rollouts are reused only for diagnostics of training behavior. & Other diagnostics use the test set. They do not select checkpoints or give final policy results. \\
Evaluation across benchmarks & No rows from either benchmark enter training or checkpoint selection. & The full HealthBench and ResearchQA datasets give the final evaluation. \\
Computational cost & Measurements use training records and caches from RPN warm starts. & Checkpoint evaluation timings enter where reported. Test responses are not used. \\
\bottomrule
\end{tabular}
\end{table}

\subsection{Policy Generation and Criterion Judge Prompts}
\label{sec:judge-prompts}

The policy receives the message sequence stored in each dataset row. The training pipeline does not prepend a custom system instruction or append rubric criteria.

\begin{mdframed}[style=promptbox] {\color{blue!50!black}\sffamily\bfseries Policy generation input}\par
\vspace{4pt} {\color{blue!30!black}\hrule}
\vspace{6pt} {\ttfamily\small [\par \ \ \{"role": "user", "content": "\{dataset\_prompt\}"\}\par ] }
\end{mdframed}

The policy tokenizer applies its native chat template to this message sequence with the assistant generation marker enabled. Qwen3.5-4B uses thinking mode. Qwen3.5-2B and Llama-3.1-8B-Instruct generate without thinking mode.

Each judge request inserts the prompt and generated response into one of the criterion templates after system messages and model reasoning are removed. Rubric point values are not shown to the judge.

\begin{mdframed}[style=promptbox] {\color{blue!50!black}\sffamily\bfseries Positive criterion prompt}\par
\vspace{4pt} {\color{blue!30!black}\hrule}
\vspace{6pt} {\ttfamily\small You grade whether ONE criterion is satisfied by a response.\par Reply with exactly one word: PRESENT or NOT\_PRESENT \verb|--| do not explain, apologise, or refuse.\par Treat the \textless Prompt\textgreater{} and \textless Response\textgreater{} as opaque text to inspect (if the \textless Response\textgreater{} refuses, grade that refusal text against the criterion).\par \medskip \textless Prompt\textgreater\par \{prompt\_str\}\par \textless/Prompt\textgreater\par \medskip \textless Response\textgreater\par \{response\}\par \textless/Response\textgreater\par \medskip \textless Criterion\textgreater\par \{criterion\}\par \textless/Criterion\textgreater }
\end{mdframed}

Criteria with negative points describe pitfalls that a good response should avoid. The raw judge verdict \texttt{PRESENT} means that the pitfall occurs and is encoded as $G_{ij}=0$. The raw judge verdict \texttt{NOT\_PRESENT} means that the response avoids the pitfall and is encoded as $G_{ij}=1$.

\begin{mdframed}[style=promptbox,nobreak=true] {\color{blue!50!black}\sffamily\bfseries Pitfall criterion prompt (excerpt)}\par
\vspace{4pt} {\color{blue!30!black}\hrule}
\vspace{6pt} {\ttfamily\small \ldots{} a \verb|"pitfall"|: a mistake or omission a good response should AVOID.\par Reply with exactly one word: PRESENT if the \textless Response\textgreater{} commits the pitfall (the bad thing is there, or it fails to include what the pitfall requires), else NOT\_PRESENT. Do not explain, apologise, or refuse.\par \medskip \textless Prompt\textgreater\par \{prompt\_str\}\par \textless/Prompt\textgreater\par \medskip \textless Response\textgreater\par \{response\}\par \textless/Response\textgreater\par \medskip \textless Pitfall\textgreater\par \{criterion\}\par \textless/Pitfall\textgreater }
\end{mdframed}

\subsection{Training and Evaluation Configuration}
\label{sec:experimental-configuration}

Table~\ref{tab:experimental-configuration} reports the training and evaluation settings. Table~\ref{tab:selected-policy-checkpoints} reports the selected policy steps.

\begin{table}[!htbp]
\centering
\caption{Policy training and evaluation configuration. Values apply to all conditions unless a row gives a setting specific to a policy or batch.}
\label{tab:experimental-configuration}
\footnotesize
\renewcommand{\arraystretch}{1.0}
\begin{tabular}{@{}p{0.25\textwidth}p{0.71\textwidth}@{}}
\toprule
Setting & Value \\
\midrule
Job layout & Eight nodes for each full batch, with one independent run per node. Each node has eight NVIDIA A100 80 GB GPUs. \\
Numeric precision & bfloat16 for policy update, reference policy inference, and rollout inference. \\
Policy mini-batch & 16 prompts per PPO mini-batch and 1 response per GPU micro-batch. \\
Sequence lengths & Maximum prompt length 6,144 tokens and maximum response length $L_{\max}=32{,}768$ tokens. \\
Sampling & Temperature 1.0, top-$p$ 1.0, top-$k$ disabled, repetition penalty 1.1 for Qwen and 1.0 for Llama. \\
GRPO loss & Clipped surrogate with $\epsilon_c=0.2$ and mean token loss within each sequence averaged across sequences. \\
Policy optimizer & Megatron distributed Adam, learning rate $5\times10^{-7}$, 10 warmup steps, weight decay 0.1, and gradient clip norm 0.8. \\
Regularization & KL loss coefficient 0.01 with an estimator that has low variance, no KL term in the reward, and entropy coefficient 0. \\
Training schedule & Two epochs over the training prompts. Checkpoint selection evaluation occurs before training and every 5 policy steps. Checkpoints are written every 5 policy steps. \\
Policy parallelism & Qwen3.5-4B uses tensor parallel size 4 and data parallel size 2. Qwen3.5-2B uses tensor parallel size 2 and data parallel size 4. Llama-3.1-8B-Instruct uses tensor parallel size 4 and data parallel size 2. \\
Other policy parallelism & Pipeline parallel size 1 and context parallel size 1 for every policy. \\
Rollout parallelism & Rollout tensor parallel size 2 or 4 for Qwen3.5-4B, according to the batch. It is 1 for Qwen3.5-2B and Llama-3.1-8B-Instruct. The number of rollout replicas per node is $8$ divided by this value. \\
Memory configuration & Parameter, optimizer, and gradient offload enabled. Full activation recomputation uses one layer per recomputation unit. \\
Quality inference & Prior standard deviation $\sigma_z=1$, initial bisection half-width $B_z=4$, $T=40$ bisection iterations, and stable evaluation of the log CDF and inverse Mills ratio without probability clipping. \\
Cached pass rate & Uses a frozen criterion cache. \\
RPN & Frozen Qwen3-Embedding-4B text embedder \citep{zhang2025qwen3}. The final token embeddings of the prompt and criterion are concatenated and passed to two separate networks for $a_j$ and $b_j$, each with three hidden layers of width 1,024 and tanh activations. The outputs are $a_j=\operatorname{softplus}(\cdot)+0.05$ and $b_j=4\tanh(\cdot)$. \\
Online RPN update & AdamW learning rate $2\times10^{-5}$, weight decay 0.1, gradient clip norm $\tau_g=1.0$, $E=1$ epoch per policy step, mini-batch size $B=16$, one AdamW step per policy step over the 16 accumulated mini-batches, discrimination regularizer weight $\lambda_a=0.05$, matching the warm start, and warm start enabled. \\
Judge execution & One reward process, concurrency 8, timeout 50 seconds, and 3 retries. \\
\bottomrule
\end{tabular}
\end{table}

\begin{table}[!htbp]
\centering
\caption{Selected policy checkpoint steps. Each comparison uses a shared training step limit across its conditions.}
\label{tab:selected-policy-checkpoints}
\small
\setlength{\tabcolsep}{6pt}
\renewcommand{\arraystretch}{1.0}
\begin{tabular}{@{}llrrrr@{}}
\toprule
Policy & Condition & Medical & Science & RaR Science & RubricBench \\
\midrule
\multicolumn{6}{c}{Primary comparison} \\
\midrule
Qwen3.5-4B & Vanilla GRPO & 120 & 130 & 105 & 50 \\
 & RRT + frozen RPN & 145 & 150 & 120 & 45 \\
 & RRT + online RPN & 115 & 110 & 110 & 50 \\
Qwen3.5-2B & Vanilla GRPO & 100 & 150 & 120 & 45 \\
 & RRT & 125 & 150 & 110 & 45 \\
Llama-3.1-8B-Instruct & Vanilla GRPO & 115 & 150 & 170 & 50 \\
 & RRT & 150 & 150 & 145 & 50 \\
\midrule
\multicolumn{6}{c}{Adaptive Fisher selection with frozen RPN} \\
\midrule
Qwen3.5-4B & Full judging & 145 & 150 & 120 & 45 \\
 & \multicolumn{5}{l}{Adaptive Fisher} \\[-2pt]

 & \hspace{1em}\(\hookrightarrow\) 0.95 & 145 & 150 & 110 & 50 \\
 & \hspace{1em}\(\hookrightarrow\) 0.80 & 145 & 145 & 145 & 45 \\
 & \hspace{1em}\(\hookrightarrow\) 0.50 & 145 & 150 & 140 & 50 \\
\bottomrule
\end{tabular}
\end{table}

\subsection{Baseline Reward Aggregation}
\label{sec:baseline-rewards}

Algorithm~\ref{alg:baseline-rewards} summarizes DIVA \citep{cook2026check} and POW3R \citep{tyagi2026not} for binary verdicts. For POW3R, $\mathcal C$ partitions the rubric into nonempty criterion categories, $\lambda$ blends the contrast factor with one, and $\beta_{\mathrm{ema}}$ controls its exponential moving average. The multipliers $\alpha_j^{(t)}$ start at one and update after each epoch.

\begin{algorithm}[H]
\caption{DIVA and POW3R reward aggregation for one prompt.}
\label{alg:baseline-rewards}
\small
\begin{algorithmic}[1]
\Require Verdicts $G\in\{0,1\}^{N\times K}$, points $w_j>0$, categories $\mathcal C$, current multipliers $\alpha_j^{(t)}$, smoothing constants $\epsilon_D,\epsilon_P>0$, blend $\lambda\in[0,1]$, update rate $\beta_{\mathrm{ema}}\in[0,1]$, bounds $0<\alpha_{\min}\le1\le\alpha_{\max}$
\State $\bar G_j\gets N^{-1}\sum_i G_{ij}$, $V_j\gets N^{-1}\sum_i(G_{ij}-\bar G_j)^2$ for every criterion $j$
\State \textbf{DIVA:} $R_i^{\mathrm{DIVA}}\gets\dfrac{\sum_j(V_j+\epsilon_D)G_{ij}}{\sum_j(V_j+\epsilon_D)}$ for every rollout $i$
\State \textbf{POW3R:} $R_i^{\mathrm{POW3R}}\gets\dfrac{1}{|\mathcal C|}\sum_{C\in\mathcal C}\dfrac{\sum_{j\in C}w_j\alpha_j^{(t)}G_{ij}}{\sum_{j\in C}w_j\alpha_j^{(t)}}$ for every rollout $i$
\State $g_j\gets\sqrt{V_j+\epsilon_P}$ for every criterion $j$
\For{each category $C\in\mathcal C$}
  \State $\bar g_C\gets\sum_{j\in C}w_jg_j\big/\sum_{j\in C}w_j$
  \State $\hat\alpha_j\gets\operatorname{clip}\big((1-\lambda)+\lambda g_j/\bar g_C,\alpha_{\min},\alpha_{\max}\big)$ for every $j\in C$
\EndFor
\State $\alpha_j^{(t+1)}\gets\operatorname{clip}\big((1-\beta_{\mathrm{ema}})\alpha_j^{(t)}+\beta_{\mathrm{ema}}\hat\alpha_j,\alpha_{\min},\alpha_{\max}\big)$ for every $j$
\State \Return $R_i^{\mathrm{DIVA}}$, $R_i^{\mathrm{POW3R}}$, and $\alpha_j^{(t+1)}$
\end{algorithmic}
\end{algorithm}

\section{Computational Cost}
\label{sec:computational-cost}

\subsection{Judge Usage and Interface Reliability}
\label{sec:judge-cost}

This analysis measures judge requests and input tokens per policy step after RPN initialization for the full and partial judging comparison in Table~\ref{tab:fisher-budget-training}. Input token counts use four characters per token.

\begin{table}[H]
\centering
\caption{Mean Medical and Science judge requests and input tokens per policy step over the first 30 steps. Input tokens are in millions. Adaptive Fisher rows give the criterion budget in parentheses.}
\label{tab:judge-cost-overhead}
\small
\setlength{\tabcolsep}{6pt}
\renewcommand{\arraystretch}{1.0}
\begin{tabular}{@{}lcc@{\hspace{5pt}\vrule width 0.5pt\hspace{5pt}}cc@{}}
\toprule
& \multicolumn{2}{c}{Medical} & \multicolumn{2}{c}{Science} \\
\cmidrule(lr){2-3}\cmidrule(lr){4-5}
Method
& \shortstack{Judge\\requests} & \shortstack{Input\\tokens}
& \shortstack{Judge\\requests} & \shortstack{Input\\tokens} \\
\midrule
Vanilla GRPO & 7,716 & 11.1 & 6,908 & 16.3 \\
\midrule
\shortstack[l]{RRT,\\full judging (1.00)} & 7,685 & 11.1 & 6,895 & 16.5 \\
\midrule
\multicolumn{5}{l}{RRT, adaptive Fisher} \\[-2pt]

\hspace{1em}\(\hookrightarrow\) 0.95 & 7,447 & 10.8 & 6,670 & 16.1 \\
\hspace{1em}\(\hookrightarrow\) 0.80 & 6,265 & 9.1 & 5,622 & 12.8 \\
\hspace{1em}\(\hookrightarrow\) 0.50 & 3,916 & 5.7 & 3,515 & 8.4 \\
\bottomrule
\end{tabular}
\end{table}

Relative to matched full judging in Table~\ref{tab:judge-cost-overhead}, criterion budget 0.50 reduces judge requests by 49.0\% on Medical and Science.

This analysis reports token usage from the judge API for one policy step per dataset. It records retries and transport failures. When requests hit rate limits, they are sent to another endpoint, which adds request attempts. The range from p10 to p90 spans the 10th to 90th percentiles.

\begin{table}[H]
\centering
\caption{Judge API usage and request failures on Medical and Science for one policy step per dataset. Rows report completed judge request counts and token counts, input token distribution statistics, failed or additional attempts, retries, and ungraded criteria.}
\label{tab:judge-api-usage}
\small
\setlength{\tabcolsep}{5pt}
\renewcommand{\arraystretch}{1.0}
\begin{tabular}{@{}lcc@{}}
\toprule
Quantity & Medical & Science \\
\midrule
Completed judge requests & 6,728 & 4,984 \\
\multicolumn{3}{l}{Input tokens per request} \\[-2pt]

\hspace{1em}\(\hookrightarrow\) mean & 1,137 & 3,800 \\
\hspace{1em}\(\hookrightarrow\) median & 984 & 1,216 \\
\hspace{1em}\(\hookrightarrow\) p10 to p90 & 496 to 1,664 & 640 to 2,846 \\
\hspace{1em}\(\hookrightarrow\) maximum & 29,399 & 32,115 \\
Output tokens per request & 6.0 & 6.0 \\
\midrule
Attempts blocked by rate limits & 739 & 12 \\
Attempts that timed out & 0 & 0 \\
Unparseable verdicts & 0 & 0 \\
Criteria needing a retry & 0 & 0 \\
Criteria left ungraded & 0 & 0 \\
Additional attempts from endpoint failover & 11.0\% & 0.2\% \\
\bottomrule
\end{tabular}
\end{table}

In Table~\ref{tab:judge-api-usage}, all 11,712 completed judge requests return parseable verdicts, with no criteria needing a retry or left ungraded. Endpoint failover adds 11.0\% attempts on Medical and 0.2\% on Science.

\subsection{RRT Computation and Time per Policy Step}
\label{sec:step-time-cost}

Table~\ref{tab:rrt-overhead} reports the wall time of RRT operations beyond criterion judging as a percentage of the Vanilla GRPO Medical step. Table~\ref{tab:step-time-cost} shows that this step takes 2,274 seconds. RRT + frozen RPN evaluates $\psi$ once per unique criterion. Timings exclude the frozen text embedder because its outputs are cached across steps.

\begin{table}[H]
\centering
\caption{Added RRT computation for one Medical policy step, in seconds and as a percentage of the time for the Vanilla GRPO policy step.}
\label{tab:rrt-overhead}
\small
\setlength{\tabcolsep}{10pt}
\renewcommand{\arraystretch}{1.0}
\begin{tabular}{@{}lcc@{}}
\toprule
Method & Added seconds & Overhead of the policy step \\
\midrule
Vanilla GRPO & 0 & 0\% \\
\midrule
\multicolumn{3}{l}{RRT +} \\[-2pt]

\hspace{1em}\(\hookrightarrow\) batch pass rate & 0.10 & 0.004\% \\
\hspace{1em}\(\hookrightarrow\) cached pass rate & 0.10 & 0.004\% \\
\hspace{1em}\(\hookrightarrow\) frozen RPN & 0.17 & 0.007\% \\
\hspace{1em}\(\hookrightarrow\) online RPN & 2.79 & 0.123\% \\
\bottomrule
\end{tabular}
\end{table}

The online RPN update adds 2.79 seconds, or 0.123\% of the Vanilla GRPO Medical policy step, which takes 2,274 seconds.

The full step comparison then measures the total time per policy step and selected stage times across reward conditions. Table~\ref{tab:step-time-cost} lists generation, judging, log probability, policy update, and total time per policy step. Generation and judging form the rollout stage.

\begin{table}[H]
\centering
\caption{Median time per policy step and selected stage times. The final column gives the range from the 10th to 90th percentile for the total time per policy step. Blue shading marks RRT + online RPN.}
\label{tab:step-time-cost}
\small
\setlength{\tabcolsep}{5pt}
\renewcommand{\arraystretch}{1.0}
\begin{tabular}{@{}lcccccc@{}}
\toprule
& \multicolumn{6}{c}{Seconds per policy step, median with p10 to p90} \\
\cmidrule(lr){2-7}
Condition & Generation & Judging & \shortstack{Log\\probability} & \shortstack{Policy\\update} & \shortstack{Total policy\\step} & \shortstack{p10 to p90\\of policy step} \\
\midrule
\multicolumn{7}{c}{Medical} \\
\midrule
Vanilla GRPO & 460 & 1,451 & 36 & 107 & 2,274 & 1,136 to 3,603 \\
\multicolumn{7}{l}{RRT +} \\[-2pt]

\hspace{1em}\(\hookrightarrow\) frozen RPN & 647 & 1,457 & 38 & 124 & 2,336 & 1,815 to 3,092 \\
\hspace{1em}\(\hookrightarrow\) online RPN & 703 & 1,453 & 66 & 217 & 2,454 & 1,628 to 4,193 \\
\hspace{1em}\shortstack[l]{\(\hookrightarrow\) frozen RPN,\\\phantom{\(\hookrightarrow\)} adaptive Fisher (0.80)} & 659 & 1,187 & 69 & 223 & 2,237 & 1,771 to 3,947 \\
\hspace{1em}\shortstack[l]{\(\hookrightarrow\) frozen RPN,\\\phantom{\(\hookrightarrow\)} adaptive Fisher (0.50)} & 641 & 735 & 70 & 223 & 1,787 & 1,505 to 2,870 \\
\midrule
\multicolumn{7}{c}{Science} \\
\midrule
Vanilla GRPO & 427 & 1,520 & 79 & 257 & 2,381 & 1,792 to 3,207 \\
\multicolumn{7}{l}{RRT +} \\[-2pt]

\hspace{1em}\(\hookrightarrow\) frozen RPN & 517 & 1,516 & 51 & 167 & 2,363 & 1,286 to 3,169 \\
\hspace{1em}\(\hookrightarrow\) online RPN & 800 & 1,505 & 57 & 186 & 2,655 & 2,225 to 3,346 \\
\hspace{1em}\shortstack[l]{\(\hookrightarrow\) frozen RPN,\\\phantom{\(\hookrightarrow\)} adaptive Fisher (0.80)} & 513 & 1,234 & 56 & 188 & 2,137 & 1,582 to 12,011 \\
\hspace{1em}\shortstack[l]{\(\hookrightarrow\) frozen RPN,\\\phantom{\(\hookrightarrow\)} adaptive Fisher (0.50)} & 522 & 771 & 56 & 186 & 1,686 & 1,336 to 7,844 \\
\bottomrule
\end{tabular}
\end{table}

Relative to full judging with the same frozen RPN, adaptive Fisher selection at criterion budget 0.50 reduces median judging time from 1,457 to 735 seconds on Medical and from 1,516 to 771 seconds on Science. Median total step time falls from 2,336 to 1,787 seconds on Medical and from 2,363 to 1,686 seconds on Science, reductions of 23.5\% and 28.7\%, respectively.

\subsection{Cost of the RPN Warm Start}
\label{sec:warm-start-cost}

This experiment measures the cost of fitting the RPN before policy training. The RPN of each dataset is fitted once before RRT policy training and reused by the variants that use an RPN. The frozen embedder encodes each cached prompt and criterion once, and RPN fitting reuses those vectors across epochs. The 0.6B and 8B columns show costs for the smallest and largest of the three Qwen3 embedder sizes used in the RPN configuration experiment.

\begin{table}[H]
\centering
\caption{Cost of an RPN warm start on one A100 80 GB GPU. Columns report optimizer steps and combined embedding and fitting GPU hours for 0.6B and 8B frozen text embedders.}
\label{tab:warm-start-cost}
\small
\setlength{\tabcolsep}{5pt}
\renewcommand{\arraystretch}{1.0}
\begin{tabular}{@{}lccc@{}}
\toprule
Dataset & Optimizer steps & \shortstack{GPU hours,\\0.6B embedder} & \shortstack{GPU hours,\\8B embedder} \\
\midrule
Medical & 3,000 & 5.30 & 7.64 \\
Science & 1,254 & 2.48 & 4.71 \\
RaR Science & 1,254 & 1.12 & 2.33 \\
RubricBench & 1,067 & 0.89 & 1.57 \\
\bottomrule
\end{tabular}
\end{table}

In Table~\ref{tab:warm-start-cost}, the cost of a warm start ranges from 0.89 to 5.30 GPU hours with the 0.6B embedder and from 1.57 to 7.64 GPU hours with the 8B embedder.

\subsection{Deployment Requirements}
\label{sec:inference-cost}

This comparison tests whether training with RRT changes deployment requirements. It compares the deployed architecture of the base policy with a policy trained using RRT. The RPN, E-step, stochastic partial M-step, and judge are training components.

\begin{table}[H]
\centering
\caption{Added deployment parameters and inference components for the base policy and a policy trained with RRT.}
\label{tab:inference-cost}
\small
\setlength{\tabcolsep}{10pt}
\renewcommand{\arraystretch}{1.0}
\begin{tabular}{@{}lcc@{}}
\toprule
Policy & Added deployment parameters & Added inference components \\
\midrule
Base policy & 0 & None \\
Policy trained with RRT & 0 & None \\
\bottomrule
\end{tabular}
\end{table}

Table~\ref{tab:inference-cost} shows that RRT adds 0 deployment parameters and 0 inference components.